\documentclass[twoside]{article}

\usepackage[accepted]{ArXiv_aistats2021}

\usepackage[utf8]{inputenc} 

\usepackage[T1]{fontenc}    
\usepackage{hyperref}       
\usepackage{url}            
\usepackage{booktabs}       
\usepackage{amsfonts}       
\usepackage{nicefrac}       
\usepackage{microtype}      

\usepackage{graphicx}
\usepackage{subcaption}
\usepackage{epsf}
\usepackage{fancyhdr}
\usepackage{graphics}
\usepackage{psfrag}

\usepackage{wrapfig}
\usepackage{float}

\usepackage{algorithm}
\usepackage{algorithmic}

\usepackage{color,xcolor}
\usepackage{amsmath}
\usepackage{amssymb}
\usepackage{mathtools}

\usepackage{amsthm,amssymb}
\usepackage{bm,bbm}

\usepackage{multirow}
\usepackage{balance}

\usepackage{verbatim}

\newcommand{\beq}{\begin{equation}}
\newcommand{\eeq}{\end{equation}}

\newtheorem{theorem}{Theorem}[section]
\newtheorem{corollary}[theorem]{Corollary}
\newtheorem{definition}[theorem]{Definition}
\newtheorem{remark}[theorem]{Remark}
\newtheorem{lemma}[theorem]{Lemma}
\newtheorem{proposition}[theorem]{Proposition}

\DeclareMathOperator*{\argmax}{arg\,max}
\DeclareMathOperator*{\argmin}{arg\,min}

\def\ET{\textnormal{ET}}
\def\KT{\textnormal{KT}}
\def\tildeET{\widetilde{\ET}}
\def\tildeETlambda{\tildeET_{\lambda}^{0}}

\def\dAlpha{d_{\alpha}}

\newcommand{\R}{{\mathbb R}}

 \newcommand{\K}{{\mathbb K}}

\newcommand{\e}{\varepsilon}

\newcommand{\calM}{{\mathcal M}}

\newcommand{\calT}{{\mathcal T}}
\newcommand{\calW}{{\mathcal W}}

\newcommand{\calF}{{\mathcal F}}

\def\longequals{\mathbin{=\kern-2pt=}}

\def\RR{\mathbb{R}}

\def\NN{\mathbb{N}}

\def\Xx{\mathcal{X}}
\def\Tt{\mathcal{T}}

\def\O{\mathcal{O}}

\begin{document}

%

%

\twocolumn[

\aistatstitle{Entropy Partial Transport with Tree Metrics: Theory and Practice}

\aistatsauthor{Tam Le* \And Truyen Nguyen*}

\aistatsaddress{RIKEN AIP \And  The University of Akron} ]

\begin{abstract}

Optimal transport (OT) theory provides powerful tools to compare probability measures. However, OT is limited to nonnegative measures having the same mass, and suffers serious drawbacks about its computation and statistics. This leads to several proposals of regularized variants of OT in the recent literature. In this work, we consider an \textit{entropy partial transport} (EPT) problem for nonnegative measures on a tree having different masses. The EPT is shown to be equivalent to a standard complete OT problem on a one-node extended tree. We derive its dual formulation, then leverage this to propose a novel regularization for EPT which admits fast computation and negative definiteness. To our knowledge, the proposed regularized EPT is the first approach that yields a \textit{closed-form} solution among available variants of unbalanced  OT. For practical applications without priori knowledge about the tree structure for measures, we propose tree-sliced variants of the regularized EPT, computed by averaging the regularized EPT between these measures using random tree metrics, built adaptively from support data points. Exploiting the negative definiteness of our regularized EPT, we introduce a positive definite kernel, and evaluate it against other baselines on benchmark tasks such as document classification with word embedding and topological data analysis. In addition, we empirically demonstrate that our regularization also provides effective approximations.

\end{abstract}

\section{Introduction}\label{sec:introduction}

Optimal transport (OT) theory offers powerful tools to compare probability measures \cite{villani2008optimal}. OT has been applied for various tasks in machine learning \cite{bunne2019, courty2017joint, nadjahi2019asymptotic, peyre2019computational}, statistics \cite{mena2019statistical, pmlr-v99-weed19a} and computer graphics \cite{lavenant2018dynamical, solomon2015convolutional}. However, OT requires input measures having the same mass which may limit its applications in practice since one often needs to deal with measures
of unequal masses. For instance, in natural language processing, we can view a document as a measure where each word is regarded as a point in the support with a unit mass. Thus, documents with different lengths lead to their associated measures having different masses.

To tackle the transport problem for measures having different masses, Caffarelli and McCann \cite{CM} proposed the \textit{partial optimal transport} (POT) where one only transports a fixed amount of mass from a measure into another. Later, Figalli \cite{figalli2010optimal} extended the theory of POT, notably, about the uniqueness of solutions. A different  approach is to optimize the sum of a transport functional and two convex entropy functionals which quantify the deviation of the marginals of the transport plan from the input measures \cite{Liero2018}, i.e., the \textit{optimal entropy transport} (OET) problem. This formulation  recovers many different previous works. For examples, when the entropy is equal to the total variation distance or the $\ell^2$ distance, the OET is respectively equivalent to the \textit{generalized Wasserstein distance} \cite{P1, P2} or the \textit{unbalanced mass transport} \cite{benamou2003numerical}. It is worth noting that the \textit{generalized Wasserstein distance}  shares the same spirit as the \textit{Kantorovich-Rubinstein} discrepancy \cite{guittet2002extended, hanin1992kantorovich, lellmann2014imaging}. Another variant is the \textit{unnormalized optimal transport} \cite{gangbo2019unnormalized} which mixes Wasserstein distance and the $\ell^p$ distance. There are several applications of the transport problem for measures having different masses such as in machine learning \cite{frogner2015learning, janati2019wasserstein}, deep learning \cite{yang2018scalable}, topological data analysis \cite{lacombe2018large}, computational imaging \cite{lee2019parallel}, and computational biology \cite{schiebinger2019optimal}.

One important case for the OET problem is when the entropy is equal to the Kullback-Leibler (KL) divergence and a particular cost function is used, then OET is equivalent to the \textit{Kantorovich-Hellinger} distance (i.e., \textit{Wasserstein-Fisher-Rao} distance) \cite{chizat2018scaling, Liero2018}. In addition, one can apply the Sinkhorn-based algorithm to efficiently solve OET problem when the entropy is equal to KL divergence, i.e., \textit{Sinkhorn-based approach for unbalanced optimal transport} (Sinkhorn-UOT)~\cite{chizat2018scaling, frogner2015learning}. In \cite{pham2020unbalanced}, Pham et al. showed that the complexity of Sinkhorn-based algorithm for Sinkhorn-UOT is quadratic which is similar to the case of entropic regularized OT \cite{Cuturi-2013-Sinkhorn} for probability measures. However, for large-scale applications where the supports of measures contain a large number of points, the computation of Sinkhorn-UOT becomes prohibited. Following the sliced-Wasserstein (SW) distance \cite{bonneel2015sliced, rabin2011wasserstein} which projects supports into a one-dimensional space and employs the closed-form solution of the univariate optimal transport (1d-OT), Bonneel and Coeurjolly \cite{bonneel2019spot} propose the \textit{sliced partial optimal transport} (SPOT) for nonnegative measures having different masses. Unlike the standard 1d-OT, one does not have a closed-form solution for measures of unequal masses that are supported in a one-dimensional space. With an assumption of a unit mass on each support, Bonneel and Coeurjolly \cite{bonneel2019spot} derived an efficient algorithm to solve the SPOT problem in quadratic complexity for the worst case. Especially, in practice, their proposed algorithm is nearly linear for computation. However, as in SW, the SPOT uses one-dimensional projection for supports which limits its capacity to capture a structure of a distribution, especially in high-dimensional settings \cite{LYFC, liutkus2019sliced}.

In this work, we aim to develop an efficient and scalable approach for the transport problem when input measures have different masses. Inspired by the tree-sliced Wasserstein (TSW) distance \cite{LYFC} which has fast closed-form computation and remedies the curse of dimensionality for SW, we propose to consider the \textit{entropy partial transport} (EPT) problem with tree metrics. As a high level, our main contribution is three-fold as follows:
\begin{itemize}
\item We establish a relationship between the EPT problem with mass constraint and a formulation with  Lagrangian multiplier.
Then, we employ it to transform the EPT problem to the standard complete OT problem on a suitable one-node extended tree.

\item We derive a dual formulation for our EPT problem. We then leverage it to propose a novel regularization which admits a closed-form formula and negative definiteness. Consequently, we introduce positive definite kernels for our regularized EPT. We also derive tree-sliced variants of the regularized EPT for applications without priori knowledge about tree structure for measures.

\item We empirically show that (i) our regularization provides both efficient approximations and fast computations, and (ii) the performances of the proposed kernels for our regularized EPT compare favorably with other baselines in applications. 
\end{itemize}



\section{Preliminaries}\label{sec:background}

Let $\calT =(V,E)$ be a tree rooting at node $r$ and with nonnegative edge lengths $\{w_e\}_{e\in E}$, where $V$ is the collection of nodes and $E$ is the collection of edges. For convenience, we use $\calT$ to denote the set of all nodes together with all points on its edges. We then recall the definition of tree metric as follow:

\begin{definition}[Tree metric~\cite{semple2003phylogenetics}(\S7, p.145--182)] 
A metric $\texttt{d}:\Omega \times \Omega \rightarrow [0,\infty)$ is called a \textit{tree metric} on $\Omega$ if there exists tree $\Tt$ such that $\Omega \subseteq \Tt$ and for $x, y \in \Omega$, $\texttt{d}(x, y)$ equals to the length of the (unique) path between $x$ and $y$.
\end{definition}


Assume that $V$ is a subset of a vector space, and let $d_\calT(\cdot,\cdot)$ be the tree metric on $\calT$. Hereafter, the unique shortest path in $\calT$ connecting $x$ and $y$ is denoted by $[x,y]$.  Let $\omega$ be the unique Borel measure (i.e., the length measure) on $\calT$ satisfying $\omega([x,y]) = d_\calT(x,y)$ for all $x,y\in \calT$. Given $x\in \calT$, the set $\Lambda(x)$ stands for the subtree below $x$. Precisely, 
\begin{equation}\label{subtree}
 \Lambda(x):= \big\{y\in \calT: \, x\in [r,y]\big\}.
 \vspace{-6pt}
\end{equation}
We shall use notation $\calM(\calT)$ to represent the set  of all nonnegative Borel measures on $\calT$ with a finite mass.
Also let 
$C(\calT)$  be  the set of all continuous functions on $\calT$, while $L^\infty(\calT)$ be the collection of all Borel measurable functions on $\calT$ that are bounded $\omega$-a.e. Then, $L^\infty(\calT)$ is a Banach space under the norm 
\[
\|f\|_{L^\infty(\calT)} := \inf\{a\in \R:\, |f(x)|\leq a \mbox{ for  $\omega$-a.e. } x\in\calT\}.
\]

\section{Entropy Partial Transport (EPT) with Tree Metrics}\label{sec:entropy_partialOT}

 Let $b\geq 0$ be a constant,  $c:\calT\times\calT \to \R$ be a continuous cost with $c(x,x)=0$,   $F_1, \, F_2: [0,\infty)\to (0,\infty)$ be entropy functions which are  convex and lower semicontinuous,
 and let $w_1, w_2:\calT \to [0,\infty)$ be two nonnegative weights.
 For $\mu, \nu \in \calM(\calT)$, consider the region 
\[
\Pi_{\leq}(\mu,\nu) :=\Big\{ \gamma \in \calM(\calT \times \calT):  \, \gamma_1\leq \mu, \, \gamma_2\leq \nu \Big\}
\vspace{-6pt}
\]
with  $\gamma_i$ ($i=1, 2$) denoting  the $i$th marginal of the measure $\gamma$.  
For $\gamma \in \Pi_{\leq}(\mu,\nu)$, 
the Radon-Nikodym derivatives of $\gamma_1$ with respect to $\mu$ and of $\gamma_2$ with respect to  $\nu$
exist due to $\gamma_1\leq \mu$ and $\gamma_2\leq \nu$.
From now on, we let  $f_1$ and $f_2$ respectively denote  these 
Radon-Nikodym derivatives, i.e., $\gamma_1=f_1 \mu$ and $\gamma_2 = f_2 \nu$. Then $0\leq f_1\leq 1$ $\mu$-a.e. and $0\leq f_2\leq 1$ $\nu$-a.e. Throughout the paper, $\bar m$ stands for the minimum of the total masses
of $\mu$ and $\nu$. That is, $\bar m := \min\{\mu(\calT), \nu(\calT) \}$.  
Inspiring by  \cite{CM, Liero2018}, we fix a number  $m\in [0,\bar m]$ and  consider the following EPT problem:
\begin{eqnarray}\label{original}
 \calW_{c,m}(\mu,\nu) :=\inf_{\gamma \in \Pi_{\leq}(\mu,\nu), \, \gamma(\calT\times \calT)=m}\Big[  \calF_1(\gamma_1| \mu )  \nonumber \\
 + \calF_2(\gamma_2| \nu )  + b \, \int_{\calT \times \calT} c(x,y) \gamma(dx, dy) \Big],
 \vspace{-6pt}
\end{eqnarray}
where $\calF_1(\gamma_1| \mu) := \int_\calT w_1(x) F_1(f_1(x) ) \mu(dx)$ 
and $\calF_2(\gamma_2| \nu) := \int_\calT w_2(x) F_2(f_2(x) ) \nu(dx)$ are the weighted relative entropies. The role of the two entropies  in the minimization problem is to force the marginals of $\gamma$ close to $\mu$ and  $\nu$ respectively.
Let us introduce a Lagrange multiplier 
$\lambda\in\R$ conjugate to the constraint $\gamma(\calT\times \calT)=m$. As a result, we instead study the following formulation
 \begin{eqnarray*}
\mathrm{ET}_{c,\lambda}(\mu,\nu) := \inf_{\gamma \in \Pi_{\leq}(\mu,\nu)}\Big[  \calF_1(\gamma_1| \mu )  +  \calF_2(\gamma_2| \nu )  \\
 + b \, \int_{\calT \times \calT} [c(x,y)-\lambda] \gamma(dx, dy) \Big].
 \vspace{-6pt}
\end{eqnarray*}

In this paper, we focus on the specific entropy functions  $F_1(s)=F_2(s)=|s-1|$.
Thus, the quantity of interest  becomes
\begin{eqnarray}\label{P1}
\mathrm{ET}_{c,\lambda}(\mu,\nu) 
= \inf_{\gamma \in \Pi_{\leq}(\mu,\nu)} \mathcal{C}_\lambda(\gamma),
\end{eqnarray}
where $\mathcal{C}_\lambda(\gamma)$
is defined as follow: 
\begin{align}\label{easier-form}
&\hspace{-0.8 em} \mathcal{C}_\lambda(\gamma) := \int_\calT \hspace{-0.3 em} w_1 [1-f_1(x)] \mu(dx)  + \int_\calT  \hspace{-0.3 em} w_2 [1-f_2(x)] \nu(dx) \nonumber\\
& \hspace{9 em}+ b \, \int_{\calT \times \calT} [c(x,y)-\lambda] \gamma(dx, dy) \nonumber\\
& \hspace{-0.2 em} = \int_\calT  w_1 \mu(dx) + \int_\calT  w_2  \nu(dx) -\int_\calT  w_1 \gamma_1(dx) \nonumber\\
& \hspace{0.5 em} - \int_\calT  w_2\gamma_2(dx)+ b  \int_{\calT \times \calT} [c(x,y)-\lambda]\gamma(dx, dy).
\end{align}
Notice that problem \eqref{P1}
is a generalization of  the  \textit{generalized Wasserstein distance} $\calW_1^{a,b}(\mu,\nu) $ introduced in \cite{P1,P2}. We next display some relationships between problem \eqref{original} with mass constraint $m$  and problem \eqref{P1} with Lagrange multiplier $\lambda$. For this,  let  $\Gamma^0(\lambda)$ denote the set of all optimal plans (i.e., minimizers $\gamma$) for $\mathrm{ET}_{c,\lambda}(\mu,\nu)$. Then, since $\mathcal{C}_\lambda(\gamma)$
is an affine function of  $\gamma \in \Pi_{\leq}(\mu,\nu)$, the set $\Gamma^0(\lambda)$ is a nonempty convex set. Indeed, for any  $\tilde \gamma,\hat \gamma\in \Gamma^0(\lambda )$ and for any $t\in [0, 1]$ we have $(1-t) \tilde \gamma + t \hat\gamma \in \Gamma^0(\lambda)$ due to 
$\mathcal{C}_\lambda((1-t) \tilde \gamma + t \hat\gamma) 
  = (1-t) \mathcal{C}_\lambda( \tilde \gamma) + t  \mathcal{C}_\lambda( \hat \gamma)
  \leq (1-t) \mathcal{C}_\lambda(  \gamma) + t  \mathcal{C}_\lambda( \gamma)
  = \mathcal{C}_\lambda(  \gamma) $ for every $\gamma\in \Pi_{\leq}(\mu,\nu)$.
The following result extends Corollary~2.1 in \cite{CM} and reveals the connection between   problem  \eqref{original} and problem \eqref{P1}.

\begin{theorem}\label{thm:m-via-lambda} Let 
$u(\lambda) := -\mathrm{ET}_{c,\lambda}(\mu,\nu)$ for $\lambda\in\R$, and denote 
\[
\partial u(\lambda) := \Big\{p\in \R: u(t)\geq u(\lambda) + p(t-\lambda), \forall t\in\R \Big\}
\vspace{-6pt}
\]
for the set of all subgradients of $u$ at $\lambda$. Also,  set $\partial u(\R) :=\cup_{\lambda\in \R} \partial u(\lambda)$. Then, we have 
\begin{itemize}
\item[i)] $u$ is a convex function on $\R$, and 
 \[
 \partial u(\lambda) =\big\{ b\, \gamma(\calT\times\calT): \gamma\in \Gamma^0(\lambda)\big\}
 \quad \forall \lambda\in\R.
 \vspace{-4pt}
 \]
 Also if $\lambda_1<\lambda_2$, then $m_1 \leq m_2$ for every $m_1\in \partial u(\lambda_1)$ and $m_2\in \partial u(\lambda_2)$. 


\item[ii)] $u$ is differentiable at $\lambda$ if and only if every optimal plan in $\Gamma^0(\lambda)$ has the same mass. When this happens, we in addition have 
$u'(\lambda) =b \, \gamma(\calT\times\calT)$ for any  $\gamma\in \Gamma^0(\lambda)$.

\item[iii)] If there exists a constant $M>0$ such that $w_1(x) +w_2(y) \leq b\,  c(x,y) + M$ for all $x,y\in \calT$, then  $\partial u(\R)=[0,b\, \bar m]$. Moreover, $u(\lambda)=-\int_\calT  w_1 \mu(dx) 
- \int_\calT  w_2  \nu(dx)$ when $\lambda<-M$, and 
$u'(\lambda)=b\, \bar m $ 
for $\lambda> \|c\|_{L^\infty(\calT\times\calT)}$.   
\end{itemize}
\end{theorem}

Proof is placed in the Supplementary (\S A.1). For any $m\in [0, \bar m]$,  part iii) of  Theorem~\ref{thm:m-via-lambda} implies  that there exists $\lambda\in \R$ such that $b\, m \in \partial u(\lambda)$. It then follows from part i) of this theorem  that  $m=  \gamma^*(\calT\times\calT)$ for some $\gamma^*\in \Gamma^0(\lambda)$. It is also clear that this $\gamma^*$ is an optimal plan for $ \calW_{c,m}(\mu,\nu)$, and
\begin{align*} 
&\calW_{c,m}(\mu,\nu)=\mathrm{ET}_{c,\lambda}(\mu,\nu) +\lambda b\,  m.
\end{align*}
Thus solving the auxiliary problem \eqref{P1} gives us a solution  to  the original problem  \eqref{original}.
When $u$ is differentiable, the relation between $m$ and $\lambda$ is given explicitly as $u'(\lambda)=b \, m$.
Note that the above selection of $\lambda$ is unique only if the function $u$ is strictly convex. Nevertheless, it enjoys the following monotonicity regardless of the uniqueness: if $m_1< m_2$, then $\lambda_1 \leq \lambda_2$. Indeed, we have  $m_1=  \gamma^1(\calT\times\calT)$ and $m_2=  \gamma^2(\calT\times\calT)$ for some $\gamma^1\in \Gamma^0(\lambda_1)$ and $\gamma^2\in \Gamma^0(\lambda_2)$. Since $\gamma^1(\calT\times\calT)<\gamma^2(\calT\times\calT)$, one has  $\lambda_1\leq \lambda_2$ by i) of Theorem~\ref{thm:m-via-lambda}.


 To investigate  problem \eqref{P1},  we recast it as the standard complete OT problem by using an observation in \cite{CM}. More precisely, let $\hat s$ be a point outside $\calT$ and consider the set  $\hat\calT:= \calT \cup \{\hat s\}$. We next  extend the cost function to $\hat \calT\times \hat \calT$ as follow 
\begin{equation*}
\hat c(x,y) :=
\left\{\begin{array}{lr}
\!\!b[c(x,y)-\lambda] \hspace{1 em} \mbox{ if } x,y\in \calT,\\
\!\!w_1(x) \hspace{4 em} \mbox{ if }  x\in \calT \mbox{ and } y=\hat s,\\
 \!\! w_2(y) \hspace{4 em}  \mbox{ if }  x=\hat s \mbox{ and } y\in \calT,\\
  \!\! 0 \hspace{6 em} \mbox{ if }  x=y=\hat s.
\end{array}\right.
\end{equation*}
The measures $\mu, \nu$ are extended accordingly by adding a Dirac mass at the isolated point $\hat s$: $\hat\mu = \mu +\nu(\calT) \delta_{\hat s}$ and $\hat\nu = \nu +\mu(\calT) \delta_{\hat s}$. As $\hat\mu, \hat\nu$ have the same total mass on $\hat \calT$, we can consider the standard complete OT problem between $\hat\mu, \hat\nu$ as follow
\begin{align}\label{P2}
\mathrm{KT}(\hat \mu,\hat \nu) := \inf_{\hat \gamma \in \Gamma(\hat\mu,\hat \nu)}  \int_{\hat \calT\times \hat \calT} \hat c(x,y) \hat\gamma(dx, dy),
\end{align}
where $\Gamma(\hat \mu,\hat \nu) :=\Big\{ \hat\gamma \in \calM( \hat\calT \times \hat \calT): \hat \mu(U) =\hat\gamma(U\times \hat \calT),\, \hat\nu(U)= \hat\gamma(\hat \calT\times U) \mbox{ for all Borel sets } U\subset \hat \calT\Big\}$.

A one-to-one correspondence between $\gamma\in \Pi_{\leq}(\mu,\nu)$ and $\hat \gamma\in \Gamma(\hat \mu,\hat \nu)$ is given by 
\begin{align}\label{one-to-one}
\hat \gamma = \gamma + [(1-f_1) \mu] \otimes \delta_{\hat s} +\delta_{\hat s} \otimes [(1-f_2) \nu] \nonumber\\
+\gamma(\calT\times \calT) \delta_{(\hat s, \hat s)}.
\end{align}
Indeed, if $\gamma\in \Pi_{\leq}(\mu,\nu)$, then it is clear that $\hat\gamma$ defined by \eqref{one-to-one} satisfies $\hat \gamma\in \Gamma(\hat \mu,\hat \nu)$. The converse is guaranteed by the next technical result.
\begin{lemma}\label{rep-formula}
For $\hat \gamma\in \Gamma(\hat \mu,\hat \nu)$, let $\gamma$  be the restriction of  $\hat\gamma$ to $\calT$. Then, relation \eqref{one-to-one} holds  and $\gamma\in \Pi_{\leq}(\mu,\nu)$.
 \end{lemma}
 
Proof is placed in the Supplementary (\S A.2). 

These observations in particular display the following connection between the EPT problem and the standard complete OT problem.
\begin{proposition}[EPT versus complete OT]\label{distance-agree}
For every $\mu, \nu \in \calM(\calT )$, we have
 $\mathrm{ET}_{c,\lambda}(\mu,\nu)=\mathrm{KT}(\hat \mu,\hat\nu)$. Moreover, relation \eqref{one-to-one} gives a one-to-one correspondence between  optimal solution  $\gamma$ for EPT problem \eqref{P1} and optimal solution $\hat \gamma$ for standard complete OT problem \eqref{P2}.
\end{proposition}

Proof is placed in the Supplementary (\S A.3). 

\subsection{Dual Formulations}

The relationship given in Proposition~\ref{distance-agree} allows us to obtain the dual formulation of EPT in problem \eqref{P1} from that of problem \eqref{P2} proved in  \cite[Corollary~2.6] {CM}.
\begin{theorem}[Dual formula for general cost]\label{thm:duality} 
For  any $\lambda \geq 0$ and   nonnegative weights $w_1(x), w_2(x)$, we have 
\begin{equation*}
\mathrm{ET}_{c,\lambda}(\mu,\nu) = \sup_{(u,v)\in \K} \Big[\int_{\calT}  u(x) \mu(dx) +  \int_{\calT}  v(x)  \nu(dx)\Big],
\vspace{-10pt}
\end{equation*}
where $\K :=\Big\{(u,v):\,  u\leq w_1,\, -b \lambda + \inf_{x\in \calT} [b\, c(x,y) -w_1(x)]\leq v(y)\leq w_2(y),\,  u(x) + v(y)\leq  b[c(x,y) - \lambda]\Big\}$.
\end{theorem}

Proof is placed in the Supplementary (\S A.4). 

This dual formula is our main theoretical result and can be rewritten more explicitly when the cost $c$ is the tree distance. Hereafter, we use $c(x,y ) = d_\calT(x,y)$. To ease the notations, we simply write $\mathrm{ET}_\lambda(\mu,\nu)$ for 
$\mathrm{ET}_{d_\calT,\lambda}(\mu,\nu)$.

\begin{corollary}[Dual formula for tree distance]\label{cor:duality} Assume that  $\lambda \geq 0$ and   the nonnegative weights $w_1, w_2$ are $b$-Lipschitz w.r.t. $d_\calT$. Then, we have
\begin{eqnarray}\label{equ:ETlambda}
& \hspace{-5 em}\mathrm{ET}_\lambda(\mu,\nu) = \sup \left\{ \int_\calT f (\mu - \nu):\, f\in \mathbb{L}  \right\} \nonumber \\
& \hspace{10 em} - \frac{b\lambda}{2}\big[ \mu(\calT) +  \nu(\calT)\big],
\vspace{-4pt}
\end{eqnarray}
where $ \mathbb{L} :=\Big\{f\in C(\calT):\, 
  -w_2 - \frac{b\lambda}{2}\leq f \leq    w_1 + \frac{b\lambda}{2}, \, |f(x)-f(y) |\leq b \, d_\calT(x,y)\Big\}$.
\end{corollary}

Proof is placed in the Supplementary (\S A.5). 

Corollary~\ref{cor:duality} extends the dual formulation for the
\textit{generalized Wasserstein distance}
$\calW_1^{a,b}(\mu,\nu) $ proved in \cite[Theorem~2]{P2} and   \cite{chung2019duality}. In the next section, we will leverage \eqref{equ:ETlambda} to propose 
an effective regularization 
for computation in practice. 

\begin{remark}\label{rem:bLipschitz}
An example of $b$-Lipschitz weight  is $w(x) =a_1\, d_\calT(x, x_0) + a_0$ for some $x_0\in\calT$ and for some constants  $a_1\in [0,b]$ and  $a_0\in [0,\infty)$.
\end{remark}

As a consequence of  the dual formulation, we obtain the following   geometric properties:
\begin{proposition}[Geometric structures of metric d]\label{geodesic-space-part1}  
Assume that  $\lambda \geq 0$ and   the  weights $w_1, w_2$ are  positive and $b$-Lipschitz w.r.t. $d_\calT$. Define $d(\mu,\nu) := \mathrm{ET}_\lambda(\mu,\nu) +\frac{b\lambda}{2}\big[ \mu(\calT) +  \nu(\calT)\big]$. Then, we have 
\begin{enumerate}
\item[i)]  $d(\mu +\sigma,\nu +\sigma) = d(\mu,\nu)$, $\forall \sigma \in\calM(\calT)$.
\item[ii)]  $d$ is a divergence and satisfies the triangle inequality $d(\mu,\nu)\leq d(\mu,\sigma) + d(\sigma, \nu)$. 
\item[iii)]  If in addition $w_1=w_2$, then  $(\calM(\calT), d)$ is a complete metric space. Moreover, it is a geodesic space in the sense that for every two points $\mu$ and $\nu$ in $\calM(\calT)$ there exists a  path 
$\varphi: [0,a ]\to \calM(\calT)$ with $a:=d(\mu,\nu)$ such that $\varphi(0)=\mu$, $\varphi(a)=\nu$, and
\[
d(\varphi(t), \varphi(s)) = |t-s|\quad \mbox{for all } t,s\in [0,a].
\vspace{-6pt}
\]
\end{enumerate}
\end{proposition}


Proof is placed in the Supplementary (\S A.6).

Let $m\in [0,\bar m]$, and choose $\lambda\geq 0$   such that there exists   an optimal plan   $\gamma^0$ for $ \mathrm{ET}_\lambda(\mu,\nu)$ with
$\gamma^0(\calT\times\calT)=m$. As pointed out right after Theorem~\ref{thm:m-via-lambda}, this choice of $\lambda$ is possible.
Then, the proof of Lemma A.1 in the Supplementary (\S A.6)  shows that
\begin{align*}
& \hspace{-0.3 em} \inf_{\gamma \in \Pi_{\leq}(\mu,\nu), \, \gamma(\calT\times \calT)=m}\Big[  \calF_1(\gamma_1| \mu )  +  \calF_2(\gamma_2| \nu )  \\
& \hspace{7 em}+ b \, \int_{\calT \times \calT} c(x,y) \gamma(dx, dy) \Big]\leq d(\mu,\nu).
\vspace{-6pt}
\end{align*}
Moreover, the equality happens if and only if there exists an optimal plan $\gamma^0$ for $ \mathrm{ET}_\lambda(\mu,\nu)$ such that
$m=\gamma^0(\calT\times\calT)=\frac12[\mu(\calT) +\nu(\calT)]$. The necessary conditions for the latter one to hold is $\mu(\calT)=\nu(\calT)$ and $m=\bar m$.

\subsection{An Efficient Regularization for Entropy Partial Transport with Tree Metrics}

First observe that any $f\in \mathbb L$ can be represented by 
\[
f(x) =f(r) + \int_{[r,x]} g(y) \omega(dy) 
\vspace{-4pt}
\]
for some function $g\in L^\infty(\calT)$ with $\|g\|_{L^\infty(\calT)}\leq b$. Note that 
condition $|f(x)-f(y) |\leq b \, d_\calT(x,y)$  is equivalent to $\|g\|_{L^\infty(\calT)}\leq b$.
 It follows that $\mathbb L \subset \mathbb L_0$, where we define for $0\leq \alpha\leq \frac12 [b\lambda + w_1(r) + w_2(r)]$ that $\mathbb L_\alpha$ is the collection of all functions $f$ of the form
 \[
  f(x)= s +  \int_{[r,x]} g(y) \omega(dy),
  \vspace{-6pt}
 \]
with $s$ being a constant in the interval $\Big[  -w_2(r)- \frac{b\lambda}{2}+\alpha, w_1(r) + \frac{b\lambda}{2} -\alpha\Big]$ and with  $\|g\|_{L^\infty(\calT)}\leq b$. This leads us to consider the following regularization for $\mathrm{ET}_\lambda(\mu,\nu)$:
\begin{eqnarray}\label{equ:tildeETlambda}
& \hspace{-5 em} \widetilde{\mathrm{ET}}_\lambda^\alpha(\mu,\nu) := \sup \left\{ \int_\calT f (\mu - \nu):\, f\in \mathbb{L}_\alpha  \right\} \nonumber \\
& \hspace{10 em} - \frac{b\lambda}{2}\big[ \mu(\calT) +  \nu(\calT)\big].
\vspace{-10pt}
\end{eqnarray}
Especially, when $\alpha=0$ and notice that $\mathbb L \subset \mathbb L_0$, $\widetilde{\mathrm{ET}}_\lambda^0(\mu,\nu)$ is an upper bound of $\mathrm{ET}_\lambda(\mu,\nu)$ through the dual formulation. The next result gives a closed-form formula for $\widetilde{\mathrm{ET}}_\lambda^\alpha(\mu,\nu)$ and is our main formula used for computation in practice. 

\begin{proposition}[closed-form for regularized EPT]\label{pro:tildeET}  Assume that $\lambda, w_1(r), w_2(r)$ are nonnegative numbers. Then, for $0\leq \alpha\leq \frac12 [b\lambda + w_1(r) + w_2(r)]$, we have 
\begin{eqnarray*}
& \hspace{-6.5 em} \widetilde{\mathrm{ET}}_\lambda^\alpha(\mu,\nu) 
= \int_{\calT} | \mu(\Lambda(x)) -  \nu(\Lambda(x))| \, \omega(dx) \\
& - \frac{b\lambda}{2}\big[ \mu(\calT) +  \nu(\calT)\big]  +   \big[w_i(r) +\frac{b\lambda}{2} -\alpha\big] |\mu(\calT)-\nu(\calT)| 
\end{eqnarray*}
with
$i :=1$ if $ \mu(\calT)\geq \nu(\calT)$ and $i:=2$ if $ \mu(\calT)< \nu(\calT)$. 
In particular,  the map $\alpha \longmapsto \widetilde{\mathrm{ET}}_\lambda^\alpha(\mu,\nu) $ is 
nonincreasing and 
\[
|\widetilde{\mathrm{ET}}_\lambda^{\alpha_1}(\mu,\nu)  -\widetilde{\mathrm{ET}}_\lambda^{\alpha_2}(\mu,\nu) |= |\alpha_1 -\alpha_2| |\mu(\calT)-\nu(\calT)|.
\]
\end{proposition}

Proof is placed in the Supplementary (\S A.7). 

It is also possible to use $\widetilde{\mathrm{ET}}_\lambda^\alpha(\mu,\nu) $ to upper or lower bound the
distance $\mathrm{ET}_\lambda(\mu,\nu)$ as follows:
\begin{proposition}\label{prop:bound_tildeET}  
Assume that  $\lambda \geq 0$ and the weights $w_1, w_2$ are $b$-Lipschitz w.r.t. $d_\calT$. Then, 
\[ \mathrm{ET}_\lambda(\mu,\nu) \leq  \widetilde{\mathrm{ET}}_\lambda^0(\mu,\nu).
\vspace{-4pt}
\]
In addition, if $[4L_{\calT} -\lambda]b \leq w_1(r)+w_2(r)$ where 
$
L_{\calT} : = \max_{x\in \calT } \omega([r,x])$, then
\[
\widetilde{\mathrm{ET}}_\lambda^\alpha(\mu,\nu)  \leq \mathrm{ET}_\lambda(\mu,\nu),
\vspace{-4pt}
\]
for every  $2b L_{\calT}\leq \alpha\leq \frac12 [b\lambda + w_1(r) + w_2(r)]$.
\end{proposition}

Proof is placed in the Supplementary (\S A.8). 

Analogous to Proposition~\ref{geodesic-space-part1}, we obtain: 
\begin{proposition}[Geometric structures of regularized metric $d_\alpha$]\label{geodesic-space-part2} 
Assume that  $\lambda, w_1(r), w_2(r)$ are nonnegative numbers. For   $0\leq \alpha< \frac{b\lambda}{2} +\min\{w_1(r), w_2(r)\}$, define
\begin{equation}\label{equ:dAlpha}
d_\alpha(\mu,\nu) := \widetilde{\mathrm{ET}}_\lambda^\alpha(\mu,\nu)  +\frac{b\lambda}{2}\big[ \mu(\calT) +  \nu(\calT)\big].
\vspace{-6pt}
\end{equation}
Then, we have 
\begin{enumerate}
\item[i)]  $d_\alpha(\mu +\sigma,\nu +\sigma) = d_\alpha(\mu,\nu)$, $\forall \sigma \in\calM(\calT)$.
\item[ii)]  $d_\alpha$ is a divergence and satisfies the triangle inequality $d_\alpha(\mu,\nu)\leq d_\alpha(\mu,\sigma) + d_\alpha(\sigma, \nu)$.
\item[iii)] If in addition $w_1(r)=w_2(r)$, then $(\calM(\calT), d_\alpha)$ is a complete metric space. Moreover, it is a geodesic space in the sense defined in part iii) of  Proposition~\ref{geodesic-space-part1} but with $d_\alpha$ replacing $d$.
\end{enumerate}
\end{proposition}
Proof is placed in the Supplementary (\S A.9).

\begin{proposition}\label{prop:negative_definite}  
With the same assumptions as in Proposition~\ref{pro:tildeET} for $\widetilde{\mathrm{ET}}_{\lambda}^{\alpha}$ and in Proposition~\ref{geodesic-space-part2} for $\dAlpha$, both $\widetilde{\mathrm{ET}}_{\lambda}^{\alpha}$ and $\dAlpha$ are negative definite.
\end{proposition}

Proof is placed in the Supplementary (\S A.10). 

From Proposition~\ref{prop:negative_definite} and following Berg et al. \cite{Berg84} (Theorem 3.2.2, p.74), given $t > 0$, the kernels $k_{\widetilde{\mathrm{ET}}_{\lambda}^{\alpha}}(\mu, \nu) := \exp\left(-t \widetilde{\mathrm{ET}}_{\lambda}^{\alpha}(\mu, \nu)\right)$ and $k_{\dAlpha}(\mu, \nu) := \exp\left(-t \dAlpha(\mu, \nu)\right)$ are positive definite.


\subsection{Tree-sliced Variants by Sampling Tree Metrics}

In most of practical applications, we usually do not have priori knowledge about tree structure for measures. Therefore, we need to choose or sample tree metrics from support data points for a given task. We use the tree metric sampling methods in \cite{LYFC}: (i) \textit{partition-based tree metric sampling} for a low-dimensional space, or (ii) \textit{clustering-based tree metric sampling} for a high-dimensional space. Moreover, those tree metric sampling methods are not only fast for computation\footnote{E.g., the complexity of the clustering-based tree metric is $\O(H_{\Tt} m \log \kappa)$ when we set $\kappa$ clusters for the farthest-point clustering \cite{gonzalez1985clustering}, and $H_{\Tt}$ for the predefined tree deepest level for $m$ input support data points.}, but also adaptive to the distribution of supports. We further propose the tree-sliced variants of the regularized EPT, computed by averaging the regularized EPT using those randomly sampled tree metrics. One advantage is to reduce the quantization effects or cluster sensitivity problems (i.e, support data points are quantized, or clustered into an adjacent hypercube, or cluster respectively) within the tree metric sampling procedure.

Although one can leverage tree metrics to approximate arbitrary metrics \cite{bartal1996probabilistic, bartal1998approximating, charikar1998approximating, fakcharoenphol2004tight, indyk2001algorithmic}, our goal is rather to sample tree metrics and use them as ground metrics in the regularized EPT, similar to TSW. 
Despite the fact that one-dimensional projections do not have interesting properties in terms of distortion viewpoints, they remain useful for SPOT (or SW, sliced-Gromov-Wasserstein~\cite{vayer2019sliced}). In the same vein, we believe that trees with high distortion are still useful for EPT, similar as in TSW. 
Moreover, one may not need to spend excessive effort to optimize $\ET_{\lambda}$ (in Equation~\eqref{equ:ETlambda}) for a randomly sampled tree metric since it can lead to overfitting within the computation of the EPT itself. Therefore, the proposed efficient regularization of EPT (e.g, $\widetilde{\ET}_{\lambda}^{\alpha}$ in Equation~\eqref{equ:tildeETlambda}) is not only fast for computation (i.e., closed-form), but also gives a benefit to overcome the overfitting problem within the computation of the EPT. 

\section{Discussion and Related Work}\label{sec:related_work}

One can leverage tree metrics to approximate arbitrary metrics for speeding up a computation \cite{bartal1996probabilistic, bartal1998approximating,  charikar1998approximating, fakcharoenphol2004tight, indyk2001algorithmic}. For instances, (i) Indyk and Thaper \cite{indyk2003fast} applied tree metrics (e.g., quadtree) to approximate OT with Euclidean cost metric for a fast image retrieval. (ii) Sato et al. \cite{sato2020fast} considered a generalized Kantorovich-Rubinstein discrepancy \cite{guittet2002extended, hanin1992kantorovich, lellmann2014imaging} with general weights for unbalanced OT, and used a quadtree as in \cite{indyk2003fast} to approximate the proposed distance
via a dynamic programming with infinitely many states.  They then derived an efficient algorithm with a quasi-linear time complexity to speed up the dynamic programming computation by leveraging high-level programming techniques. However, such approximations following the approach of \cite{indyk2003fast} result in large distortions in high dimensional spaces \cite{naor2007planar}.

\section{Experiments}\label{sec:experiments}


In this section, we first illustrate that $\widetilde{\ET}_{\lambda}^{\alpha}$ (Equation~\eqref{equ:tildeETlambda}) is an efficient approximation for $\ET_{\lambda}$ (Equation \eqref{equ:ETlambda}). Then, we evaluate our proposed $\widetilde{\ET}_{\lambda}^{\alpha}$ and $\dAlpha$ (Equation~\eqref{equ:dAlpha}) for comparing measures in document classification with word embedding and topological data analysis (TDA). Experiments are evaluated with Intel Xeon CPU E7-8891v3 2.80GHz and 256GB RAM.

\paragraph{Documents with word embedding.} We consider each document as a measure where each word is regarded as a point in the support with a unit mass. Following \cite{kusner2015word, LYFC}, we applied the \textit{word2vec} word embedding \cite{mikolov2013distributed}, pretrained on Google News\footnote{https://code.google.com/p/word2vec} containing about 3 millions words/phrases. Each word/phrase in a document is mapped into a vector in $\RR^{300}$. We removed all SMART stop word \cite{salton1988term}, and dropped words in documents if they are not available in the pretrained \textit{word2vec}.

\paragraph{Geometric structured data via persistence diagrams in TDA.} TDA has recently emerged in machine learning community as a powerful tool to analyze geometric structured data such as material data, or linked twist maps \cite{adams2017persistence, lacombe2018large, le2018persistence}. TDA applies algebraic topology methods (e.g., persistence homology) to extract robust topological features (e.g., connected components, rings, cavities) and output a multiset of 2-dimensional points, i.e., persistence diagram (PD). The coordinates of a 2-dimensional point in PD are corresponding to the birth and death time of a particular topological feature. Therefore, each point in PD summarizes a life span of a topological feature. We can regard PD as measures where each 2-dimensional point is considered as a point in the support with a unit mass.

\paragraph{Tree metric sampling.} In our experiments, we do not have priori knowledge about tree metrics for neither word embeddings in documents nor 2-dimensional points in persistence diagrams (PDs). To compute the EPT, e.g., $\widetilde{\ET}_{\lambda}^{\alpha}$ and its associated $d_{\alpha}$, we considered $n_s$ randomized tree metrics. We employed the clustering-based tree metric sampling for word embeddings in documents (i.e., high-dimensional space $\RR^{300}$), while we used the partition-based tree metric sampling for 2-dimensional points in PDs (i.e., low-dimensional space $\RR^2$). Those tree metric sampling methods are built with a predefined deepest level $H_{\Tt}$ of tree $\Tt$ as a stopping condition as in \cite{LYFC}.

\paragraph{Baselines and setup.} We considered 2 baselines based on OT theory for measures with different masses: (i) Sinkhorn-UOT~\cite{chizat2018scaling, frogner2015learning}, and (ii) SPOT \cite{bonneel2019spot}. Following \cite{LYFC}, we apply the kernel approach in the form $\exp(-t\bar{d})$  with SVM for document classification with word embedding. Here,   $\bar{d}$ is a discrepancy between measures and $t>0$.
We also employed this kernel approach for various tasks in TDA, e.g., orbit recognition and object shape classification with SVM, as well as change point detection for material data analysis with kernel Fisher discriminant ratio (KFDR) \cite{harchaoui2009kernel}. While kernels for $\widetilde{\ET}_{\lambda}^{\alpha}$ and $d_{\alpha}$ are positive definite, kernels for Sinkhorn-UOT and SPOT are empirically indefinite\footnote{In practice, we observed negative eigenvalues of some Gram matrices corresponding to kernels for Sinkhorn-UOT and SPOT.}. When kernels are indefinite, we regularized for the corresponding Gram matrices by adding a sufficiently large diagonal term as in \cite{Cuturi-2013-Sinkhorn, LYFC}. For SVM, we randomly split each dataset into $70\%/30\%$ for training and test with 10 repeats. Typically, we choose hyper-parameters via cross validation, choose $1/t$ from $\{q_{10}, q_{20}, q_{50}\}$ where $q_s$ is the $s\%$ quantile of a subset of corresponding discrepancies observed on a training set, use 1-vs-1 strategy with Libsvm\footnote{https://www.csie.ntu.edu.tw/$\sim$cjlin/libsvm/} for multi-class classification, and choose SVM regularization from $\left\{10^{-2:1:2}\right\}$. For Sinkhorn-UOT, we select the entropic regularization from $\left\{0.01, 0.05, 0.1, 0.5, 1\right\}$. Following Proposition~\ref{prop:bound_tildeET}, we take $\alpha = 0$ for $\widetilde{\ET}_{\lambda}^{\alpha}$ and $\dAlpha$ in all our experiments.

\subsection{Efficient Approximation of $\widetilde{\ET}_{\lambda}^{0}$ for $\ET_{\lambda}$}


\begin{wrapfigure}{r}{0.2\textwidth}
 \vspace{-25pt}
  \begin{center}
    \includegraphics[width=0.2\textwidth]{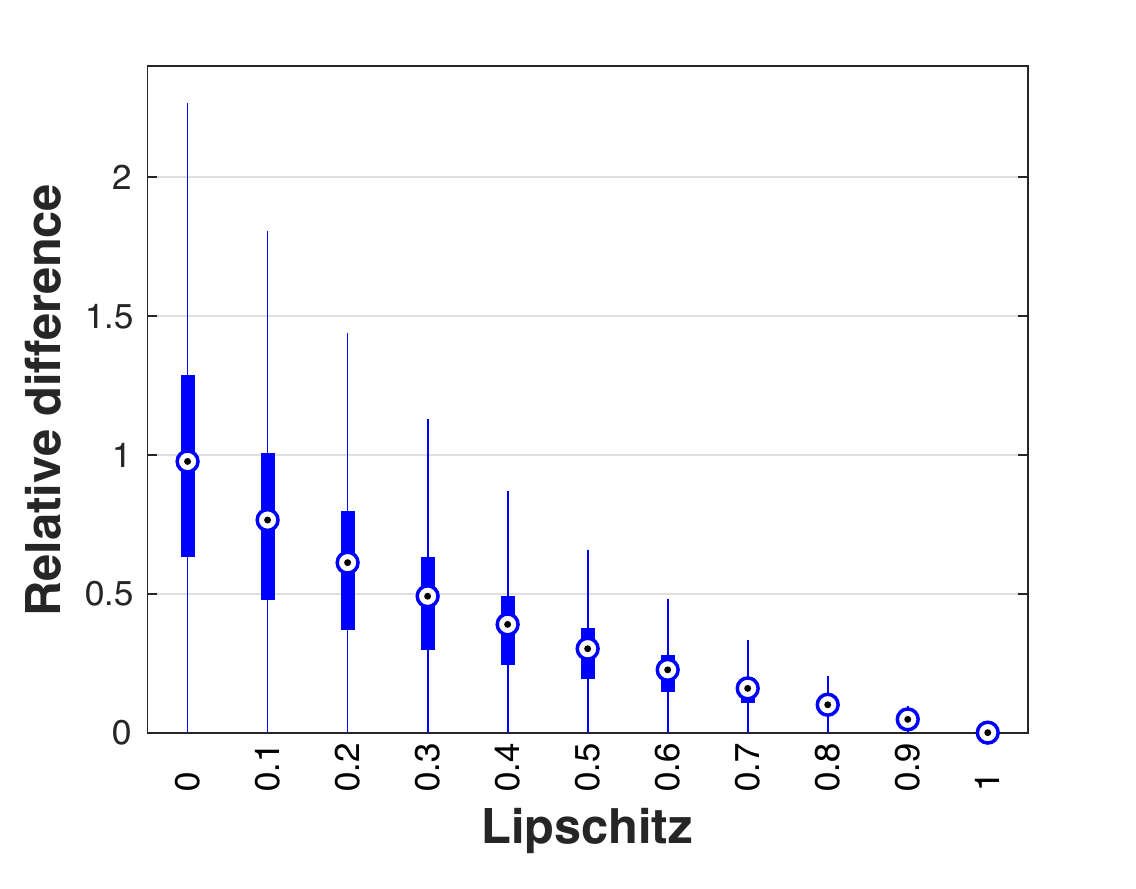}
  \end{center}
  \vspace{-12pt}
  \caption{Relative difference between $\tildeET_{\lambda}^{0}$ and $\ET_{\lambda}$ w.r.t. Lipschitz const. of $w_1, w_2$.} \label{fg:Diff_KT_tET_Lipchitz}
  \vspace{-4pt}
\end{wrapfigure}

We randomly sample 500K pairs of documents in $\texttt{TWITTER}$ dataset. Following Proposition~\ref{distance-agree}, we compute $\ET_{\lambda}$ via the corresponding $\KT$ (Equation~\eqref{P2}). Our goal is to compare $\widetilde{\ET}_{\lambda}^{0}$ to $\ET_{\lambda}$.

\begin{wrapfigure}{r}{0.2\textwidth}
 \vspace{-26pt}
  \begin{center}
    \includegraphics[width=0.2\textwidth]{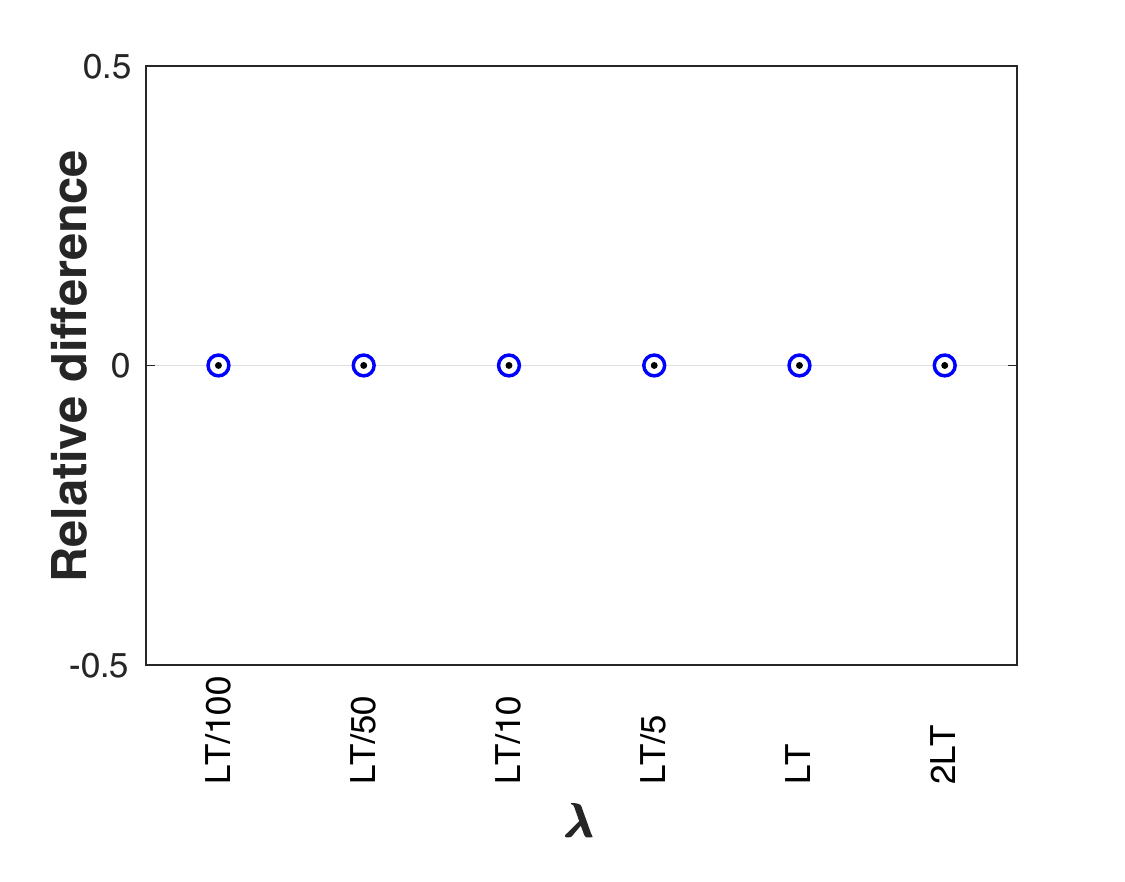}
  \end{center}
  \vspace{-12pt}
  \caption{Relative difference between $\tildeET_{\lambda}^{0}$ and $\ET_\lambda$ w.r.t. $\lambda$ when $a_1=b$. (LT := $L_{\Tt}$)} \label{fg:Diff_KT_tET_Lambda_LargestLipschitz}
  \vspace{-30pt}
\end{wrapfigure}

\textbf{Change Lipschitz constants.} We choose $w_1(x) = w_2(x) = a_1 d_{\Tt}(r, x) + a_0$, and set $\lambda=b=1$, $a_0=1$. In particular,  $a_1 \in [0, b]$ since $w_1, w_2$ are $b$-Lipschitz functions (see Corollary \ref{cor:duality} and Remark \ref{rem:bLipschitz}). We illustrate the relative difference $(\tildeET_{\lambda}^{0} - \ET_\lambda)/\ET_\lambda$ when $a_1$ is changed in $[0, b]$ in Figure~\ref{fg:Diff_KT_tET_Lipchitz}. We observe that when $a_1$ is close to $b$ (i.e., the Lipschitz constants of $w_1, w_2$ are close to $b$), $\tildeET_{\lambda}^{0}$ becomes closer to $\ET_\lambda$. When $a_1=b$, the values of $\tildeET_{\lambda}^{0}$ is almost identical to $\ET_\lambda$.

\textbf{Change $\lambda$.} From the results in Figure~\ref{fg:Diff_KT_tET_Lipchitz}, we set $a_1=b$ to investigate the relative different between $\tildeET_{\lambda}^{0}$ and $\ET_\lambda$ when $\lambda$ is changed. As illustrated in Figure~\ref{fg:Diff_KT_tET_Lambda_LargestLipschitz}, $\tildeET_{\lambda}^{0}$ is almost identical to $\ET_\lambda$ regardless the value of $\lambda$ when $a_1 = b$.

\subsection{Document Classification with Word Embedding}\label{sec:Doc}

We consider 4 datasets: \texttt{TWITTER}, \texttt{RECIPE}, \texttt{CLASSIC} and \texttt{AMAZON} for document classification with word embedding. Statistical characteristics of these datasets are summarized in Figure~\ref{fg:result_DOC}.

 \begin{figure}
 \vspace{-12pt}
  \begin{center}
    \includegraphics[width=0.48\textwidth]{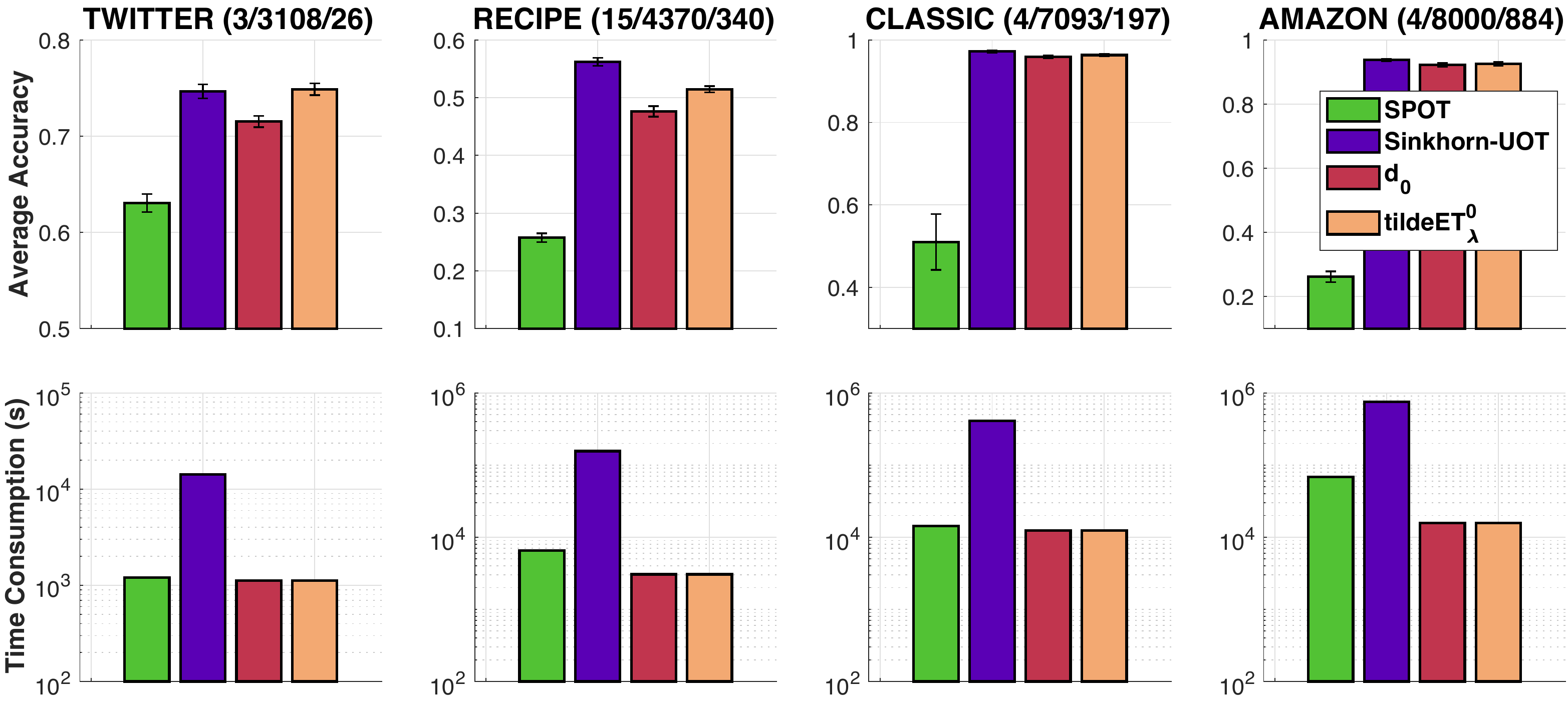}
  \end{center}
  \vspace{-12pt}
  \caption{SVM results on document classification (row 1), and corresponding time consumption of kernel matrices (row 2). For each dataset, the numbers in the parenthesis are  respectively the number of classes, the number of documents, and the maximum number of unique words for each document.}
  \label{fg:result_DOC}
  \vspace{-18pt}
\end{figure}

\subsection{Topological Data Analysis (TDA)}\label{sec:TDA}

\begin{wrapfigure}{r}{0.25\textwidth}
 \vspace{-26pt}
  \begin{center}
    \includegraphics[width=0.24\textwidth]{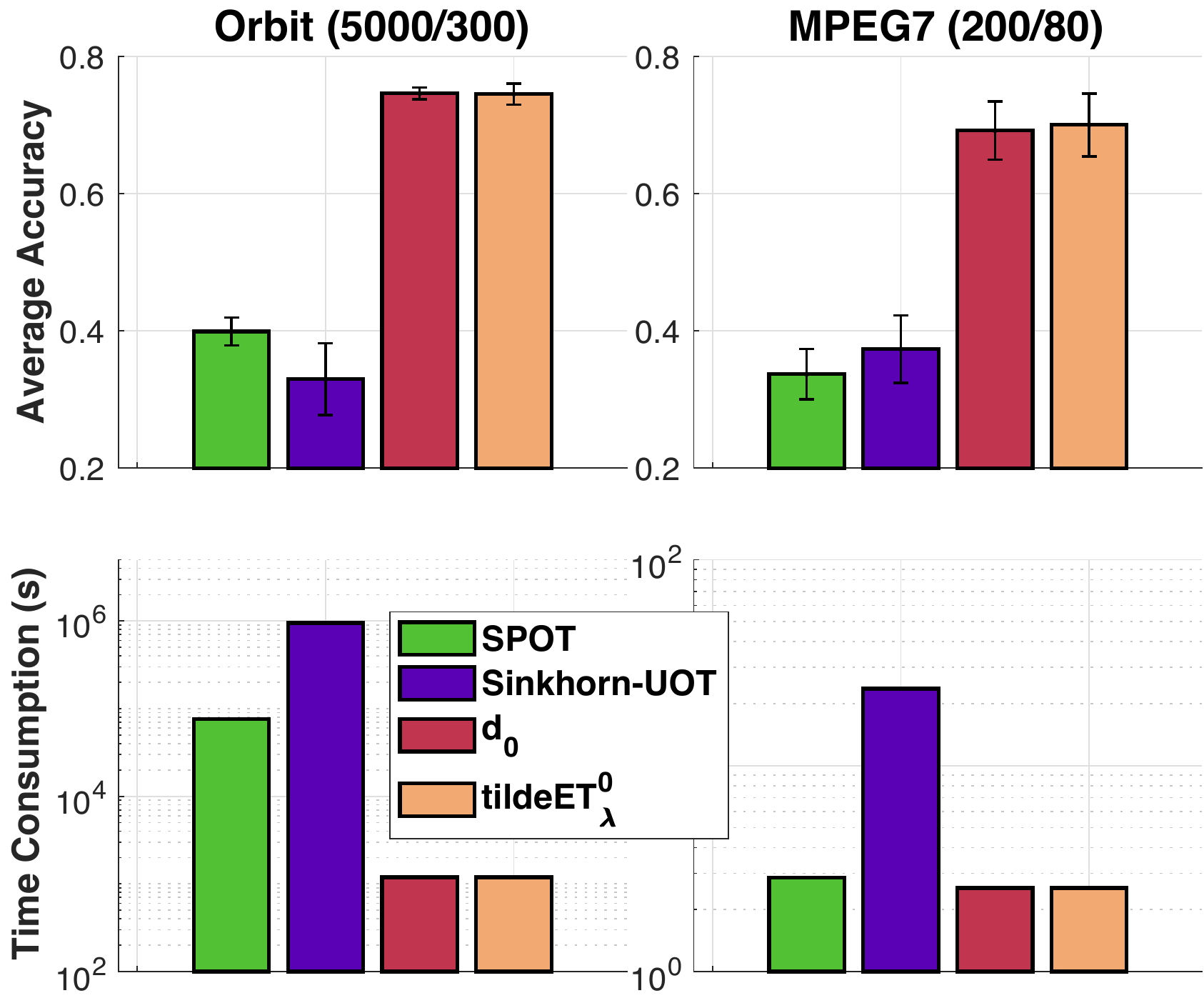}
  \end{center}
  \vspace{-12pt}
  \caption{SVM results for TDA (row 1), and corresponding time consumption of kernel matrices (row 2). For each dataset, the numbers in the parenthesis are 
  respectively the number of PDs, and the maximum number of points in PD.}
  \label{fg:result_TDA}
  \vspace{-10pt}
\end{wrapfigure}

\subsubsection{Orbit Recognition}
We considered a synthesized dataset, proposed by Adams et al. \cite{adams2017persistence}, for link twist map which is a discrete dynamical system to model flows in DNA microarrays \cite{hertzsch2007dna}. There are 5 classes of orbits. Following \cite{le2018persistence}, we generated 1000 orbits for each class of orbits, and each orbit has 1000 points. We used the 1-dimensional topological features for PD extracted with Vietoris-Rips complex filtration \cite{edelsbrunner2008persistent}. 



\subsubsection{Object Shape Classification} 
We evaluated our approach for object shape classification on a subset of \texttt{MPEG7} dataset \cite{latecki2000shape} containing 10 classes where each class has 20 samples as in \cite{le2018persistence}. For simplicity, we followed \cite{le2018persistence} to extract $1$-dimensional topological features for PD with Vietoris-Rips complex filtration\footnote{A more complicated and advanced filtration for this task is considered in \cite{turner2014persistent}.} \cite{edelsbrunner2008persistent}. 


\subsubsection{Change Point Detection for Material Analysis}

We applied our approach on change point detection for material analysis with KFDR as a statistical score on granular packing system (GPS) \cite{francois2013geometrical} and SiO$_2$ \cite{nakamura2015persistent} datasets. Statistical characteristics of these datasets are summarized in Figure~\ref{fg:RFDR_ChangePoint}. Following \cite{le2018persistence}, we set $10^{-3}$ for the regularization parameter in KFDR and used the ball model filtration to extract 2-dimensional topological features for PD in GPS dataset, and 1-dimensional topological features for PD in SiO$_2$ dataset. Note that we omit the baseline kernel for Sinkhorn-UOT in this application since its computation of Sinkhorn-UOT is out of memory.

 \begin{figure}
 \vspace{-10pt}
  \begin{center}
    \includegraphics[width=0.4\textwidth]{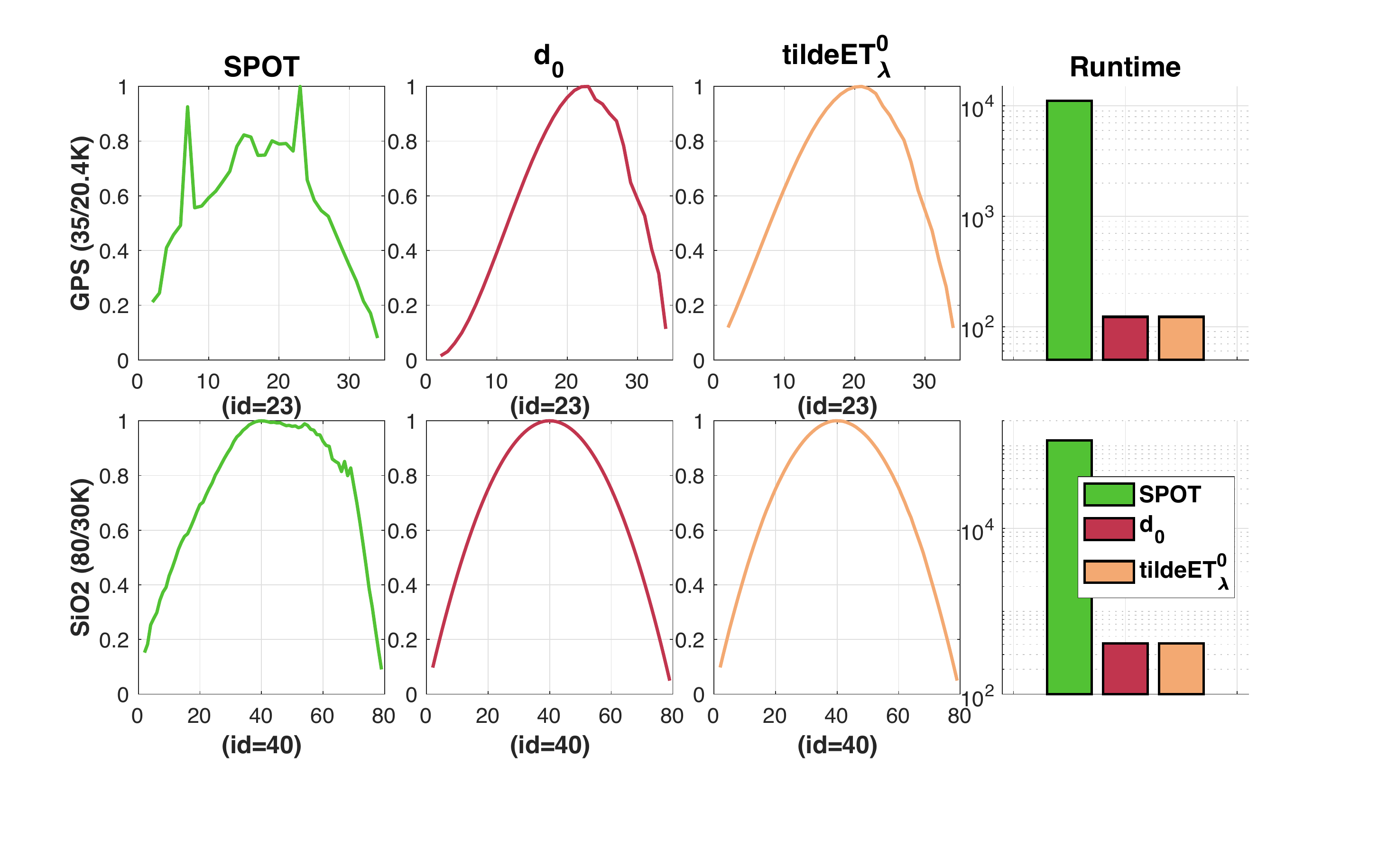}
  \end{center}
  \vspace{-12pt}
  \caption{KFDR graphs and time consumption of kernel matrices for change point detection. For each dataset, the numbers in the parenthesis are respectively the number of PDs, and the maximum number of points in PD.}
  \label{fg:RFDR_ChangePoint}
  \vspace{-14pt}
\end{figure}

We illustrate the KFDR graphs for both datasets in Figure~\ref{fg:RFDR_ChangePoint}. For GPS dataset, all kernel approaches get the change point at the index 23 which supports the observation (corresponding id = 23) in \cite{anonymous72}. For SiO$_2$ dataset, all kernel approaches get the change point in a supported range ($35 \le \text{id} \le 50)$, obtained by a traditional physical approach \cite{elliott1983physics}. The KFDR results of kernels corresponding to $d_{0}$ and $\tildeET_{\lambda}^{0}$ compare favorably with those of kernel for SPOT.

\subsection{Results of SVM, Time Consumption and Discussions}

\begin{wrapfigure}{r}{0.25\textwidth}
 \vspace{-13pt}
  \begin{center}
    \includegraphics[width=0.25\textwidth]{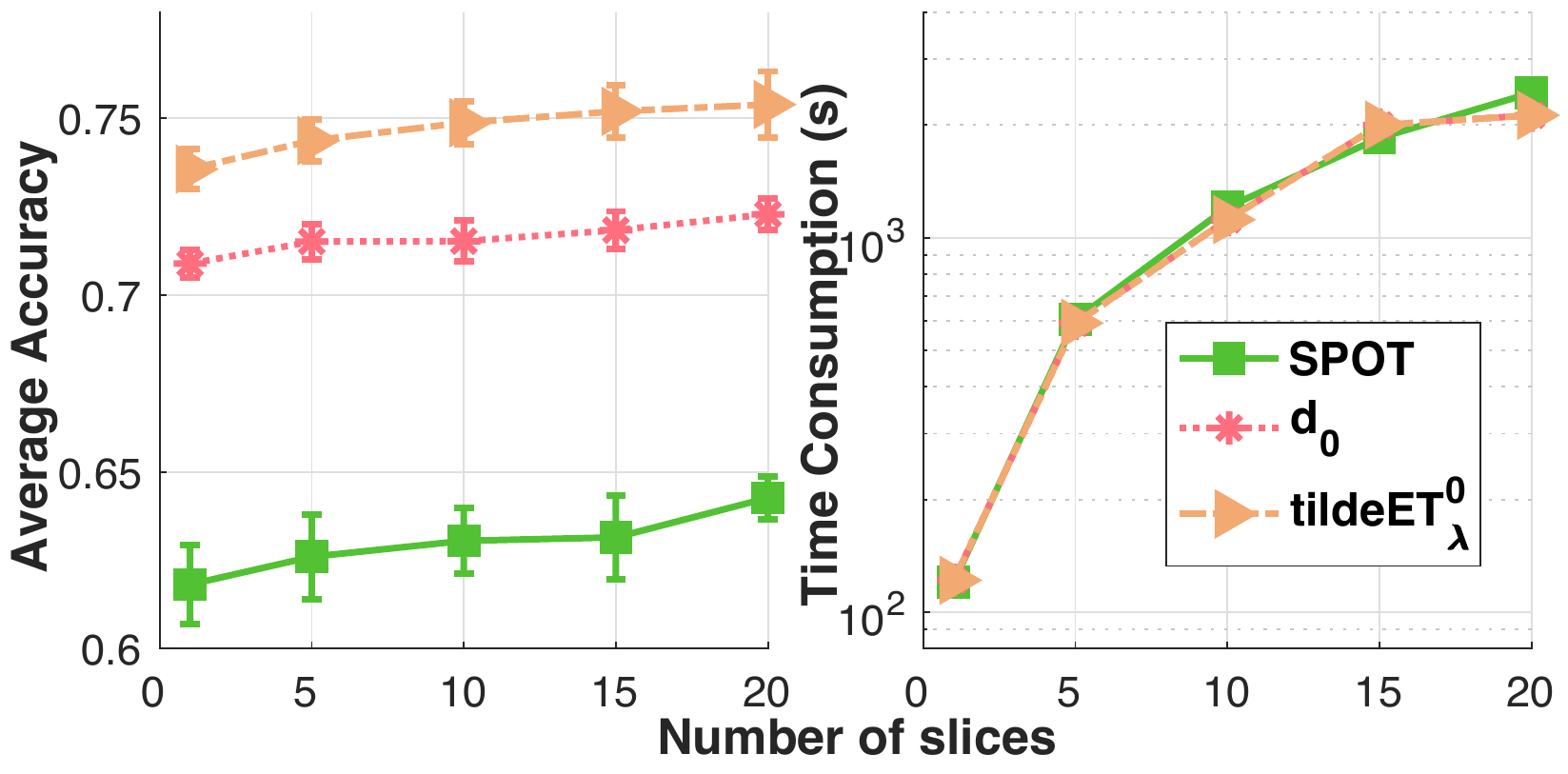}
  \end{center}
  \vspace{-12pt}
  \caption{SVM results and time consumption for corresponding kernel matrices in \texttt{TWITTER} dataset w.r.t. the number of (tree) slices.}
  \label{fg:TWITTER_slices}
  \vspace{-12pt}
\end{wrapfigure}

We illustrate the results of SVM and time consumption of kernel matrices in document classification with word embedding and TDA in Figure~\ref{fg:result_DOC} and Figure~\ref{fg:result_TDA} respectively. The performances of kernels for $\tildeETlambda$ and $d_0$ outperform those of kernels for SPOT. They also outperform those of kernels for Sinkhorn-UOT on TDA, and are comparative on document classification. The fact that SPOT uses the 1-dimensional projection for support data points may limit its ability to capture high-dimensional structure in data distributions \cite{LYFC, liutkus2019sliced}. The regularized EPT remedies this problem by leveraging the tree metrics which have more flexibility and degrees of freedom (e.g., choose a tree rather than a line). In addition, while kernels for $\tildeETlambda$ and $d_0$ are positive definite, kernels for SPOT and Sinkhorn-UOT are empirically indefinite. The indefiniteness of kernels may affect their performances in some applications, e.g., kernels for Sinkhorn-UOT work well for document classification with word embedding, but perform poorly in TDA applications. There are also   similar   observations in \cite{LYFC}. Additionally, we illustrate a trade-off between performances and computational time for different number of (tree) slices in \texttt{TWITTER} dataset in Figure~\ref{fg:TWITTER_slices}. The performances are usually improved with more slices, but with a trade-off about the computational time. In applications, we observed that a good trade off is about $n_s = 10$ slices.

\textbf{Tree metric sampling.} Time consumption for the tree metric sampling is negligible in applications. With the predefined tree deepest level $H_{\Tt}=6$ and tree branches $\kappa=4$ as in~\cite{LYFC}, it took $1.5, 11.0, 17.5, 20.5$ seconds for \texttt{TWITTER}, \texttt{RECIPE}, \texttt{CLASSIC}, \texttt{AMAZON} datasets respectively, and $21.0, 0.1$ seconds for \texttt{Orbit}, \texttt{MPEG7} datasets respectively.

\textbf{$\widetilde{\ET}_{\lambda}^{0}$ versus $\ET_{\lambda}$.} We also compare $\widetilde{\ET}_{\lambda}^{0}$ and $\ET_{\lambda}$ (or KT) in \texttt{TWITTER} dataset for document classification, and in \texttt{MPEG7} dataset for object shape recognition in TDA. The performances of $\widetilde{\ET}_{\lambda}^{0}$ and $\ET_{\lambda}$ are identical (i.e., their kernel matrices are almost the same for those datasets), but $\widetilde{\ET}_{\lambda}^{0}$ is faster than $\ET_{\lambda}$ about 11 times in \texttt{TWITTER} dataset, and 81 times in \texttt{MPEG7} dataset when $n_s = 10$ slices. 

Further results are placed in the supplementary (\S B).


\section{Conclusion}\label{sec:conclusion}

We have developed a rigorous theory for the entropy partial transport (EPT) problem for nonnegative measures on a tree having different masses. We show that the EPT problem is equivalent to a standard complete OT problem on a suitable one-node extended tree which allows us to develop its dual formulation. By leveraging the dual problem, we proposed efficient novel regularization for EPT which yields closed-form solution for a fast computation and negative definiteness---an important property to build positive definite kernels required in many kernel-dependent frameworks. Moreover, our regularization also provides effective approximations in applications. We further derive tree-sliced variants of the regularized EPT for practical applications without priori knowledge about a tree structure for measures. The question about sampling efficient tree metrics for the tree-sliced variants from data points is left for future work. 

\balance
\bibliographystyle{plain}
\bibliography{aistats2021}

\begin{thebibliography}{10}

\bibitem{adams2017persistence}
Henry Adams, Tegan Emerson, Michael Kirby, Rachel Neville, Chris Peterson,
  Patrick Shipman, Sofya Chepushtanova, Eric Hanson, Francis Motta, and Lori
  Ziegelmeier.
\newblock Persistence images: A stable vector representation of persistent
  homology.
\newblock {\em Journal of Machine Learning Research}, 18(1):218--252, 2017.

\bibitem{anonymous72}
Anonymous.
\newblock What is random packing?
\newblock {\em Nature}, 239:488--489, 1972.

\bibitem{bartal1996probabilistic}
Yair Bartal.
\newblock Probabilistic approximation of metric spaces and its algorithmic
  applications.
\newblock In {\em Proceedings of 37th Conference on Foundations of Computer
  Science}, pages 184--193, 1996.

\bibitem{bartal1998approximating}
Yair Bartal.
\newblock On approximating arbitrary metrices by tree metrics.
\newblock In {\em ACM Symposium on Theory of Computing (STOC)}, volume~98,
  pages 161--168, 1998.

\bibitem{benamou2003numerical}
Jean-David Benamou.
\newblock Numerical resolution of an “unbalanced” mass transport problem.
\newblock {\em ESAIM: Mathematical Modelling and Numerical
  Analysis-Mod{\'e}lisation Math{\'e}matique et Analyse Num{\'e}rique},
  37(5):851--868, 2003.

\bibitem{Berg84}
C.~Berg, J.~P.~R. Christensen, and P.~Ressel, editors.
\newblock {\em Harmonic analysis on semigroups}.
\newblock Springer-Verglag, New York, 1984.

\bibitem{bonneel2019spot}
Nicolas Bonneel and David Coeurjolly.
\newblock Spot: sliced partial optimal transport.
\newblock {\em ACM Transactions on Graphics (TOG)}, 38(4):1--13, 2019.

\bibitem{bonneel2015sliced}
Nicolas Bonneel, Julien Rabin, Gabriel Peyr{\'e}, and Hanspeter Pfister.
\newblock Sliced and radon wasserstein barycenters of measures.
\newblock {\em Journal of Mathematical Imaging and Vision}, 51(1):22--45, 2015.

\bibitem{bunne2019}
Charlotte Bunne, David Alvarez-Melis, Andreas Krause, and Stefanie Jegelka.
\newblock {Learning Generative Models across Incomparable Spaces}.
\newblock In {\em International Conference on Machine Learning (ICML)},
  volume~97, 2019.

\bibitem{CM}
Luis~A Caffarelli and Robert~J McCann.
\newblock Free boundaries in optimal transport and monge-ampere obstacle
  problems.
\newblock {\em Annals of mathematics}, pages 673--730, 2010.

\bibitem{charikar1998approximating}
Moses Charikar, Chandra Chekuri, Ashish Goel, Sudipto Guha, and Serge Plotkin.
\newblock Approximating a finite metric by a small number of tree metrics.
\newblock In {\em Proceedings 39th Annual Symposium on Foundations of Computer
  Science (FOCS)}, pages 379--388, 1998.

\bibitem{chizat2018scaling}
Lenaic Chizat, Gabriel Peyr{\'e}, Bernhard Schmitzer, and Fran{\c{c}}ois-Xavier
  Vialard.
\newblock Scaling algorithms for unbalanced optimal transport problems.
\newblock {\em Mathematics of Computation}, 87(314):2563--2609, 2018.

\bibitem{chung2019duality}
Nhan-Phu Chung and Thanh-Son Trinh.
\newblock Duality and quotient spaces of generalized wasserstein spaces.
\newblock {\em arXiv preprint arXiv:1904.12461}, 2019.

\bibitem{courty2017joint}
Nicolas Courty, R{\'e}mi Flamary, Amaury Habrard, and Alain Rakotomamonjy.
\newblock Joint distribution optimal transportation for domain adaptation.
\newblock In {\em Advances in Neural Information Processing Systems}, pages
  3730--3739, 2017.

\bibitem{Cuturi-2013-Sinkhorn}
M.~Cuturi.
\newblock Sinkhorn distances: {L}ightspeed computation of optimal transport.
\newblock In {\em Advances in Neural Information Processing Systems}, pages
  2292--2300, 2013.

\bibitem{edelsbrunner2008persistent}
Herbert Edelsbrunner and John Harer.
\newblock Persistent homology-a survey.
\newblock {\em Contemporary mathematics}, 453:257--282, 2008.

\bibitem{elliott1983physics}
Stephen~Richard Elliott.
\newblock Physics of amorphous materials.
\newblock {\em Longman Group}, 1983.

\bibitem{EM}
Steven~N Evans and Frederick~A Matsen.
\newblock The phylogenetic kantorovich--rubinstein metric for environmental
  sequence samples.
\newblock {\em Journal of the Royal Statistical Society: Series B (Statistical
  Methodology)}, 74(3):569--592, 2012.

\bibitem{fakcharoenphol2004tight}
Jittat Fakcharoenphol, Satish Rao, and Kunal Talwar.
\newblock A tight bound on approximating arbitrary metrics by tree metrics.
\newblock {\em Journal of Computer and System Sciences}, 69(3):485--497, 2004.

\bibitem{figalli2010optimal}
Alessio Figalli.
\newblock The optimal partial transport problem.
\newblock {\em Archive for rational mechanics and analysis}, 195(2):533--560,
  2010.

\bibitem{francois2013geometrical}
Nicolas Francois, Mohammad Saadatfar, R~Cruikshank, and A~Sheppard.
\newblock Geometrical frustration in amorphous and partially crystallized
  packings of spheres.
\newblock {\em Physical review letters}, 111(14):148001, 2013.

\bibitem{frogner2015learning}
Charlie Frogner, Chiyuan Zhang, Hossein Mobahi, Mauricio Araya, and Tomaso~A
  Poggio.
\newblock Learning with a wasserstein loss.
\newblock In {\em Advances in neural information processing systems}, pages
  2053--2061, 2015.

\bibitem{gangbo2019unnormalized}
Wilfrid Gangbo, Wuchen Li, Stanley Osher, and Michael Puthawala.
\newblock Unnormalized optimal transport.
\newblock {\em Journal of Computational Physics}, 399:108940, 2019.

\bibitem{gonzalez1985clustering}
Teofilo~F Gonzalez.
\newblock Clustering to minimize the maximum intercluster distance.
\newblock {\em Theoretical Computer Science}, 38:293--306, 1985.

\bibitem{guittet2002extended}
Kevin Guittet.
\newblock Extended kantorovich norms: a tool for optimization.
\newblock {\em INRIA report}, 2002.

\bibitem{hanin1992kantorovich}
Leonid~G Hanin.
\newblock Kantorovich-rubinstein norm and its application in the theory of
  lipschitz spaces.
\newblock {\em Proceedings of the American Mathematical Society},
  115(2):345--352, 1992.

\bibitem{harchaoui2009kernel}
Zaid Harchaoui, Eric Moulines, and Francis~R Bach.
\newblock Kernel change-point analysis.
\newblock In {\em Advances in neural information processing systems}, pages
  609--616, 2009.

\bibitem{hertzsch2007dna}
Jan-Martin Hertzsch, Rob Sturman, and Stephen Wiggins.
\newblock Dna microarrays: design principles for maximizing ergodic, chaotic
  mixing.
\newblock {\em Small}, 3(2):202--218, 2007.

\bibitem{indyk2001algorithmic}
Piotr Indyk.
\newblock Algorithmic applications of low-distortion geometric embeddings.
\newblock In {\em Proceedings 42nd IEEE Symposium on Foundations of Computer
  Science (FOCS)}, pages 10--33, 2001.

\bibitem{indyk2003fast}
Piotr Indyk and Nitin Thaper.
\newblock Fast image retrieval via embeddings.
\newblock In {\em International workshop on statistical and computational
  theories of vision}, volume~2, page~5, 2003.

\bibitem{janati2019wasserstein}
Hicham Janati, Marco Cuturi, and Alexandre Gramfort.
\newblock Wasserstein regularization for sparse multi-task regression.
\newblock In {\em The 22nd International Conference on Artificial Intelligence
  and Statistics}, pages 1407--1416, 2019.

\bibitem{janati2020entropic}
Hicham Janati, Boris Muzellec, Gabriel Peyr{\'e}, and Marco Cuturi.
\newblock Entropic optimal transport between (unbalanced) gaussian measures has
  a closed form.
\newblock In {\em Advances in neural information processing systems}, 2020.

\bibitem{kloeckner2015geometric}
Beno{\^\i}t~R Kloeckner.
\newblock A geometric study of {W}asserstein spaces: ultrametrics.
\newblock {\em Mathematika}, 61(1):162--178, 2015.

\bibitem{kolouri2019generalized}
Soheil Kolouri, Kimia Nadjahi, Umut Simsekli, Roland Badeau, and Gustavo Rohde.
\newblock Generalized sliced wasserstein distances.
\newblock In {\em Advances in Neural Information Processing Systems}, pages
  261--272, 2019.

\bibitem{kusano2017kernel}
Genki Kusano, Kenji Fukumizu, and Yasuaki Hiraoka.
\newblock Kernel method for persistence diagrams via kernel embedding and
  weight factor.
\newblock {\em The Journal of Machine Learning Research}, 18(1):6947--6987,
  2017.

\bibitem{kusner2015word}
Matt Kusner, Yu~Sun, Nicholas Kolkin, and Kilian Weinberger.
\newblock From word embeddings to document distances.
\newblock In {\em International conference on machine learning}, pages
  957--966, 2015.

\bibitem{lacombe2018large}
Th{\'e}o Lacombe, Marco Cuturi, and Steve Oudot.
\newblock Large scale computation of means and clusters for persistence
  diagrams using optimal transport.
\newblock In {\em Advances in Neural Information Processing Systems}, pages
  9770--9780, 2018.

\bibitem{latecki2000shape}
Longin~Jan Latecki, Rolf Lakamper, and T~Eckhardt.
\newblock Shape descriptors for non-rigid shapes with a single closed contour.
\newblock In {\em Proceedings of the IEEE Conference on Computer Vision and
  Pattern Recognition (CVPR)}, volume~1, pages 424--429, 2000.

\bibitem{lavenant2018dynamical}
Hugo Lavenant, Sebastian Claici, Edward Chien, and Justin Solomon.
\newblock Dynamical optimal transport on discrete surfaces.
\newblock In {\em SIGGRAPH Asia 2018 Technical Papers}, page 250. ACM, 2018.

\bibitem{le2021fba}
Tam Le, Nhat Ho, and Makoto Yamada.
\newblock Flow-based alignment approaches for probability measures in different
  spaces.
\newblock In {\em International Conference on Artificial Intelligence and
  Statistics (AISTATS)}. 2021.

\bibitem{le2019wb}
Tam Le, Viet Huynh, Nhat Ho, Dinh Phung, and Makoto Yamada.
\newblock On scalable variant of wasserstein barycenter.
\newblock {\em arXiv preprint arXiv:1910.04483}, 2019.

\bibitem{le2018persistence}
Tam Le and Makoto Yamada.
\newblock Persistence {F}isher kernel: A {R}iemannian manifold kernel for
  persistence diagrams.
\newblock In {\em Advances in Neural Information Processing Systems}, pages
  10007--10018, 2018.

\bibitem{LYFC}
Tam Le, Makoto Yamada, Kenji Fukumizu, and Marco Cuturi.
\newblock Tree-sliced variants of {W}asserstein distances.
\newblock In {\em Advances in neural information processing systems}, pages
  12283--12294, 2019.

\bibitem{lee2019parallel}
John Lee, Nicholas~P Bertrand, and Christopher~J Rozell.
\newblock Parallel unbalanced optimal transport regularization for large scale
  imaging problems.
\newblock {\em arXiv preprint arXiv:1909.00149}, 2019.

\bibitem{lellmann2014imaging}
Jan Lellmann, Dirk~A Lorenz, Carola Schonlieb, and Tuomo Valkonen.
\newblock Imaging with kantorovich--rubinstein discrepancy.
\newblock {\em SIAM Journal on Imaging Sciences}, 7(4):2833--2859, 2014.

\bibitem{Liero2018}
Matthias Liero, Alexander Mielke, and Giuseppe Savar{\'e}.
\newblock Optimal entropy-transport problems and a new hellinger--kantorovich
  distance between positive measures.
\newblock {\em Inventiones mathematicae}, 211(3):969--1117, 2018.

\bibitem{liutkus2019sliced}
Antoine Liutkus, Umut Simsekli, Szymon Majewski, Alain Durmus, and
  Fabian-Robert St{\"o}ter.
\newblock Sliced-wasserstein flows: Nonparametric generative modeling via
  optimal transport and diffusions.
\newblock In {\em International Conference on Machine Learning}, pages
  4104--4113, 2019.

\bibitem{memoli2021ultrametric}
Facundo M{\'e}moli, Axel Munk, Zhengchao Wan, and Christoph Weitkamp.
\newblock The ultrametric gromov-wasserstein distance.
\newblock {\em arXiv preprint arXiv:2101.05756}, 2021.

\bibitem{mena2019statistical}
Gonzalo Mena and Jonathan Niles-Weed.
\newblock Statistical bounds for entropic optimal transport: sample complexity
  and the central limit theorem.
\newblock In {\em Advances in Neural Information Processing Systems}, pages
  4541--4551, 2019.

\bibitem{mikolov2013distributed}
Tomas Mikolov, Ilya Sutskever, Kai Chen, Greg~S Corrado, and Jeff Dean.
\newblock Distributed representations of words and phrases and their
  compositionality.
\newblock In {\em Advances in neural information processing systems}, pages
  3111--3119, 2013.

\bibitem{nadjahi2019asymptotic}
Kimia Nadjahi, Alain Durmus, Umut Simsekli, and Roland Badeau.
\newblock Asymptotic guarantees for learning generative models with the
  sliced-wasserstein distance.
\newblock In {\em Advances in Neural Information Processing Systems}, pages
  250--260, 2019.

\bibitem{nakamura2015persistent}
Takenobu Nakamura, Yasuaki Hiraoka, Akihiko Hirata, Emerson~G Escolar, and
  Yasumasa Nishiura.
\newblock Persistent homology and many-body atomic structure for medium-range
  order in the glass.
\newblock {\em Nanotechnology}, 26(30):304001, 2015.

\bibitem{naor2007planar}
Assaf Naor and Gideon Schechtman.
\newblock Planar {E}arthmover is not in {L}\_1.
\newblock {\em SIAM Journal on Computing}, 37(3):804--826, 2007.

\bibitem{peyre2019computational}
Gabriel Peyr{\'e} and Marco Cuturi.
\newblock Computational optimal transport.
\newblock {\em Foundations and Trends{\textregistered} in Machine Learning},
  11(5-6):355--607, 2019.

\bibitem{pham2020unbalanced}
Khiem Pham, Khang Le, Nhat Ho, Tung Pham, and Hung Bui.
\newblock On unbalanced optimal transport: An analysis of {S}inkhorn algorithm.
\newblock In {\em Proceedings of the International Conference on Machine
  Learning}, 2020.

\bibitem{P1}
Benedetto Piccoli and Francesco Rossi.
\newblock Generalized wasserstein distance and its application to transport
  equations with source.
\newblock {\em Archive for Rational Mechanics and Analysis}, 211(1):335--358,
  2014.

\bibitem{P2}
Benedetto Piccoli and Francesco Rossi.
\newblock On properties of the generalized wasserstein distance.
\newblock {\em Archive for Rational Mechanics and Analysis}, 222(3):1339--1365,
  2016.

\bibitem{rabin2011wasserstein}
Julien Rabin, Gabriel Peyr{\'e}, Julie Delon, and Marc Bernot.
\newblock Wasserstein barycenter and its application to texture mixing.
\newblock In {\em International Conference on Scale Space and Variational
  Methods in Computer Vision}, pages 435--446, 2011.

\bibitem{salton1988term}
Gerard Salton and Christopher Buckley.
\newblock Term-weighting approaches in automatic text retrieval.
\newblock {\em Information processing \& management}, 24(5):513--523, 1988.

\bibitem{sato2020fast}
Ryoma Sato, Makoto Yamada, and Hisashi Kashima.
\newblock Fast unbalanced optimal transport on tree.
\newblock In {\em Advances in neural information processing systems}, 2020.

\bibitem{schiebinger2019optimal}
Geoffrey Schiebinger, Jian Shu, Marcin Tabaka, Brian Cleary, Vidya Subramanian,
  Aryeh Solomon, Joshua Gould, Siyan Liu, Stacie Lin, Peter Berube, et~al.
\newblock Optimal-transport analysis of single-cell gene expression identifies
  developmental trajectories in reprogramming.
\newblock {\em Cell}, 176(4):928--943, 2019.

\bibitem{semple2003phylogenetics}
Charles Semple and Mike Steel.
\newblock Phylogenetics.
\newblock {\em Oxford Lecture Series in Mathematics and its Applications},
  2003.

\bibitem{solomon2015convolutional}
Justin Solomon, Fernando De~Goes, Gabriel Peyr{\'e}, Marco Cuturi, Adrian
  Butscher, Andy Nguyen, Tao Du, and Leonidas Guibas.
\newblock Convolutional {W}asserstein distances: Efficient optimal
  transportation on geometric domains.
\newblock {\em ACM Transactions on Graphics (TOG)}, 34(4):66, 2015.

\bibitem{Sommerfeld2016InferenceFE}
Max Sommerfeld and A.~Munk.
\newblock Inference for empirical wasserstein distances on finite spaces.
\newblock {\em Journal of The Royal Statistical Society Series B-statistical
  Methodology}, 80:219--238, 2016.

\bibitem{turner2014persistent}
Katharine Turner, Sayan Mukherjee, and Doug~M Boyer.
\newblock Persistent homology transform for modeling shapes and surfaces.
\newblock {\em Information and Inference: A Journal of the IMA}, 3(4):310--344,
  2014.

\bibitem{vayer2019sliced}
Titouan Vayer, R{\'e}mi Flamary, Romain Tavenard, Laetitia Chapel, and Nicolas
  Courty.
\newblock Sliced {G}romov-{W}asserstein.
\newblock {\em Advances in Neural Information Processing Systems}, 2019.

\bibitem{villani2008optimal}
C{\'e}dric Villani.
\newblock {\em Optimal transport: old and new}, volume 338.
\newblock Springer Science \& Business Media, 2008.

\bibitem{pmlr-v99-weed19a}
Jonathan Weed and Quentin Berthet.
\newblock Estimation of smooth densities in wasserstein distance.
\newblock In {\em Proceedings of the Thirty-Second Conference on Learning
  Theory}, volume~99, pages 3118--3119, 2019.

\bibitem{yang2018scalable}
Karren~D. Yang and Caroline Uhler.
\newblock Scalable unbalanced optimal transport using generative adversarial
  networks.
\newblock In {\em International Conference on Learning Representations}, 2019.

\end{thebibliography}

\newpage
\onecolumn

\appendix

In the supplementary,
\begin{itemize}

\item We give detailed proofs of the theoretical results in the main text for the entropy partial transport (EPT) problem for nonnegative measures on a tree having different masses in \S \ref{supp:sec:proofs}.

\item We provide further experimental results in \S \ref{supp:sec:experiments}. For examples, 
\begin{itemize}
	\item about more setups for the efficient approximation of $\tildeETlambda$ for $\widetilde{\ET}_\lambda^{\alpha}$,
	\item about different values of $\alpha$,
	\item about different numbers of slices, 
	\item and about different parameters in tree metric sampling.
\end{itemize}

\item We next give more details and discussions in \S \ref{supp:sec:details_discussions}. For examples,
\begin{itemize}
	\item more details about experiments (e.g., softwares, datasets, more details about the experiment setup).
	\item some brief review about kernels, and more referred details (e.g., for tree metric sampling, persistence diagrams and related mathematical definitions in topological data analysis).
	\item more discussions about other relations with other work.
\end{itemize}

\end{itemize}

\section{Detailed Proofs}\label{supp:sec:proofs}

In this section, we present detailed proofs of the theoretical results in the main text.

\subsection{Proof for Theorem 3.1 in the main text}\label{sec:app_m-via-lambda}

\begin{proof}
i) Note that $ \mathrm{ET}_{c,\lambda}(\mu,\nu)$ is a concave function in $\lambda$ since it is the infimum of a family of concave functions in $\lambda$. Therefore,  $ u$ is   convex on $\R$. In particular, $u$ is differentiable almost everywhere on $\R$. 

Let $\lambda\in\R$, recall the definition of $\mathcal{C}_{\lambda } (\gamma)$ in Equation (4) in the main text. Then for any $\gamma\in \Gamma^0(\lambda)$,  we have 
\begin{align}\label{sub-ineq}
  \mathrm{ET}_{c,\lambda +\delta}(\mu,\nu) 
  \leq \mathcal{C}_{\lambda +\delta} (\gamma)=\mathcal{C}_{\lambda } (\gamma)- b \delta \gamma(\calT\times\calT)
  = \mathrm{ET}_{c,\lambda }(\mu,\nu)- b \delta \gamma(\calT\times\calT) \,\,\,  \forall \delta\in\R.
 \end{align}
This implies that 
\[
\big\{ b\, \gamma(\calT\times\calT): \gamma\in \Gamma^0(\lambda)\big\} \subset \partial u(\lambda).
\]
We next show that the opposite  inclusion is also true, i.e., $\big\{ b\, \gamma(\calT\times\calT): \gamma\in \Gamma^0(\lambda)\big\} = \partial u(\lambda)$. This is obviously holds if $\partial u(\lambda) $ is singleton and hence  we only need to consider  $\lambda$ for which  the convex set $\partial u(\lambda)$ has more than one element. 

Let $m \in \partial u(\lambda)$, then $m$ can be expressed as a convex combination of extreme points $m_1, \dotsc, m_N$ of $\partial u(\lambda)$, i.e.,  $m= \sum_{i=1}^N t_i m_i$ with $0\leq t_i \leq 1$ and $\sum_{i=1}^N t_i=1$. As $m_i$ is an extreme point of $\partial u(\lambda)$, there exists a sequence $\lambda_n\to \lambda$ such that $\lambda_n$ is a differentiable point of $u$ and $u'(\lambda_n)\to m_i$. 

Let $\gamma^n\in \Gamma^0(\lambda_n)$, then $b\, \gamma^n(\calT\times \calT)=u'(\lambda_n)\to m_i$. By compactness, there exists a subsequence $\{\gamma^{n_k}\}$ and $\tilde\gamma^i\in \Pi_{\leq}(\mu,\nu)$ such that  $\gamma^{n_k}\to \tilde \gamma^i$ weakly. It follows that  $\gamma^{n_k}(\calT\times \calT)\to \tilde \gamma^i(\calT\times \calT)$, and hence we must have $b\, \tilde\gamma^i(\calT\times \calT)=m_i$. We have
\begin{align*}
   \mathcal{C}_{\lambda_{n_k}} (\gamma^{\lambda_{n_k}})
  = \mathcal{C}_{\lambda } (\gamma^{\lambda_{n_k}})
 + b (\lambda -\lambda_{n_k} )  \gamma^{n_k}(\calT\times \calT) 
  &\geq \mathrm{ET}_{c,\lambda }(\mu,\nu)+ b (\lambda -\lambda_{n_k} )  \gamma^{n_k}(\calT\times \calT)\\
  &\geq \mathrm{ET}_{c,\lambda }(\mu,\nu)- b \bar m |\lambda -\lambda_{n_k}|
 \end{align*}
and for any $\gamma \in \Gamma^0(\lambda)$, there holds
\begin{align*}
   \mathcal{C}_{\lambda_{n_k}} (\gamma^{\lambda_{n_k}})\leq \mathcal{C}_{\lambda_{n_k}} (\gamma)
   =\mathcal{C}_{\lambda } (\gamma)+ b (\lambda -\lambda_{n_k} )  \gamma(\calT\times \calT) 
  = \mathrm{ET}_{c,\lambda }(\mu,\nu)+b (\lambda -\lambda_{n_k} )  \gamma(\calT\times \calT).
 \end{align*}
 We thus deduce that $\lim_{k\to\infty} \mathcal{C}_{\lambda_{n_k}} (\gamma^{\lambda_{n_k}})=\mathrm{ET}_{c,\lambda }(\mu,\nu)$. These together with the lower semicontinuity of $\mathcal{C}_{\lambda }$ give
\begin{align*}
\mathrm{ET}_{c,\lambda }(\mu,\nu) =\liminf_{k\to\infty} \mathcal{C}_{\lambda_{n_k}} (\gamma^{\lambda_{n_k}})
&=\liminf_{k\to\infty}\Big[ \mathcal{C}_{\lambda } (\gamma^{\lambda_{n_k}})
 + b (\lambda -\lambda_{n_k} )  \gamma^{n_k}(\calT\times \calT)\Big]\\
 &=\liminf_{k\to\infty}\mathcal{C}_{\lambda } (\gamma^{\lambda_{n_k}})\geq \mathcal{C}_{\lambda } (\tilde \gamma^i).
\end{align*}
Therefore,  $\tilde \gamma^i \in \Gamma^0(\lambda)$ with mass $b\, \tilde \gamma^i(\calT\times \calT)=m_i$. Due to the convexity of $\Gamma^0(\lambda)$, we have $\bar\gamma :=\sum_{i=1}^N t_i \tilde \gamma^i \in \Gamma^0(\lambda)$ with $b\, \bar  \gamma(\calT\times \calT)=\sum_{i=1}^N t_i m_i=m$. That is, 
\[
 \partial u(\lambda) \subset \big\{ b\, \gamma(\calT\times\calT): \gamma\in \Gamma^0(\lambda)\big\},
 \]
and we thus infer that $\big\{ b\, \gamma(\calT\times\calT): \gamma\in \Gamma^0(\lambda)\big\} = \partial u(\lambda)$ for all $\lambda\in\R$.

In order to prove the second part of i), let $ \gamma\in \Gamma^0(\lambda_1)$ and $ \tilde \gamma\in \Gamma^0(\lambda_2)$ be arbitrary.  We have 
 \begin{align}\label{another-sub-ineq}
  \mathrm{ET}_{c,\lambda_2}(\mu,\nu) 
  = \mathcal{C}_{\lambda_2} (\tilde\gamma)
 & = \mathcal{C}_{\lambda_1 } (\tilde \gamma)
  - b (\lambda_2 -\lambda_1) \tilde \gamma(\calT\times\calT)\nonumber\\
  &\geq \mathrm{ET}_{c,\lambda_1 }(\mu,\nu)-b (\lambda_2 -\lambda_1) \tilde \gamma(\calT\times\calT).
 \end{align}
Hence by combining with  \eqref{sub-ineq}, we deduce that
\begin{align*}
 \mathrm{ET}_{c,\lambda_1 }(\mu,\nu)- b (\lambda_2-\lambda_1) \tilde \gamma(\calT\times\calT) \leq  \mathrm{ET}_{c,\lambda_2}(\mu,\nu) 
  \leq  \mathrm{ET}_{c,\lambda_1 }(\mu,\nu)- b (\lambda_2-\lambda_1) \gamma(\calT\times\calT),
 \end{align*}
which yields $\gamma(\calT\times\calT)\leq \tilde \gamma(\calT\times\calT)$. This together with the above characterization of $\partial u(\lambda)$ implies the second part of i).

ii) If $u$ is differentiable at $\lambda$, then $\partial u(\lambda)$ is a singleton set. However, as $\partial u(\lambda) =\big\{ b\, \gamma(\calT\times\calT): \gamma\in \Gamma^0(\lambda)\big\}$  by i), we thus infer that the mass $\gamma(\calT\times\calT)$ must be the same for every $\gamma\in \Gamma^0(\lambda)$.

Next assume that every element  in $\Gamma^0(\lambda)$ has the same  mass, say $m$. For $\delta\neq 0$,
 let  $\gamma^{\lambda +\delta}\in \Gamma^0(\lambda+\delta)$ and  $m(\lambda+\delta):=\gamma^{\lambda +\delta}(\calT\times\calT)$.
 Then, we claim that 
 \begin{equation}\label{m-cont}
 \lim_{\delta\to 0} m(\lambda +\delta) =m.
 \end{equation}
 Assume the claim for the moment, and let $\delta> 0$. Then, as in \eqref{sub-ineq}--\eqref{another-sub-ineq}, we have
 \begin{align*}
  \mathrm{ET}_{c,\lambda +\delta}(\mu,\nu) 
  \leq \mathrm{ET}_{c,\lambda }(\mu,\nu)- b \delta m\quad
  \mbox{and}\quad 
  \mathrm{ET}_{c,\lambda +\delta}(\mu,\nu) 
  \geq \mathrm{ET}_{c,\lambda }(\mu,\nu)- b \delta m(\lambda+\delta).
 \end{align*}
  It follows that
 \[
 - b  m(\lambda+\delta) \leq 
 \frac{\mathrm{ET}_{c,\lambda +\delta}(\mu,\nu)-\mathrm{ET}_{c,\lambda }(\mu,\nu)}{\delta}
 \leq - b  m.
 \]
This together with  claim \eqref{m-cont} gives $\lim_{\delta\to 0^+} \frac{\mathrm{ET}_{c,\lambda +\delta}(\mu,\nu)-\mathrm{ET}_{c,\lambda }(\mu,\nu)}{\delta} =- b  m$. By the same argument, we also have 
 $\lim_{\delta\to 0^-} \frac{\mathrm{ET}_{c,\lambda +\delta}(\mu,\nu)-\mathrm{ET}_{c,\lambda }(\mu,\nu)}{\delta} =- b  m$.
 Thus, we infer that $u$ is differentiable at $\lambda$ with  
 $u'(\lambda)= b m$.
Therefore, it remains to prove claim \eqref{m-cont}. 

Indeed, by compactness there exists a subsequence, still labeled by $\gamma^{\lambda+\delta}$, and $\gamma \in \Pi_{\leq}(\mu,\nu)$ such that 
$\gamma^{\lambda+\delta}\to \gamma$ weakly as $\delta\to 0$. As in i), we can show  that $\gamma\in \Gamma^0(\lambda)$.
Then, as the mass functional  is weakly continuous, we obtain $m(\lambda+\delta)=\gamma^{\lambda+\delta}(\calT\times \calT)\to \gamma(\calT\times \calT)=m$. We in fact have shown that any subsequence of $\{m(\lambda+\delta)\}_\delta$ has a further subsequence converging to the same number $m$. Therefore, the full sequence $\{m(\lambda+\delta)\}_\delta$ must converge to $m$, and hence \eqref{m-cont} is proved.

iii) For any $\lambda\in\R$, we have $ \partial u(\lambda) = \big\{ b\, \gamma(\calT\times\calT): \gamma\in \Gamma^0(\lambda)\big\}\subset [0,b\, \bar m]$. Thus, we only need to prove $[0,b\, \bar m]\subset \partial u(\R)$. First, note that as $ \partial u(\lambda)\subset\R$ is a compact and convex set, it must be a finite and closed interval. Therefore, if we let 
\[
\gamma^\lambda_{min} :=\argmin_{\gamma\in \Gamma^0(\lambda)} \gamma(\calT\times\calT)\quad \mbox{and}\quad 
\gamma^\lambda_{max} :=\argmax_{\gamma\in \Gamma^0(\lambda)} \gamma(\calT\times\calT),
\]
then it follows from ii) that $ \partial u(\lambda)=\big[b\,\gamma^\lambda_{min}(\calT\times\calT), b\, \gamma^\lambda_{max}(\calT\times\calT)\big] $ for every $\lambda\in\R$.
From Equation (4) in the main text, it is clear that $ \partial u(\lambda)=\{0\}$ for $\lambda$ negative enough. Indeed, if we take $\lambda < -M$, then as $w_1(x) +w_2(y) \leq b\,  c(x,y) +M$, we have 
$0<  b \, c(x,y)- w_1(x) -w_2(y) -\lambda$
for all $x,y\in \calT$. Then, we obtain from 
Equation (4) in the main text that $\mathcal{C}_\lambda(  0)\leq \mathcal{C}_\lambda(  \gamma)$ for every  $\gamma\in \Pi_{\leq}(\mu,\nu)$ and the strict inequality holds if $\gamma \neq 0$. Thus, $ \Gamma^0(\lambda)=\{0\}$ which gives  $\partial u(\lambda)=\{0\}$ and $u(\lambda)=-\int_\calT  w_1 \mu(dx) 
- \int_\calT  w_2  \nu(dx)$. 

We next show that $ \partial u(\lambda)=\{b\, \bar m\}$ for $\lambda$ positive enough.
Since $c(x,y)$ is bounded due to its continuity on $\calT\times\calT$, we can choose $\lambda\in\R$ such that $c(x,y)-\lambda<0$ for all $x,y\in \calT$. Let $\gamma \in \Gamma^0(\lambda)$. We claim that  either $\gamma_1=\mu$ or $\gamma_2=\nu$. Indeed, since otherwise we have  $\gamma_1(A_0)<\mu(A_0)$ and  $\gamma_2(B_0)<\nu(B_0)$ for some Borel sets $A_0, B_0\subset \calT$. Let $\tilde \gamma := \gamma +  [(\mu-\gamma_1)\chi_{A_0}]\otimes [(\nu-\gamma_2)\chi_{B_0}]$. Then, for any Borel set $A\subset \calT$ we have
\begin{align*}
 \tilde\gamma_1(A) = \gamma_1(A) + \mu(A\cap A_0) - \gamma_1(A\cap A_0) & =\gamma_1(A\setminus  A_0) +\mu(A\cap A_0)\\
 &\leq     \mu(A\setminus  A_0) +\mu(A\cap A_0)=\mu(A).
\end{align*}
Likewise, $\tilde\gamma_2(B)\leq \nu(B)$ for any Borel set $B\subset\calT$. Thus 
$\tilde \gamma \in \Pi_{\leq}(\mu,\nu)$. On the other hand, it is clear from 
Equation (4) in the main text and the facts $\gamma_1\leq \tilde\gamma_1$,  $\gamma_2\leq \tilde\gamma_2$, and  $c-\lambda <0$ that $\mathcal{C}_\lambda(  \tilde \gamma)< \mathcal{C}_\lambda(  \gamma)$. This is impossible and so the claim is proved. That is, either $\gamma_1=\mu$ or $\gamma_2=\nu$. It follows that $\gamma(\calT\times \calT) = \bar m$ for every $\gamma\in \Gamma^0(\lambda)$, and hence $\partial u(\lambda)=\{b\, \bar m\}$. This also means that $u$ is differentiable at $\lambda$ with  $u'(\lambda)=b\, \bar m$. 

Therefore, it remains to show that 
\begin{equation}\label{sub-inclusion}
(0,b\, \bar m)\subset \partial u(\R)= \bigcup_{\lambda\in\R}\big[b\, \gamma^\lambda_{min}(\calT\times\calT), b\, \gamma^\lambda_{max}(\calT\times\calT)\big].
\end{equation}
Assume by contradiction that there exists $m\in (0,b\, \bar m)$ such that $m\not\in \partial u(\lambda)$ for every $\lambda \in \R$. For convenience, we adopt the following notation: for sets $A, B\subset \R$ and $r\in\R$, we write $A< r$ if $a<r$ for every $a\in A$, and $A<B$ if $a<b$ for every $a\in A$ and $b\in B$. Let us consider the following two sets
\[
S_1 := \{\lambda: \partial u(\lambda) <m\}\quad \mbox{and}\quad S_2 := \{\lambda: \partial u(\lambda) >m\}.
\]
Then $\lambda\in S_1$ if $\lambda$ is negative enough, and  $\lambda\in S_2$ if $\lambda$ is positive enough. For any $\lambda_1\in S_1$ and $\lambda_2\in S_2$, we have $\partial u(\lambda_1) <m<\partial u(\lambda_2)$, and hence $\lambda_1 < \lambda_2$  by the monotonicity in i). That is,   $S_1 <S_2$ and so we obtain 
\[
\lambda^*:=\sup\{\lambda: \lambda\in S_1\}\leq \inf\{\lambda: \lambda\in S_2\}=:\lambda^{**}.
\]
If $\lambda^* <\lambda^{**}$, then for any $\lambda \in (\lambda^*, \lambda^{**})$ we have $\lambda \not\in S_1$ and $\lambda \not\in S_2$. Therefore, $\partial u(\lambda) \not< m$ and $\partial u(\lambda) \not> m$. Hence, we can find $m_1, m_2\in \partial u(\lambda) $ such that $m_1\geq m$ and $m_2 \leq m$. Thus, $m\in [m_2, m_1]\subset \partial u(\lambda)$ due to the convexity of the set $\partial u(\lambda)$. This contradicts our hypothesis, and we conclude that $\lambda^* =\lambda^{**}$. 

We next select sequences $\{\lambda^1_n\}\subset S_1$ and $\{\lambda^2_n\}\subset S_2$ such that  $\lambda^1_n\to \lambda^*$
and $\lambda^2_n\to \lambda^{**}=\lambda^*$. For each $n$, let 
\[
\gamma^n_{min} :=\argmin_{\gamma\in \Gamma^0(\lambda^1_n)} \gamma(\calT\times\calT)\quad \mbox{and}\quad 
\gamma^n_{max} :=\argmax_{\gamma\in \Gamma^0(\lambda^2_n)} \gamma(\calT\times\calT).
\]
 By compactness, there exist subsequences, still labeled as $\{\gamma^n_{min}\}$ and $\{\gamma^n_{max}\}$, and $\gamma^*, \gamma^{**}\in \Pi_{\leq}(\mu,\nu)$ such that  $\gamma^n_{min} \to \gamma^*$ weakly and $\gamma^n_{max} \to \gamma^{**}$ weakly. By arguing exactly as in i), we then obtain $\gamma^*, \gamma^{**}\in \Gamma^0(\lambda^*)$, $\gamma^n_{min}(\calT\times \calT) \to  \gamma^*(\calT\times \calT)$, and $\gamma^n_{max}(\calT\times \calT) \to  \gamma^{**}(\calT\times \calT)$. As $b\, \gamma^n_{min}(\calT\times \calT)<m$ due to $\lambda^1_n\in S_1$, we must have $b\, \gamma^*(\calT\times \calT)\leq m$. Likewise, we have $b\, \gamma^{**}(\calT\times \calT)\geq m$ as $b\, \gamma^n_{max}(\calT\times \calT)>m$ for all $n$.
 Hence, $m\in [b\, \gamma^*(\calT\times \calT), b\, \gamma^{**}(\calT\times \calT)]$. Since $\gamma^*, \gamma^{**}\in \Gamma^0(\lambda^*)$, we infer that $m\in \partial u(\lambda^*)$. This is a contradiction and the proof is complete. We note that since $\lambda^1_n\leq \lambda^*\leq \lambda^2_n$, we have from the monotonicity in i) that 
 \[
 \gamma^n_{min}(\calT\times \calT) \leq \gamma(\calT\times \calT)\leq \gamma^n_{max}(\calT\times \calT)
 \]
 for every $\gamma\in \Gamma^0(\lambda^*)$. By sending $n$ to infinity,  it follows that $\gamma^*(\calT\times \calT) \leq \gamma(\calT\times \calT)\leq \gamma^{**}(\calT\times \calT)$ for every $\gamma\in \Gamma^0(\lambda^*)$. That is,
 $
\gamma^* =\gamma^{\lambda^*}_{min}$ and 
$\gamma^{**} =\gamma^{\lambda^*}_{max}
$.
\end{proof}

\subsection{Proof for Lemma 3.2 in the main text}\label{sec:app_rep-formula}

\begin{proof}
We first observe for any Borel set $A\subset \calT$ that
\begin{align*}
\hat \gamma(A\times \{\hat s\}) = \hat \gamma(A\times\hat \calT) - \hat \gamma(A\times \calT)=\hat \mu(A) - \gamma(A\times \calT)=\mu(A) -\gamma_1(A)=\int_A (1-f_1) \mu(dx).
\end{align*}
For the same reason, we have  $\hat \gamma( \{\hat s\}\times B)=  \int_B (1-f_2) \nu(dx)$ for any set
 Borel set $B\subset \calT$. Also,
 \begin{align*}
\hat \gamma(\{\hat s\}\times \{\hat s\}) 
&= \hat \gamma(\hat\calT \times \{\hat s\}) -\hat \gamma(\calT\times \{\hat s\})  \\
&= \hat \gamma(\hat\calT\times \hat\calT)
- \hat \gamma(\hat\calT\times \calT)   -\big[  \hat\gamma(\calT\times \hat\calT) -  \hat \gamma(\calT\times \calT) \big]\\
&= \hat\mu(\hat\calT) - \hat\nu(\calT) - \hat\mu(\calT) +\gamma(\calT\times \calT)
=\gamma(\calT\times \calT).
\end{align*}
 
Since 
the Equation (6) in the main text is obviously true for sets of the form $A\times B$ with $A,B\subset \calT$ being Borel sets, we only need to verify it for sets of the following forms:  $(A\cup 
 \{\hat s\})\times B$, $A\times (B\cup \{\hat s\})$, $(A\cup \{\hat s\})\times (B\cup\{\hat s\})$ for Borel sets  $A,B\subset \calT$. We check it case by case as follows.
 
 Case 1:  Using the above observation, we have 
 \begin{align*}
\hat \gamma((A\cup 
 \{\hat s\})\times B) 
&=  \hat \gamma(A\times B) + \hat \gamma(\{\hat s\}\times B) =  \gamma(A\times B) + \int_B (1-f_2) \nu(dx).
\end{align*}
 Therefore, 
 the Equation (6) in the main text holds in this case.
 
 Case 2:  
 the Equation (6) in the main text is also true for this case because 
 \begin{align*}
\hat \gamma(A \times (B\cup 
 \{\hat s\})) 
&=  \hat \gamma(A\times B) + \hat \gamma(A\times \{\hat s\}) 
=  \gamma(A\times B) + \int_A (1-f_1) \mu(dx).
\end{align*}

Case 3:  
the Equation (6) in the main text is true as well since 
 \begin{align*}
\hat \gamma((A\cup 
 \{\hat s\}) \times (B\cup 
 \{\hat s\})) 
&=  \hat \gamma(A\times B) + \hat \gamma(A\times \{\hat s\})  + \hat \gamma(\{\hat s\}\times B) +\hat \gamma(\{\hat s\}\times \{\hat s\})\\
&=  \gamma(A\times B)  + \int_A (1-f_1) \mu(dx)
+\int_B (1-f_2) \nu(dx)
+\gamma(\calT\times \calT).
\end{align*}

Now as 
the Equation (6) in the main text holds, we obviously have $\gamma(U\times \calT)\leq \hat \gamma(U\times \calT)\leq \hat \gamma(U\times \hat \calT)=\hat\mu(U)=\mu(U)$ for any Borel set $U\subset \calT$. Likewise, $\gamma( \calT\times U )\leq \nu(U)$ for any Borel set $U\subset \calT$. Therefore, $\gamma\in \Pi_{\leq}(\mu,\nu)$.
\end{proof}

\subsection{Proof for Proposition 3.3 in the main text}\label{sec:app_distance-agree}

\begin{proof}
We first show that $\mathrm{KT}(\hat \mu,\hat\nu)  \leq\mathrm{ET}_{c,\lambda}(\mu,\nu) $.  

For any  $\gamma\in \Pi_{\leq}(\mu,\nu)$, let $\hat\gamma$ be given by 
the Equation (6) in the main text. Then, $\hat \gamma\in \Gamma(\hat \mu,\hat \nu)$ and 
\begin{align*}
\mathrm{KT}(\hat \mu,\hat\nu) \leq 
\int_{\hat \calT\times \hat \calT} \hat c(x,y) \hat\gamma(dx, dy)
&=b \int_{ \calT\times  \calT} [c(x,y)-\lambda] \gamma(dx, dy)\\
&\quad +\int_\calT  w_1 [1-f_1(x)] \mu(dx)  
+ \int_\calT  w_2 [1-f_2(x)] \nu(dx).
\end{align*}
It follows that $\mathrm{KT}(\hat \mu,\hat\nu)  \leq\mathrm{ET}_{c,\lambda}(\mu,\nu)$.
 
We next  show that $\mathrm{KT}(\hat \mu,\hat\nu)\geq \mathrm{ET}_{c,\lambda}(\mu,\nu)$.  To see this, for any  $\hat \gamma\in \Gamma(\hat \mu,\hat \nu)$ we let 
$\gamma$ be the restriction of  $\hat\gamma$ to $\calT$. Then by Lemma 
3.2 in the main text, we have $\gamma\in \Pi_{\leq}(\mu,\nu)$ and 
the Equation (6) in the main text holds. Consequently,
\begin{align*}
\int_{\hat \calT\times \hat \calT} \hat c(x,y) \hat\gamma(dx, dy)
&=b \int_{ \calT\times  \calT} [c(x,y)-\lambda] \gamma(dx, dy)\\
&\quad +\int_\calT  w_1 [1-f_1(x)] \mu(dx)  
+ \int_\calT  w_2 [1-f_2(x)] \nu(dx)\\
&\geq  \mathrm{ET}_{c,\lambda}(\mu,\nu).
\end{align*}
By taking the infimum over $\hat \gamma$, we infer that $\mathrm{KT}(\hat \mu,\hat\nu)  \geq\mathrm{ET}_{c,\lambda}(\mu,\nu)$. 

Thus we obtain 
\[
\mathrm{KT}(\hat \mu,\hat\nu)  = \mathrm{ET}_{c,\lambda}(\mu,\nu). 
\]
The relation about the optimal solutions also follows from the above arguments.
\end{proof}

\subsection{Proof for Theorem 3.4 in the main text}\label{sec:app_duality}

\begin{proof} From Proposition 
3.3 in the main text and the dual formulation for $\mathrm{KT}(\hat \mu,\hat\nu)$ proved in \cite[Corollary~2.6] {CM}, we have
\begin{align*}
\mathrm{ET}_{c,\lambda}(\mu,\nu)=\sup_{\substack{\hat u \in L^1(\hat \mu),\, \hat v\in L^1(\hat \nu)\\\hat u(x) +\hat v(y)\leq \hat c(x,y)}} \int_{\hat\calT} \hat u(x) \hat \mu(dx) +  \int_{\hat\calT} \hat v(x) \hat \nu(dx)=: I.
\end{align*}
Therefore, it is enough to prove that  $I=J$ where
\begin{equation*}
J :=
\sup_{(u,v) \in \K} \Big[ \int_{\calT}  u(x) \mu(dx) +  \int_{\calT}  v(x)  \nu(dx)\Big].
\end{equation*}

For  $(u,v)$ satisfying $u\leq w_1$, $v\leq w_2$ and $ u(x) + v(y)\leq  b[c(x,y) - \lambda]$, we extend  it to $\hat \calT$ by taking $\hat u(\hat s)=0$ and $\hat v(\hat s)=0$. Then, it is clear that $\hat u(x) +\hat v(y)\leq \hat c(x,y)$ for $x,y\in \hat\calT$, and 
\begin{align*}
I\geq \int_{\hat\calT} \hat u(x) \hat \mu(dx) +  \int_{\hat\calT} \hat v(x) \hat \nu(dx) 
=\int_{\calT}  u(x) \mu(dx) +  \int_{\calT}  v(x)  \nu(dx).
\end{align*}
It follows that $I\geq J$. In order to prove the converse, let $(\hat u,\hat v)$ be a maximizer for $I$. Then, by considering $(\hat u - \hat u(\hat s),\hat v +\hat u(\hat s))$, we can assume  that $\hat u(\hat s)=0$. Also, if we let 
$
v(y) :=\inf_{x\in\hat\calT} [\hat c(x,y)-\hat u(x)]$,
then $(\hat u,v)$ is still in the admissible class for $I$ and $\hat v(y)\leq v(y)$.  This implies that  $(\hat u,v)$ is also a maximizer for $I$. For these reasons,  we can assume w.l.g. that the maximizer $(\hat u,\hat v)$  has the following additional properties: $\hat u(\hat s)=0$ and 
\[
\hat v(y) =\inf_{x\in\hat\calT} [\hat c(x,y)-\hat u(x)]\quad \forall y\in \hat\calT.
\]
In particular, $\hat v(\hat s) = \inf_{x\in\hat\calT} [\hat c(x,\hat s)-\hat u(x)]$. For convenience, define $w_1(\hat s)=0$ and consider the following two possibilities.

Case 1: $\inf_{x\in\hat \calT} [w_1(x) - \hat u(x)]\geq 0$.
Then, since $\hat c(\hat s,\hat s)-\hat u(\hat s)=0$ and $ \inf_{x\in\calT} [\hat c(x,\hat s)-\hat u(x)]=\inf_{x\in\calT} [w_1(x) - \hat u(x)]\geq 0$, we have  $\hat v(\hat s) =0$. Also, $\hat v(y) \leq \hat c(\hat s,y)-\hat u(\hat s)\leq w_2(y)$ for all $ y\in \hat\calT$.
    For each $ y\in\calT$, by using  the facts $\hat  u \leq w_1$ and $\hat c(\hat s,y)-w_1(\hat s)=w_2(y)\geq 0$  we get 
\begin{align*}
\hat  v(y) \geq \inf_{x\in \hat \calT} [\hat c(x,y)-w_1(x)]=\inf_{x\in  \calT} \{b[c(x,y)-\lambda]-w_1(x)\}
= -b \lambda + \inf_{x\in \calT} [b\, c(x,y) -w_1(x)].
\end{align*}
  Thus $(\hat  u, \hat  v)\in\K$ and 
\begin{align*}
I = \int_{\hat\calT} \hat u(x) \hat \mu(dx) +  \int_{\hat\calT} \hat v(x) \hat \nu(dx) 
&=\int_{\calT}  \hat u(x) \hat \mu(dx) +  \int_{\calT}  \hat v(x)  \hat \nu(dx)
+ \hat v(\hat s) \mu(\calT) \\
&=  \int_{\calT}  \hat u(x)  \mu(dx) +  \int_{\calT}  \hat v(x)   \nu(dx)\leq  J.
\end{align*}

Case 2: $\inf_{x\in\hat \calT} [w_1(x) - \hat u(x)]< 0$. Then, by arguing as in Case 1,  we have  $\hat v(\hat s) =\inf_{x\in\calT} [w_1(x) - \hat u(x)]<0$ and
\begin{align}\label{a-connection}
I=\int_{\calT}  \hat v(x)  \nu(dx) + \int_{\calT}  \hat u(x)  \mu(dx)   
+\mu(\calT)  \inf_{\calT} [w_1 - \hat u].
\end{align}
Let $\tilde u(x) := \min\{\hat u(x), w_1(x)\}$. Then, it is obvious that  $\tilde  u(x) +\hat v(y)\leq \hat c(x,y)$ and $\tilde  u(\hat s)=0$.
Since $\inf_{x\in\calT} [w_1(x) - \hat u(x)]< 0$, there exists $x_0\in \calT$ such that $w_1(x_0) < \hat u(x_0)$. Thus, $\tilde u(x_0) = w_1(x_0)$ and hence  $\inf_{\calT} [w_1 -\tilde u] \leq 0$. As $\tilde u\leq w_1$, we infer further that $\inf_{\calT} [w_1 -\tilde u] = 0$.
 We also have
\begin{align*}
& \int_{\calT}  \hat u(x) \mu(dx)   
+\mu(\calT)  \inf_{\calT} [w_1 - \hat u]\\
&= \int_{\calT}  \tilde  u(x)  \mu(dx)   
+ \int_{\calT: \hat  u > w_1 }  [\hat u(x) -w_1(x)] \mu(dx)  
+ \mu(\calT)  \inf_{\calT} [w_1 - \hat u] \leq \int_{\calT}  \tilde  u(x)  \mu(dx) .
\end{align*}
This together with \eqref{a-connection} gives
\begin{align*}
I\leq \int_{\calT}  \tilde  u(x)  \mu(dx)   +\int_{\calT}  \hat v(x)   \nu(dx).
\end{align*}
Now let 
$\tilde  v(y) =\inf_{x\in \hat \calT} [\hat c(x,y)-\tilde  u(x)]$ for $ y\in \calT$. Then, $\hat v(y)\leq \tilde  v(y) \leq \hat c(\hat s,y)-\tilde  u(\hat s)=w_2(y)$ for $y\in \calT$.  For each $ y\in\calT$, by using  the facts $\tilde u \leq w_1$ and $\hat c(\hat s,y)-w_1(\hat s)=w_2(y)\geq 0$  we also get 
\begin{align*}
\tilde  v(y) \geq \inf_{x\in \hat \calT} [\hat c(x,y)-w_1(x)]=\inf_{x\in  \calT} \{b[c(x,y)-\lambda]-w_1(x)\}
= -b \lambda + \inf_{x\in \calT} [b\, c(x,y) -w_1(x)].
\end{align*}
It follows that $(\tilde u, \tilde v)\in \K$ and 
\begin{align*}
I\leq \int_{\calT}  \tilde  u(x)  \mu(dx)   +\int_{\calT}  \tilde  v(x)  \nu(dx)\leq J.
\end{align*}

Thus we conclude that $I=J$ and the theorem follows.
\end{proof}

\subsection{Proof for Corollary 3.5 in the main text}\label{sec:app_cor_duality}

\begin{proof}
Notice that as $w_i$ ($i=1,2$) is  $b$-Lipschitz, we have for every $x\in\calT$ that 
\begin{equation}\label{w-lipschitz}
-w_i(x) \leq \inf_{y\in\calT} \big[b\, d_\calT(x,y)- w_i(y)\big].
\end{equation}
For each $(u,v)\in\K$, let
\begin{align*}
v^*(x) 
&:= \inf_{y\in\calT} \big\{b[d_\calT(x,y)-\lambda]- v(y)\big\} = -b\lambda + \inf_{y\in\calT} \big[b\, d_\calT(x,y)- v(y)\big]\geq u(x),\\
 v^{**}(y) 
 &:= \inf_{x\in\calT} \big\{b[d_\calT(x,y)-\lambda]- v^*(x)\big\}=-b\lambda + \inf_{x\in\calT} \big[b\, d_\calT(x,y) - v^*(x)\big]\geq v(y).
\end{align*}
By using  $ -b \lambda + \inf_{x\in \calT} [b\, d_\calT(x,y) -w_1(x)]\leq v(y)\leq w_2(y)$ and  \eqref{w-lipschitz}, we obtain  for every $x\in\calT$ that
\begin{align*}
v^*(x)
&\leq -b\lambda -v(x)\leq -\inf_{y\in \calT} [b\, d_\calT(x,y) -w_1(y)]\leq w_1(x),\\
v^*(x)
&\geq -b\lambda + \inf_{y\in\calT} \big[b\, d_\calT(x,y)- w_2(y)\big]\geq -b\lambda -  w_2(x).
\end{align*}
 We also have  $v^*$ is  $b$-Lipschitz, i.e., $|v^*(x_1)-v^*(x_2) |\leq b \, d_\calT(x_1,x_2)$. Indeed, let $x_1, x_2\in \calT$. Then for any $\e>0$, there exists $y_1\in\calT$ such that $b\, d_\calT(x_1,y_1)- v(y_1) < v^*(x_1) +b\lambda +\e$. It follows that
\[
v^*(x_2) -v^*(x_1) \leq b\, d_\calT(x_2,y_1)- v(y_1)
+\e - [b\, d_\calT(x_1,y_1)- v(y_1)]\leq b \, d_\calT(x_1, x_2) +
\e.
\]
Since this holds for every  $\e>0$, we get  $v^*(x_2) -v^*(x_1) \leq b \, d_\calT(x_1, x_2)$. By interchanging the role of $x_1$ and $x_2$, we also obtain $v^*(x_1) -v^*(x_2) \leq b \, d_\calT(x_1, x_2)$. Thus,  $|v^*(x_1)-v^*(x_2) |\leq b \, d_\calT(x_1,x_2)$. Hence, we have shown  that $v^*\in  \mathbb{L'}$ with
\[
\mathbb{L'} :=\Big\{f\in C(\calT):\, 
 -b\lambda -w_2 \leq f \leq    w_1, \, |f(x)-f(y) |\leq b \, d_\calT(x,y)\Big\}.
 \]

We next claim $v^{**}=  - b\lambda - v^*$. For this, it is clear from the definition that  $v^{**}(y) \leq - b\lambda - v^*(y)$. On the other hand, from the Lipschitz property of $v^*$ we obtain 
\[
- v^*(y) \leq b \, d_\calT(x,y) - v^*(x) \quad \forall x\in \calT,
\]
which gives  $- b\lambda - v^*(y) \leq v^{**}(y)$. Thus, we conclude that $v^{**}=  - b\lambda - v^*$ as claimed.

From these, we obtain that
\begin{align*}
\int_{\calT}  u(x) \mu(dx) +  \int_{\calT}  v(x)  \nu(dx)
&\leq \int_{\calT}  v^*(x) \mu(dx) +  \int_{\calT}  v^{**}(x)  \nu(dx)\\
&= \int_{\calT}  v^*(x) \mu(dx) - \int_{\calT}  v^{*}(x)  \nu(dx) -b\lambda \nu(\calT)\\
&\leq - b\lambda \nu(\calT) + \sup \left\{ \int_\calT f (\mu - \nu) :\, f\in \mathbb{L'}  \right\}.
\end{align*}
This together with Theorem 
3.4 in the main text implies that $\mathrm{ET}_\lambda(\mu,\nu)\leq  - b\lambda \nu(\calT) + \sup \left\{ \int_\calT f (\mu - \nu) :\, f\in \mathbb{L'}  \right\}$.
To prove the converse, let $ f\in \mathbb{L'}$.
Define $u:= f$ and $ v:= -b\lambda -f$. Then,
we have $u(x) \leq w_1(x)$, $v(x)\leq -b\lambda-[-b\lambda -w_2(x)]=w_2(x)$, and 
\[
v(x)\geq -b\lambda- w_1(x)\geq -b\lambda + \inf_{y\in \calT} [b\, d_\calT(x,y) -w_1(y)].
\]
Also, the Lipschitz property of $f$ gives 
\[
u(x) + v(y) =-b\lambda + f(x)- f(y)\leq b[ d_\calT(x,y)-\lambda]\quad \forall x,y\in\calT.
\]
Thus  $(u,v) \in \K$, and hence we obtain from Theorem 
3.4 in the main text that
 \begin{align*}
 - b\lambda \nu(\calT) +  \int_\calT f (\mu - \nu) 
 =
\int_{\calT}  u(x) \mu(dx) +  \int_{\calT}  v(x)  \nu(dx)
\leq \mathrm{ET}_\lambda(\mu,\nu).
\end{align*}
As this holds for every $f\in \mathbb{L'}$, we get 
\[
- b\lambda \nu(\calT) + \sup \left\{ \int_\calT f (\mu - \nu) :\, f\in \mathbb{L'}  \right\}
\leq \mathrm{ET}_\lambda(\mu,\nu).
\]
Thus, we have shown that
\begin{equation}\label{non-summetric}
\mathrm{ET}_\lambda(\mu,\nu)= - b\lambda \nu(\calT) + \sup \left\{ \int_\calT f (\mu - \nu) :\, f\in \mathbb{L'}  \right\}.
\end{equation}
Now consider $f= \tilde f - \frac{b\lambda}{2}$. Then, $f\in \mathbb{L'} $ if and only if $\tilde f\in \mathbb{L}$. Moreover,
 \[
 \int_\calT f (\mu - \nu)  =- \frac{b\lambda}{2}\big[\mu(\calT) -\nu(\calT)\big] + \int_\calT \tilde f (\mu - \nu) .
 \]
Therefore, the conclusion of the corollary follows from \eqref{non-summetric}.
\end{proof}

\subsection{Proof for Proposition 3.7 in the main text}\label{sec:app_geodesic-space-part1}

In order to prove Proposition 
3.7 in the main text, we need the following  auxiliary result.
\begin{lemma}\label{iff=0}  
Assume that $w_1>0$  and $w_2>0$. Then, $d(\mu,\nu)=0$ implies that  $\mu=\nu$. 
\end{lemma}
\begin{proof}
Assume that  $d(\mu,\nu)=0$. 
Let $\gamma^0$ be  an optimal plan for $ \mathrm{ET}_\lambda(\mu,\nu)$, and set  $m:=\gamma^0(\calT\times\calT)$. Then, $m\leq \min\{ \mu(\calT), \nu(\calT)\}$, and hence we obtain from  
Problem (3) in the main text that
\begin{align*} 
&\int_\calT  w_1 [1-f_1(x)] \mu(dx)  
+ \int_\calT  w_2 [1-f_2(x)] \nu(dx) 
+ b \, \int_{\calT \times \calT} d_\calT(x,y) \gamma^0(dx, dy)\\
&=\mathrm{ET}_{\lambda}(\mu,\nu) +\lambda b m \leq \mathrm{ET}_{\lambda}(\mu,\nu) +\frac{b\lambda}{2}\big[ \mu(\calT) +  \nu(\calT)\big] =d(\mu,\nu)=0.
\end{align*}
Thus,
\begin{align*}
\int_\calT  w_1 [1-f_1(x)] \mu(dx)  
=\int_\calT  w_2 [1-f_2(x)] \nu(dx) =\int_{\calT \times \calT} d_\calT(x,y)\gamma^0(dx, dy) =0.
\end{align*}
 Since $w_1$ and  $w_2$ are positive, it follows in particular  that $f_1=1$ $\mu$-a.e.  and  $f_2=1$ $\nu$-a.e. That is,  $\gamma^0_1 =\mu$ and $\gamma^0_2 =\nu$. 
Moreover, the above last identity implies that $\gamma^0$ is supported on the diagonal $(y=x)$.  Therefore,  for any continuous function $\varphi$ on $\calT$ we have
\[
\int_\calT \varphi(x) \mu(dx)=\int_{\calT \times \calT}  \varphi(x) \gamma^0(dx, dy)=\int_{\calT \times \calT}  \varphi(y) \gamma^0(dx, dy)=\int_\calT \varphi(y) \nu(dy).
\]
We thus conclude that $\mu=\nu$.
\end{proof}

\begin{proof}{[Of Proposition 
3.7 in the main text]}

i)  This follows immediately from Corollary 
3.5 in the main text.

ii)  By Corollary 
3.5 in the main text, it  is clear that   $d(\mu,\nu)\geq 0$ and $d(\mu,\mu)= 0$. Also, if $d(\mu,\nu)=0$, then by Lemma~\ref{iff=0}, we have  $\mu =\nu$.  It is obvious that $d$ satisfies the triangle inequality. 

iii) Due to the assumption $w_1 = w_2$ we have  $f\in \mathbb{L}$ if and only if $-f\in \mathbb{L}$. It follows that  $d(\mu,\nu)= d(\nu,\mu)$. This together with ii) implies that $(\calM(\calT), d)$ is a metric space. Its completeness follows from \cite[Proposition 4]{P1}.  As a complete metric space, it is well known that $(\calM(\calT), d)$ is a geodesic space if and only if for every $\mu,\nu \in\calM(\calT)$ there exists $\sigma\in \calM(\calT)$ such that 
\[
d(\mu,\sigma)= d(\nu,\sigma) =\frac12 d(\mu,\nu).
\]
To verify the latter, take $\sigma :=\frac{\mu +\nu }{2}$. Then using Corollary 
3.5 in the main text, we  obtain 
\[
d(\mu,\sigma)= \frac12 \sup_{f\in \mathbb{L} } \int_\calT f(\mu-\nu) = \frac12 d(\mu,\nu)
\]
and
\[
d(\nu,\sigma)= \frac12 \sup_{f\in \mathbb{L} } \int_\calT f(\nu-\mu) = \frac12 d(\nu,\mu)= \frac12 d(\mu,\nu).
\]
\end{proof}

\subsection{Proof for Proposition 3.8 in the main text}\label{sec:app_tildeET}

\begin{proof}
Observe that
\begin{align*}
\widetilde{\mathrm{ET}}_\lambda^\alpha(\mu,\nu) 
&=- \frac{b\lambda}{2}\big[ \mu(\calT) +  \nu(\calT)\big] \\
&\quad +  \sup \Big\{
s[\mu(\calT) -\nu(\calT)] :\, s \in \big[ - \frac{b\lambda}{2} -w_2(r)+\alpha, w_1(r) + \frac{b\lambda}{2} -\alpha\big] \Big\}\\
&\quad + \sup \left\{ \int_\calT \Big[  \int_{[r,x]} g(y) \omega(dy\Big] (\mu - \nu)(dx) :\,    \|g\|_{L^\infty(\calT)}\leq b  \right\}.
\end{align*}
The first supremum equals to $[w_1(r) +\frac{b\lambda}{2}-\alpha] [ \mu(\calT)- \nu(\calT)]$  if $ \mu(\calT)\geq \nu(\calT)$ and equals to $-[w_2(r) +\frac{b\lambda}{2}-\alpha] [ \mu(\calT)- \nu(\calT)]$ if 
 $\mu(\calT)< \nu(\calT)$. On the other hand,
by the same arguments as in \cite[p.575-576]{EM}, we see that the second supremum  equals to 
$\int_{\calT} | \mu(\Lambda(x)) -  \nu(\Lambda(x))| \, \omega(dx)$. Putting them together, we obtain the desired formula for $\widetilde{\mathrm{ET}}_\lambda^\alpha(\mu,\nu) $.
\end{proof}

\subsection{Proof for Proposition 3.9 in the main text}\label{sec:app_bound_tildeET}

\begin{proof}
The inequality $\mathrm{ET}_\lambda(\mu,\nu) \leq  \widetilde{\mathrm{ET}}_\lambda^0(\mu,\nu)$  holds  due to  $\mathbb L \subset  \mathbb L_0$ and Corollary 
3.5 in the main text. Next, let  
\[
2b L(\calT)\leq \alpha\leq \frac12 [b\lambda + w_1(r) + w_2(r)].
\]
Then, thanks to Corollary 
3.5 in the main text, the stated lower bound will follow if  $\mathbb L_\alpha\subset  \mathbb L $. This is achieved if we can show that any  $f\in \mathbb L_\alpha$ satisfies $-w_2 - \frac{b\lambda}{2}  \leq f \leq    w_1 + \frac{b\lambda}{2}$. Indeed, for such function  $f$ we have 
\[
f(x)=s +  \int_{[r,x]} g(y) \omega(dy),
\]
with  $s \in \big[ -w_2(r)- \frac{b\lambda}{2} +\alpha, w_1(r) + \frac{b\lambda}{2} -\alpha\big] $ and $  \|g\|_{L^\infty(\calT)}\leq b$. This together the $b$-Lipschitz property of $w_1, w_2$ gives for every $x\in\calT$ that 
\[
f(x) \leq s +\|g\|_{L^\infty(\calT)} \omega([r,x])\leq
w_1(r) +\frac{b\lambda}{2}-\alpha +b L(\calT)\leq w_1(x)+\frac{b\lambda}{2} -\alpha + 2b L(\calT) \leq w_1(x)+\frac{b\lambda}{2}
\]
and
\begin{align*}
f(x)&\geq s -\|g\|_{L^\infty(\calT)} \omega([r,x])\geq
- w_2(r)-\frac{b\lambda}{2}+\alpha -b  L(\calT)\\
&\geq - w_2(x)-\frac{b\lambda}{2}+\alpha -2b  L(\calT)\geq  - w_2(x)-\frac{b\lambda}{2}.
\end{align*}
It follows that $f\in   \mathbb L $. Thus,   $\mathbb L_\alpha\subset  \mathbb L $ and we obtain 
\[
\widetilde{\mathrm{ET}}_\lambda^\alpha(\mu,\nu)  \leq \mathrm{ET}_\lambda(\mu,\nu).
\]  

\end{proof}

\subsection{Proof of Proposition 3.10 in the main text}\label{sec:app_geodesic-space-part2}

We begin with the following  auxiliary result.
\begin{lemma}\label{equal-measure}
Let $\mu,\nu\in \calM(\calT)$. Then, $\mu=\nu$ if and only if  $\mu(\Lambda(x)) = \nu(\Lambda(x))$ for every $x$ in $\calT$.
\end{lemma}
\begin{proof}
It is obvious that $\mu=\nu$ implies that $\mu(\Lambda(x)) = \nu(\Lambda(x))$  for every $x$ in $\calT$.
Now assume that $\mu(\Lambda(x)) = \nu(\Lambda(x))$ for every $x$ in $\calT$. We first claim that  $\mu(\{a\}) = \nu(\{a\})$ for any $a\in\calT$. Indeed, if $a$ is not a node then we have $ \Lambda(a)\setminus \Lambda(a_n) \downarrow \{a\}$, where $\{a_n\}_{n=1}^\infty$ is a sequence of distinct points on the same edge as $a$ and converges to $a$ from below. Hence,
\[
\mu(\{a\}) = \lim_{n\to\infty} \big[ \mu(\Lambda(a)) -\mu(\Lambda(a_n))\big]
=\lim_{n\to\infty} \big[ \nu(\Lambda(a)) -\nu(\Lambda(a_n))\big]=\nu(\{a\}).
\]
In case $a$ is a common node for edges $e_1,..., e_k$, then we have $ \Gamma(a)\setminus \cup_{i=1}^k\Gamma(a^i_n) \downarrow \{a\}$, where $\{a^i_n\}_{n=1}^\infty$ is a sequence of distinct points on edge $e_i$ that converges to $a$ from below. Then, we obtain
\begin{align*}
  \mu(\{a\}) = \lim_{n\to\infty} \big[ \mu(\Lambda(a)) -\sum_{i=1}^k \mu(\Lambda(a^i_n))\big]
=\lim_{n\to\infty} \big[ \nu(\Lambda(a)) -\sum_{i=1}^k\nu(\Lambda(a^i_n))\big]=\nu(\{a\}).  
\end{align*}
Thus, the claim is proved. On the other hand, for any points $x,y$ belonging to the same edge
\[
\mu([x,y))= \mu(\Lambda(x))-\mu( \Lambda(y))=\nu(\Lambda(x))-\nu( \Lambda(y))
= \nu([x,y)).
\] 
Thus, by combining them, we infer further that 
$\mu([x,y])=  \nu([x,y])$ for any $x,y\in\calT$. It follows that $\mu=\nu$, and the proof is complete.
\end{proof}

\begin{proof}{[Of Proposition 
3.10 in the main text]}
We note first that the quantity $d_\alpha$ depends only on the values of the weights at the root $r$ of the tree. This comes from the fact that only $w_1(r)$ and $w_2(r)$ are used in the definition of $\mathbb L_\alpha$. The proofs of i) and iii) are exactly the same as that of Proposition 
3.7 in the main text.

For ii), it follows from the fact 
\begin{equation*}
d_\alpha(\mu,\nu) =  \sup \left\{ \int_\calT f (\mu - \nu):\, f\in \mathbb{L}_\alpha  \right\}
\end{equation*}
that   $d_\alpha( \mu,\nu)\geq 0$, $d_\alpha(\mu,\mu)= 0$, and $d_\alpha$ satisfies the triangle inequality.  Also, if $d_\alpha(\mu,\nu)=0$, then by  Proposition 
3.8 in the main text, we get 
\[
 \big[w_i(r) +\frac{b\lambda}{2} -\alpha\big] |\mu(\calT)-\nu(\calT)|  + \int_{\calT} | \mu(\Lambda(x)) -  \nu(\Lambda(x))| \, \omega(dx)
 =0.
\]
As $ \big[w_i(r) +\frac{b\lambda}{2} -\alpha\big]> 0$ by the assumption, we  must have 
$\mu(\calT) =\nu(\calT)$ and $\int_{\calT} | \mu(\Lambda(x)) -  \nu(\Lambda(x))| \, \omega(dx)
 =0$. Therefore,  $\mu(\Lambda(x)) = \nu(\Lambda(x))$ for every  $x\in\calT$.
By using Lemma~\ref{equal-measure},  we then conclude that $\mu=\nu$. 

Alternatively, we can argue as follows. Assume that $d_\alpha(\mu,\nu)=0$. Since
\[
\mathbb L_\alpha \supset \tilde{\mathbb L} :=\left\{ f: \, -w_2(r)- \frac{b\lambda}{2}+\alpha\leq f(x)\leq  w_1(r) + \frac{b\lambda}{2} -\alpha,\,   \|f\|_{Lip(\calT)}\leq b \right\},
\]
we have 
\[
0\leq \sup_{f\in \tilde{\mathbb L} } \int_\calT f (\mu - \nu) \leq  d_\alpha(\mu,\nu)=0.
\]
Thus, $\sup_{f\in \tilde{\mathbb L} } \int_\calT f (\mu - \nu) =0$. Then, by applying Corollary 
3.5 in the main text and Lemma~\ref{iff=0} for  constant 
weights $\tilde w_1 :=  w_1(r) + \frac{b\lambda}{2} -\alpha>0$ and $\tilde w_2 :=  w_2(r) + \frac{b\lambda}{2} -\alpha>0$, we obtain that $\mu=\nu$. 
\end{proof}

\subsection{Proof for Proposition 3.11 in the main text}\label{sec:app_negative_definite}

\begin{proof}

Let $\mathbf{f}(x_i, x_j) = \tilde{a}(x_i + x_j)$ for $\tilde{a}, x_i, x_j \in \mathbb{R}$. We first prove that $\mathbf{f}$ is negative definite.

For all $n \ge 2$, for $c_1, c_2, \dotsc, c_n$ such that $\sum_{i=1}^n c_i = 0$. Given $x_1, x_2, \dotsc, x_n \in \mathbb{R}$, we have
\[
\sum_{i, j} c_i c_j \mathbf{f}(x_i, x_j) = \sum_{i,j} c_i c_j \tilde{a} x_i + \sum_{i,j} c_i c_j \tilde{a} x_j \le 0. 
\]
Therefore, $\mathbf{f}$ is negative definite.

From Proposition 
3.8 in the main text, we have
\[
\widetilde{\mathrm{ET}}_\lambda^\alpha(\mu,\nu) 
= - \frac{b\lambda}{2}\big[ \mu(\calT) +  \nu(\calT)\big]  +   \big[w_i(r) +\frac{b\lambda}{2} -\alpha\big] |\mu(\calT)-\nu(\calT)|  + \int_{\calT} | \mu(\Lambda(x)) -  \nu(\Lambda(x))| \, \omega(dx).
\]

The first term is negative definite since $\mathbf{f}$ is negative definite. Additionally, the second and third terms are equivalent to the weighted $\ell_1$ distance with nonnegative weights (i.e., $\alpha \le w_i(r) +\frac{b\lambda}{2}$ and lengths of edges in tree $\calT$ are nonnegative). Therefore, the second  and third terms are also negative definite. Hence, $\widetilde{\mathrm{ET}}_{\lambda}^{\alpha}$ is negative definite.  

From Proposition 
3.10 in the main text, we have
\[
d_{\alpha}(\mu, \nu) = \widetilde{\mathrm{ET}}_\lambda^\alpha(\mu, \nu) + \frac{b\lambda}{2}\big[\mu(\calT) + \nu(\calT) \big].
\]
Both terms are negative definite. Therefore, $d_\alpha$ is also negative definite. 

\end{proof}

\section{Further Experimental Results}\label{supp:sec:experiments}

In this section, we illustrate further experimental results.

\subsection{Further Results on the Efficient Approximation of $\widetilde{\ET}_{\lambda}^{0}$ for $\ET_\lambda$}
In this section, we consider some further setups. 

\paragraph{Change $\lambda$.}

In Figure~\ref{fg:Diff_KT_tET_Lambda}, we use the same setup as in Figure 
2 in the main text, but set the Lipschitz $a_1=\frac{b}{2}=0.5$ for $w_1, w_2$. It shows that when $\lambda$ is increased, $\tildeET_{\lambda}^{0}$ is farther to $\ET_\lambda$.

\begin{figure}[h]
\vspace{-8pt}
     \centering
     \begin{subfigure}[t]{0.32\textwidth}
    \centering
      \includegraphics[width=0.65\textwidth]{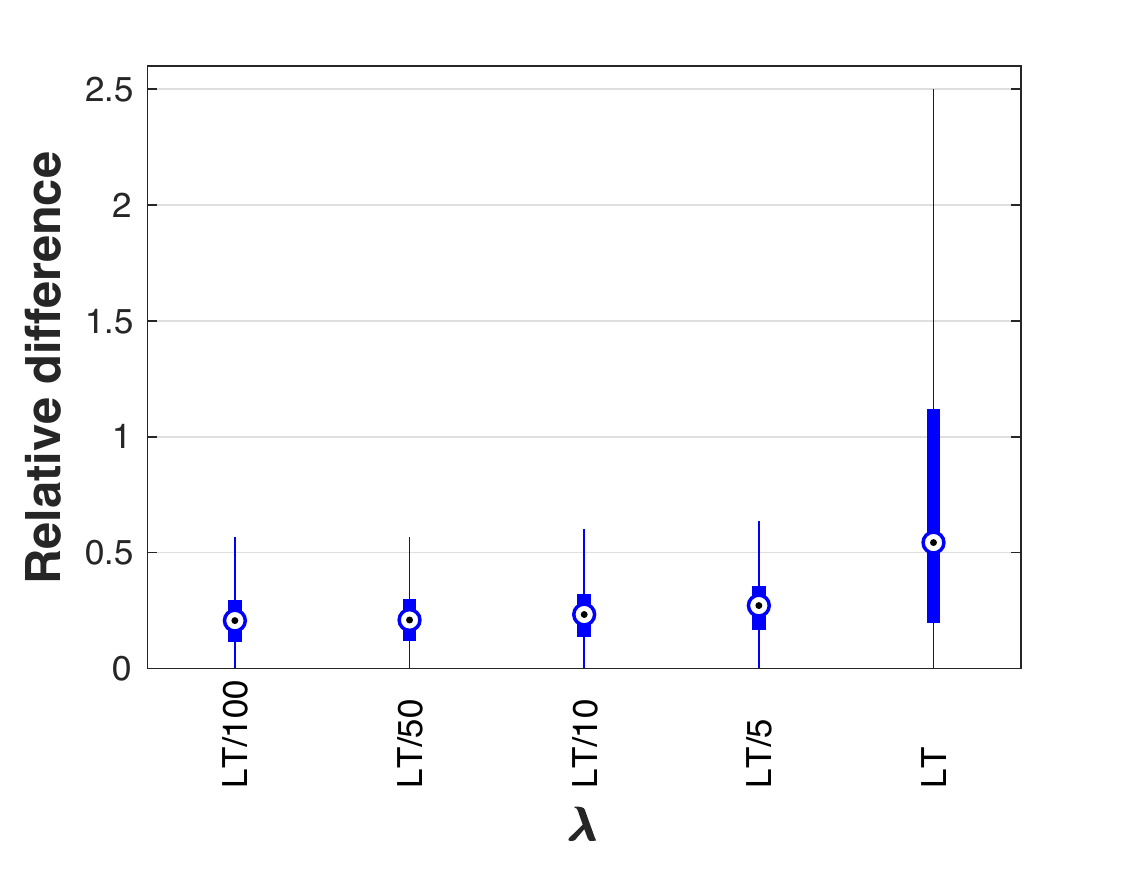}
\vspace{-4pt}        
  \caption{}
  \label{fg:Diff_KT_tET_Lambda}
    \end{subfigure} 
    \hfill
     \begin{subfigure}[t]{0.32\textwidth}
    \centering
 \includegraphics[width=0.65\textwidth]{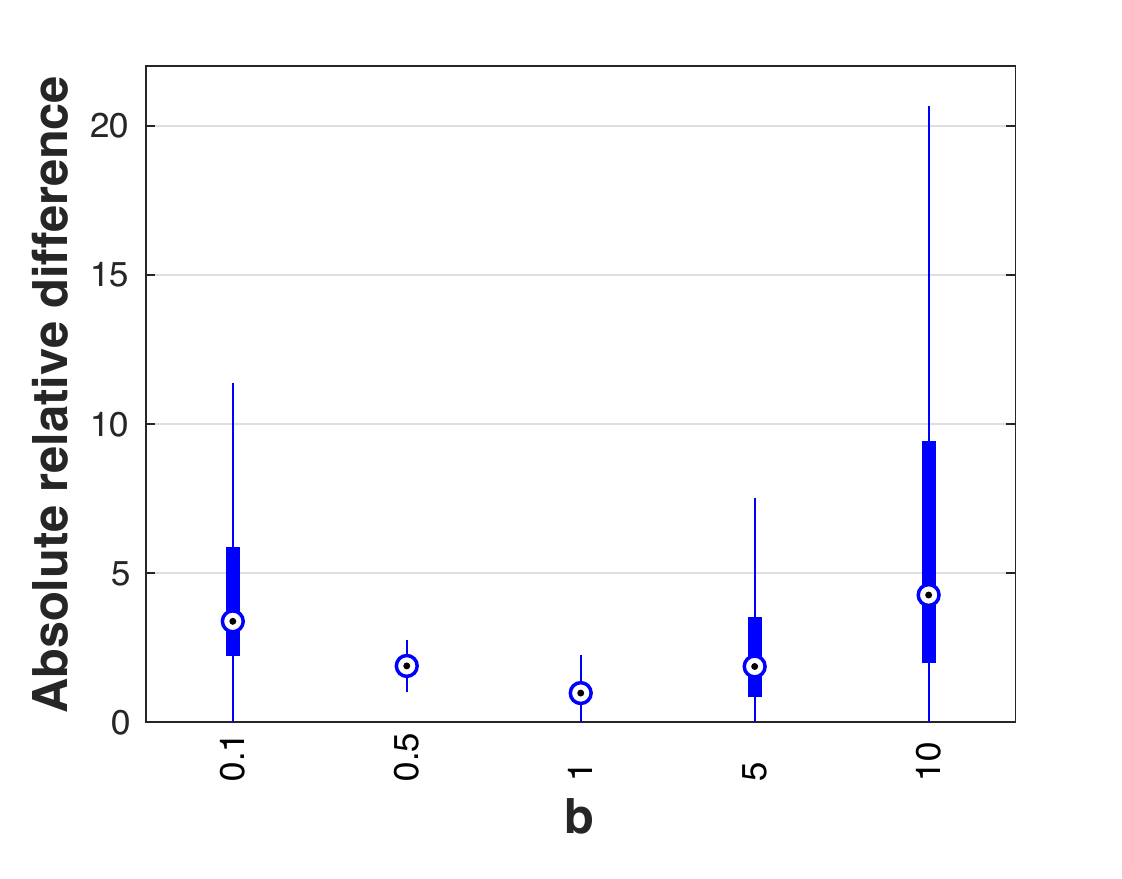}
\vspace{-4pt}        
  \caption{}
  \label{fg:Diff_KT_tET_b_ConstantW}
    \end{subfigure} 
   \hfill
    \begin{subfigure}[t]{0.32\textwidth}
        \centering
    \includegraphics[width=0.65\textwidth]{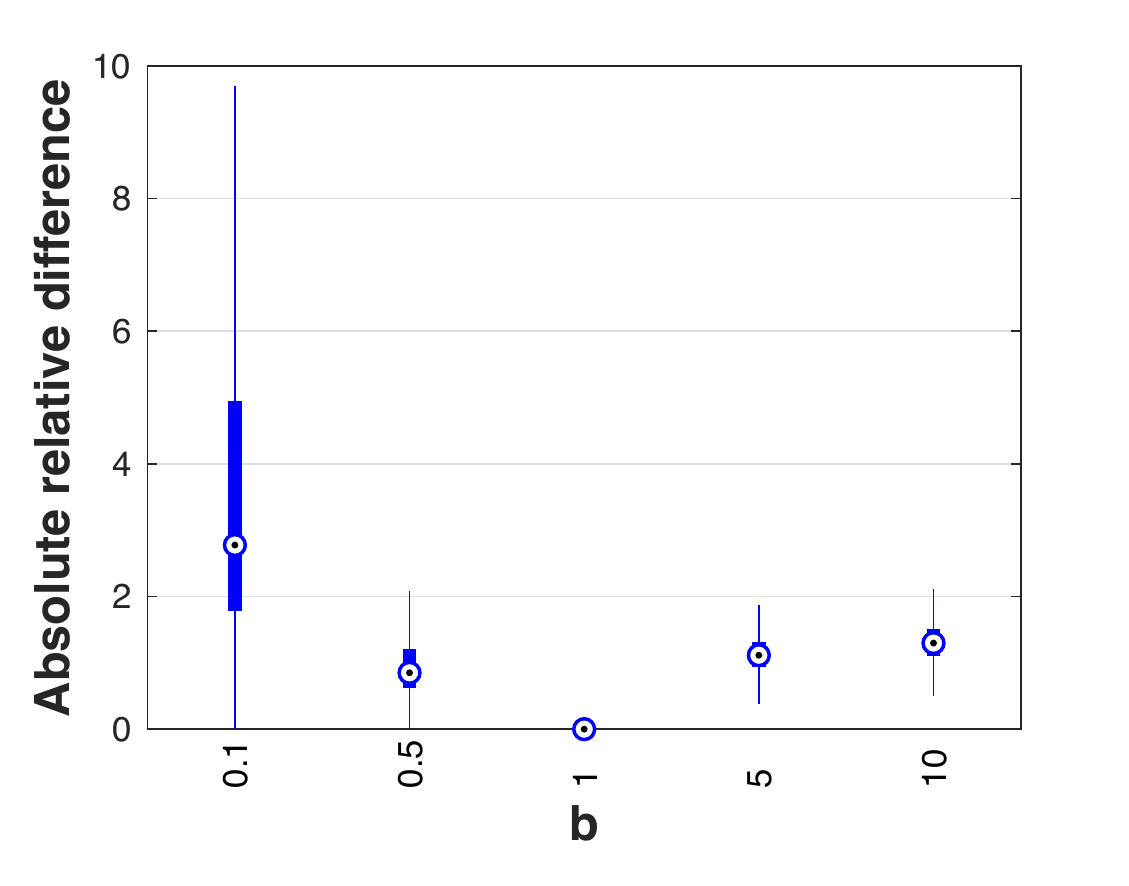}
 \vspace{-4pt}   
 \caption{}
  \label{fg:Diff_KT_tET_bc}
    \end{subfigure}
    
\vspace{-4pt}   
    \caption{In (\subref{fg:Diff_KT_tET_Lambda}), an illustration about the relative difference between $\tildeET_{\lambda}^{0}$ and $\ET_\lambda$ w.r.t. $\lambda$. LT is the longest path from a root to a node in tree $\Tt$ (LT := $L_{\Tt}$). Lipchitz for functions $w_1, w_2$ is $a_1=0.5$ (where $b=1$). In (\subref{fg:Diff_KT_tET_b_ConstantW}, \subref{fg:Diff_KT_tET_bc}), an illustration about the absolute relative difference between $\tildeET_{\lambda}^{0}$ and $\ET_\lambda$, i.e., $(\tildeET_{\lambda}^{0} - \ET_\lambda)/\left|\ET_\lambda\right|$, w.r.t. $b$. For (\subref{fg:Diff_KT_tET_b_ConstantW}), the weight functions $w_1, w_2$ are set constants ($a_1=0$, or $w_1 = w_2 = a_0$) while for (\subref{fg:Diff_KT_tET_bc}), the weight functions $w_1, w_2$ are set with largest Lipchitz ($a_1=b$).}
 \vspace{-10pt}
\end{figure}

\paragraph{Change $b$.} We consider 2 following cases:

\paragraph{$\bullet$ For constant functions $w_1, w_2$ (with $a_1=0$).} We use the same setup as in Figure 
1 in the main text, but with constant functions for $w_1, w_2$ (i.e., $a_1=0$, or $w_1 = w_2 = a_0$), and change $b$. We set $\lambda=a_0=1$. In Figure~\ref{fg:Diff_KT_tET_b_ConstantW}, we illustrate that when the regularization $b$ between entropy and partial matching is farther to $1$ (one of the two terms is more weighted, see Equation 
(2) in the main text), $\tildeET_{\lambda}^{0}$ is farther to $\ET_\lambda$.

\paragraph{$\bullet$ For functions $w_1, w_2$ with largest Lipschitz $a_1=b$.} We use the same setup as in Figure~~\ref{fg:Diff_KT_tET_b_ConstantW}, but with $a_1=b$. Figure~\ref{fg:Diff_KT_tET_bc} shows similar results as in Figure~~\ref{fg:Diff_KT_tET_b_ConstantW} for $a_1=0$. For the largest Lipchitz for functions $w_1, w_2$ (i.e., $a_1=b$), but for $b=a_0=1$, $\tildeET$ is almost identical to \KT, but they are different when when the regularization $b$ between entropy and partial matching is farther to $1$ (one of the two terms is more weighted, see Equation 
(2) in the main text). 

\subsection{Further Results w.r.t. $\alpha$}\label{supp:sec:results_wrt_alpha}

We illustrate further SVM results of $d_{\alpha}$ and $\widetilde{\ET}_{\lambda}^{\alpha}$ w.r.t. value of $\alpha$ in \texttt{TWITTER, RECIPE, CLASSIC, AMAZON} datasets in Figure~\ref{fg:DOC_alpha}, and in \texttt{Orbit, MPEG7} datasets in Figure~\ref{fg:TDA_alpha}. The value of $\alpha$ may affect performances of $d_\alpha$ and $\widetilde{\ET}_{\lambda}^{\alpha}$ in some datasets (e.g., \texttt{RECIPE, AMAZON} datasets for document classification, and \texttt{Orbit} dataset in TDA), but may not sensitive in some other datasets (e.g., \texttt{TWITTER, CLASSIC} datasets for document classification, and \texttt{MPEG7} dataset in TDA). Therefore, although $\alpha = 0$ gives $\widetilde{\ET}_{\lambda}^{\alpha}$ good property as in Proposition 3.9 in the main text (upper bound for $\ET_{\lambda}$), there is a possibility to choose suitable value for $\alpha$ (e.g., via cross validation) to improve performances of $d_{\alpha}$ and $\widetilde{\ET}_{\lambda}^{\alpha}$ for some certain datasets.

\begin{figure}[h]
\vspace{-8pt}
     \centering
     \begin{subfigure}[t]{0.67\textwidth}
    \centering
    \includegraphics[width=0.83\textwidth]{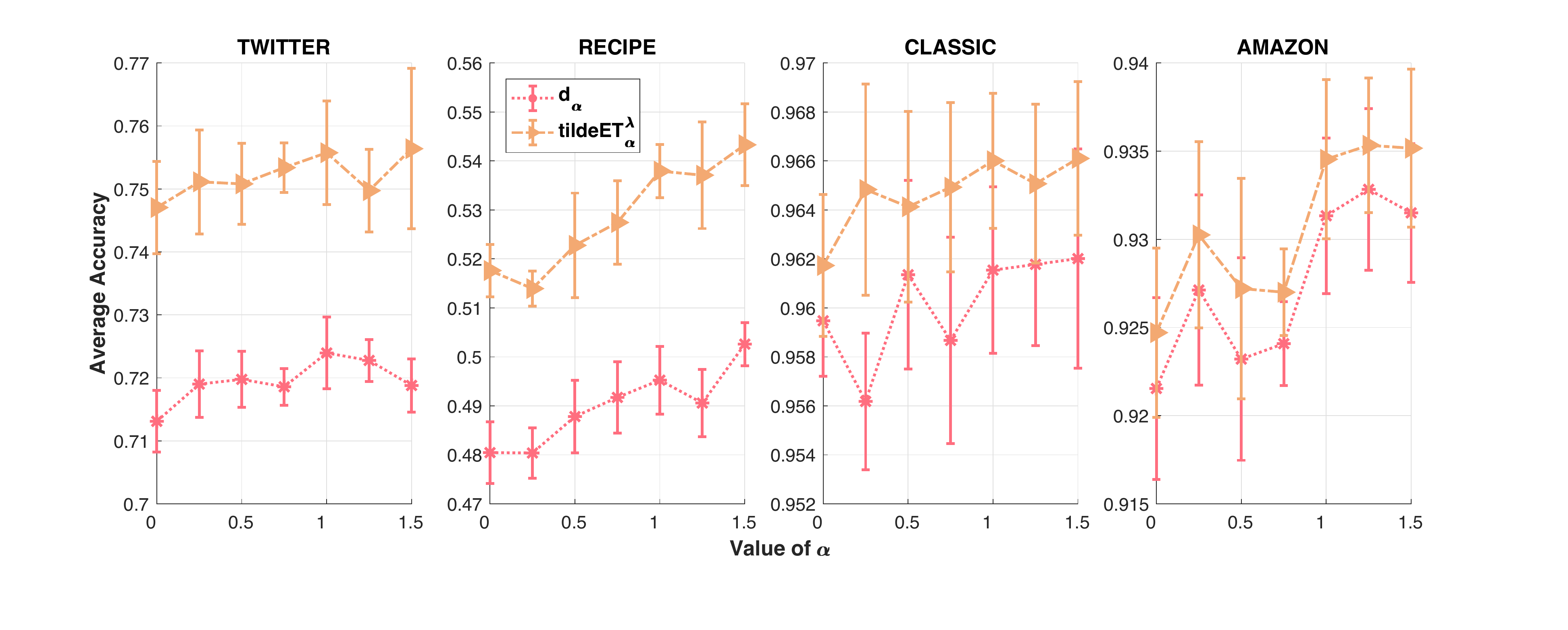}
\vspace{-4pt}        
  \caption{In \texttt{TWITTER, RECIPE, CLASSIC, AMAZON} datasets.}
  \label{fg:DOC_alpha}
    \end{subfigure} 
   \hfill
    \begin{subfigure}[t]{0.31\textwidth}
        \centering
   \includegraphics[width=\textwidth]{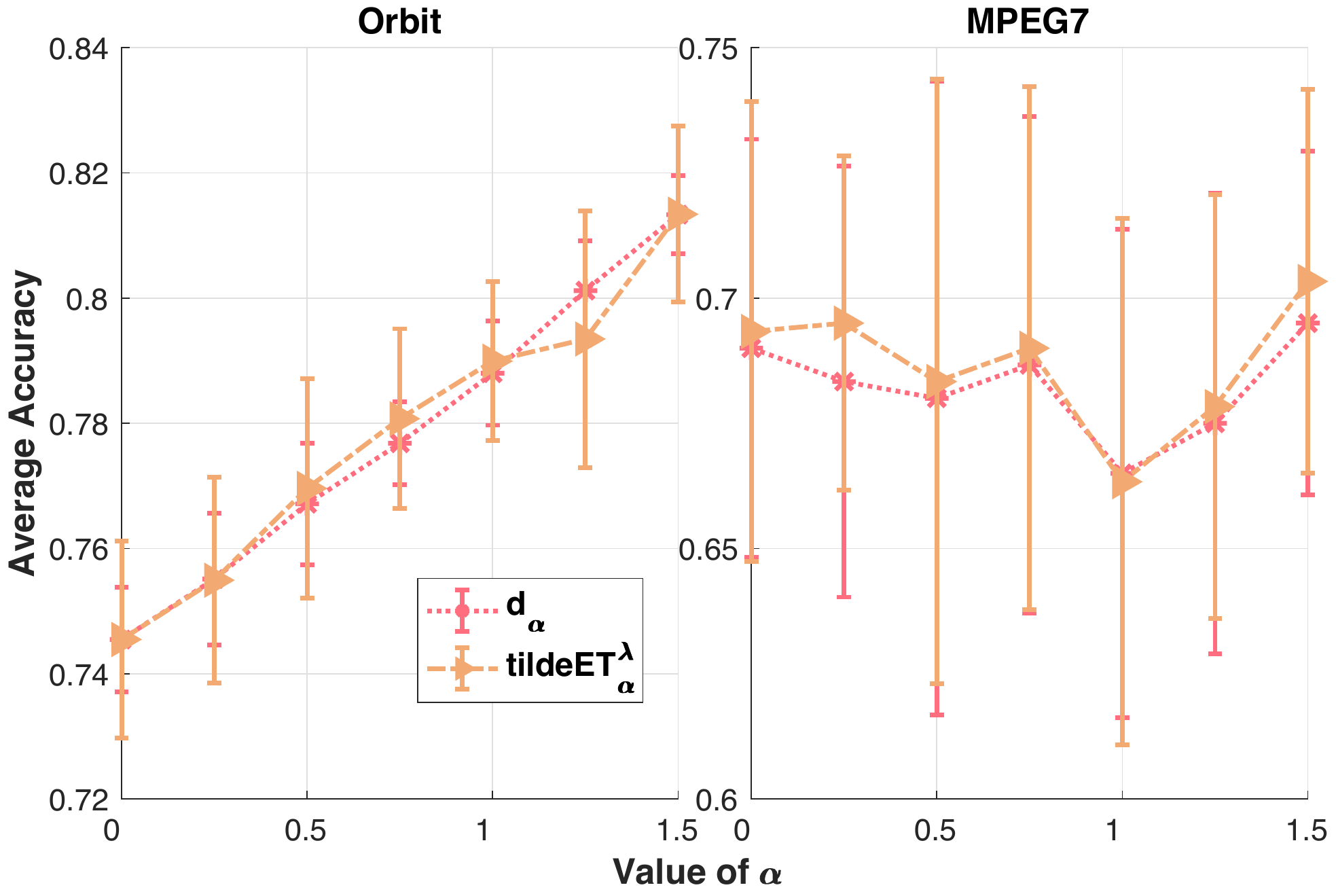}
 \vspace{-16pt}   
 \caption{In \texttt{Orbit, MPEG7} datasets.}
  \label{fg:TDA_alpha}
    \end{subfigure}
    
\vspace{-4pt}   
    \caption{SVM results of $d_{\alpha}$ and $\widetilde{\ET}_{\lambda}^{\alpha}$ w.r.t. value of $\alpha$ with 10 tree slices.}
 \vspace{-10pt}
\end{figure}


\subsection{Further Results w.r.t. the Number of (Tree) Slices}

Similar as Figure 6 in the main text, we illustrate further SVM results and time consumption for corresponding kernel matrices for document classification (e.g., \texttt{TWITTER, RECIPE, CLASSIC, AMAZON} datasets) and TDA (\texttt{Orbit, MPEG7} datasets in Figure~\ref{fg:DOC_slices} and Figure~\ref{fg:TDA_slices} respectively. For a trade-off between performances and time consumption, one can choose about $n_s = 10$ slices in applications.

\begin{figure}[h]
\vspace{-8pt}
     \centering
     \begin{subfigure}[t]{0.67\textwidth}
    \centering
    \includegraphics[width=0.83\textwidth]{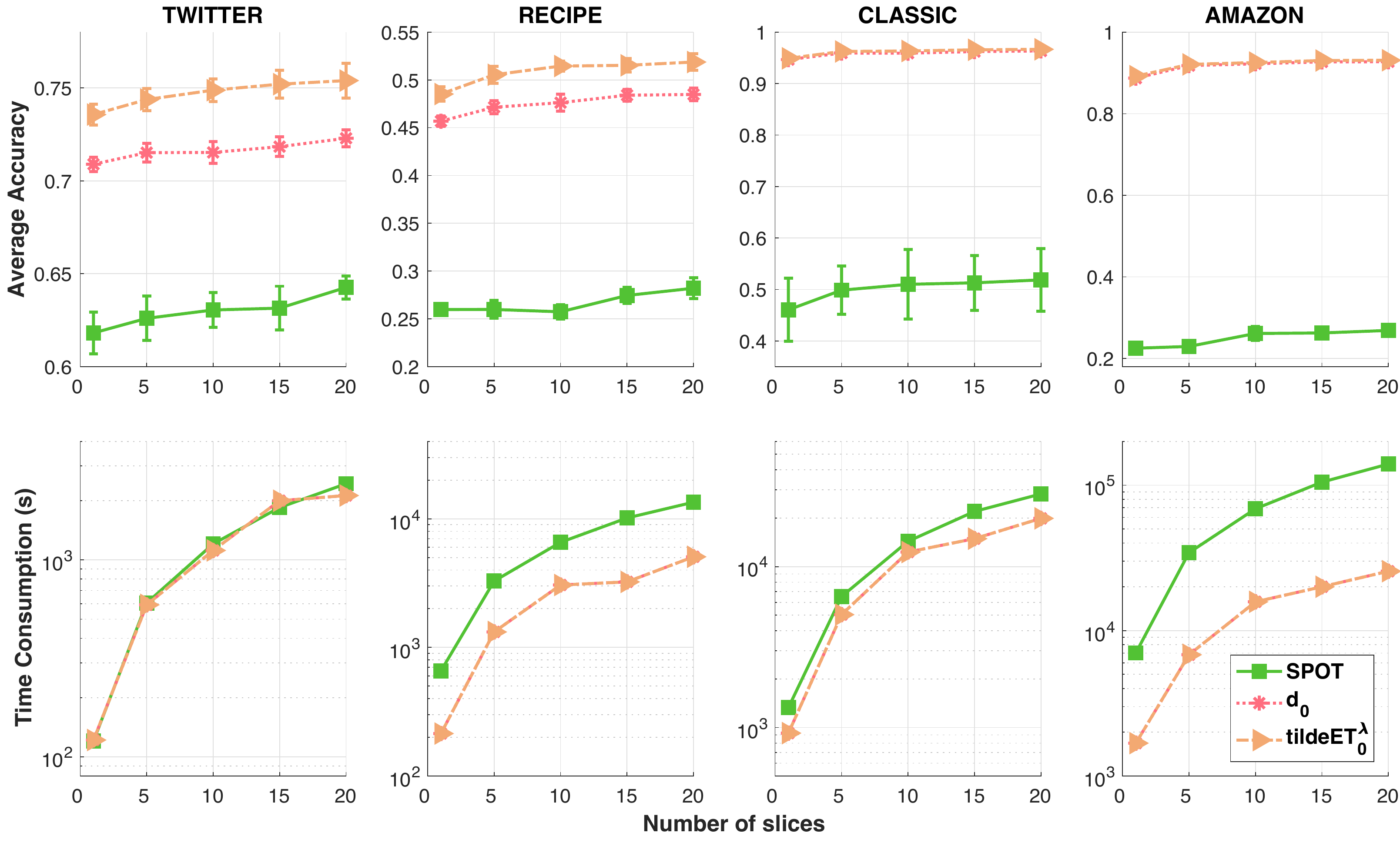}
\vspace{-4pt}        
  \caption{In \texttt{TWITTER, RECIPE, CLASSIC, AMAZON} datasets.}
  \label{fg:DOC_slices}
    \end{subfigure} 
   \hfill
    \begin{subfigure}[t]{0.31\textwidth}
        \centering
    \includegraphics[width=\textwidth]{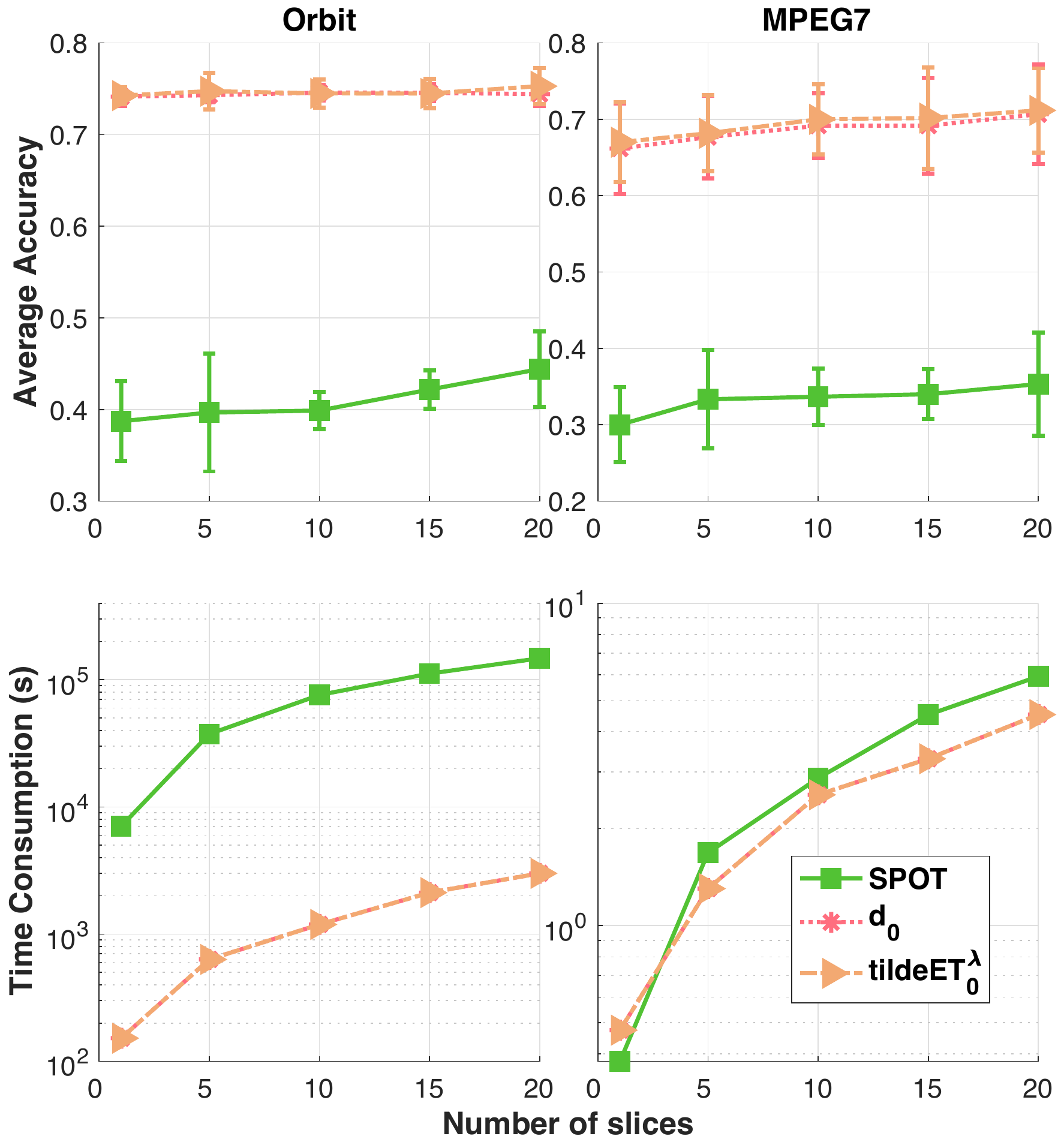}
 \vspace{-16pt}   
 \caption{In \texttt{Orbit, MPEG7} datasets.}
  \label{fg:TDA_slices}
    \end{subfigure}
    
\vspace{-4pt}   
    \caption{SVM results and time consumption for corresponding kernel matrices w.r.t. the number of (tree) slices.}
 \vspace{-10pt}
\end{figure}

\subsection{Further Results w.r.t. Parameters of Tree Metric Sampling}

\paragraph{Document classification.}
\begin{itemize}
\item In Figure~\ref{fg:TWITTER_TM_dAlpha}, Figure~\ref{fg:RECIPE_TM_dAlpha}, Figure~\ref{fg:CLASSIC_TM_dAlpha}, Figure~\ref{fg:AMAZON_TM_dAlpha}, we illustrate further SVM results and time consumption for corresponding kernel matrices of $d_0$ in \texttt{TWITTER, RECIPE, CLASSIC, AMAZON} datasets respectively w.r.t. different parameters for clustering-based tree metric sampling such as the predefined tree deepest level $H_{\Tt}$, and number of tree branches $\kappa$ which is the number of clusters in the farthest-point clustering.
\item In Figure~\ref{fg:TWITTER_TM_tildeET}, Figure~\ref{fg:RECIPE_TM_tildeET}, Figure~\ref{fg:CLASSIC_TM_tildeET}, Figure~\ref{fg:AMAZON_TM_tildeET}, we illustrate further SVM results and time consumption for corresponding kernel matrices of $\tildeETlambda$ in \texttt{TWITTER, RECIPE, CLASSIC, AMAZON} datasets respectively w.r.t. different parameters for clustering-based tree metric sampling such as the predefined tree deepest level $H_{\Tt}$, and number of tree branches $\kappa$ which is the number of clusters in the farthest-point clustering.
\end{itemize}


\begin{figure}[h]
     \centering
     \begin{subfigure}[t]{0.48\textwidth}
    \centering
  \includegraphics[width=\textwidth]{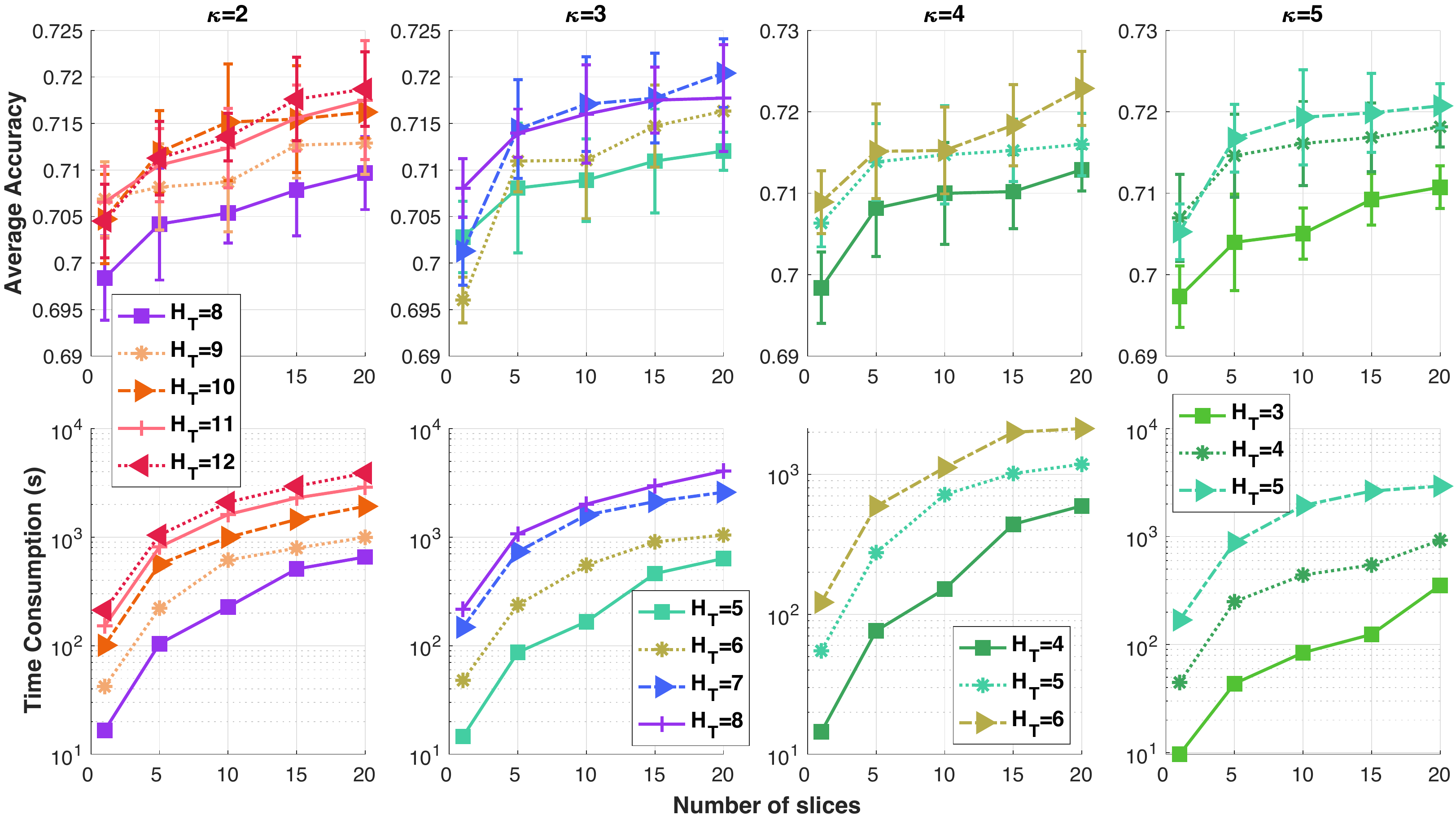}
\vspace{-12pt}        
  \caption{In \texttt{TWITTER} dataset.}
  \label{fg:TWITTER_TM_dAlpha}
    \end{subfigure} 
   \hfill
    \begin{subfigure}[t]{0.48\textwidth}
        \centering
  \includegraphics[width=\textwidth]{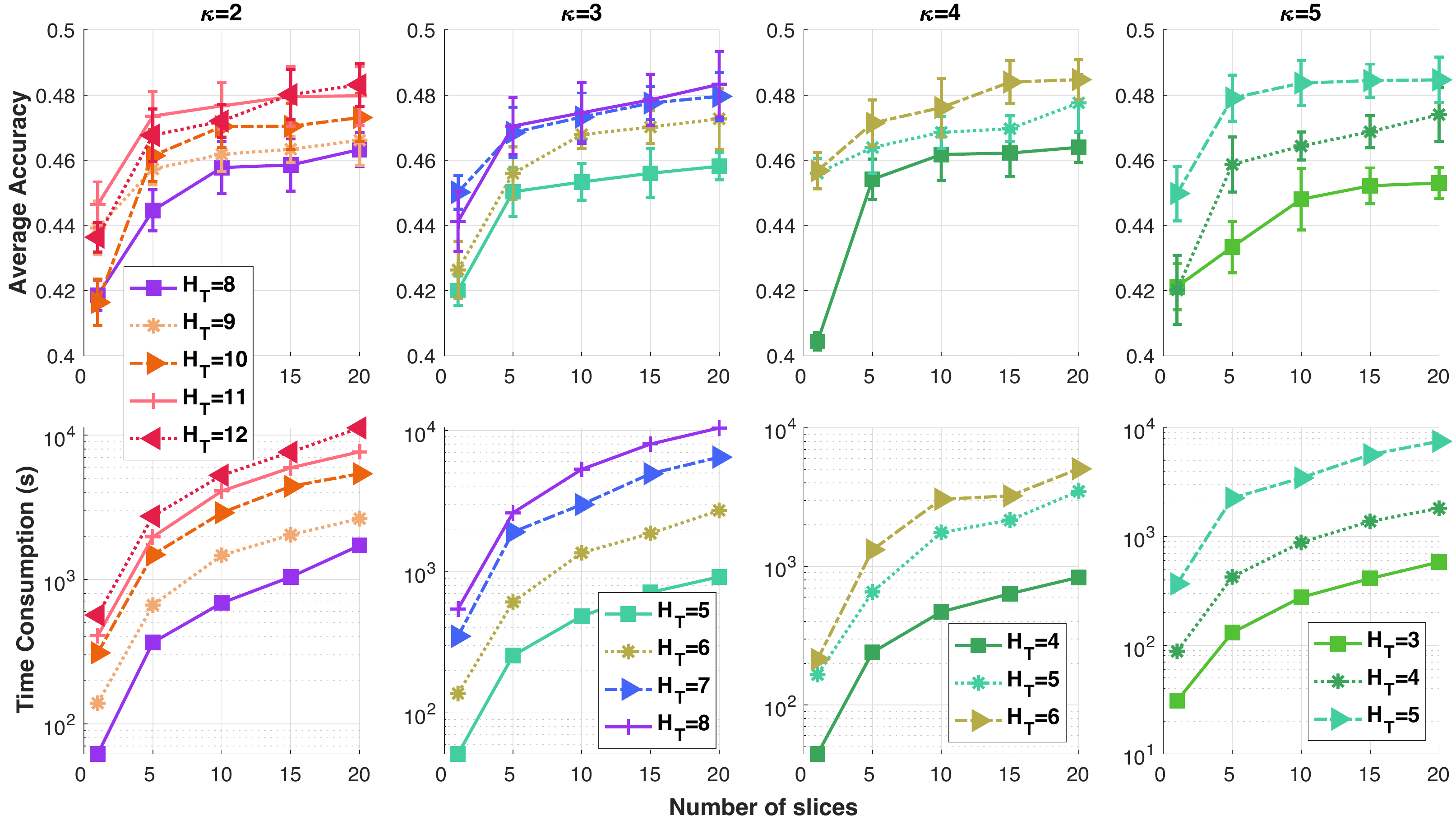}
 \vspace{-12pt}   
 \caption{In \texttt{RECIPE} dataset.}
  \label{fg:RECIPE_TM_dAlpha}
    \end{subfigure}
    
    \vspace{2pt}
    
         \begin{subfigure}[t]{0.48\textwidth}
    \centering
  \includegraphics[width=\textwidth]{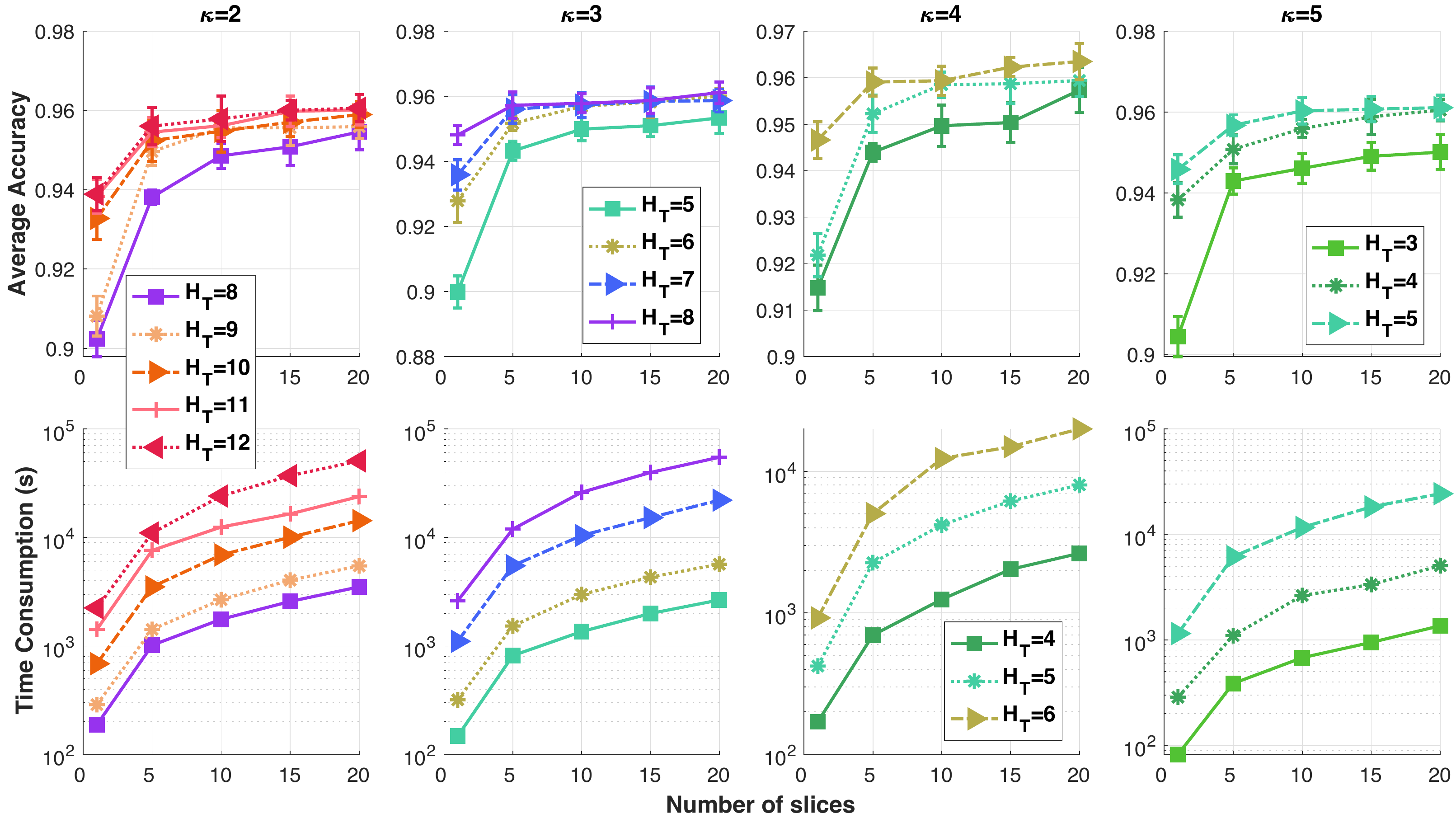}
\vspace{-12pt}        
  \caption{In \texttt{CLASSIC} dataset.}
  \label{fg:CLASSIC_TM_dAlpha}
    \end{subfigure} 
   \hfill
    \begin{subfigure}[t]{0.48\textwidth}
        \centering
  \includegraphics[width=\textwidth]{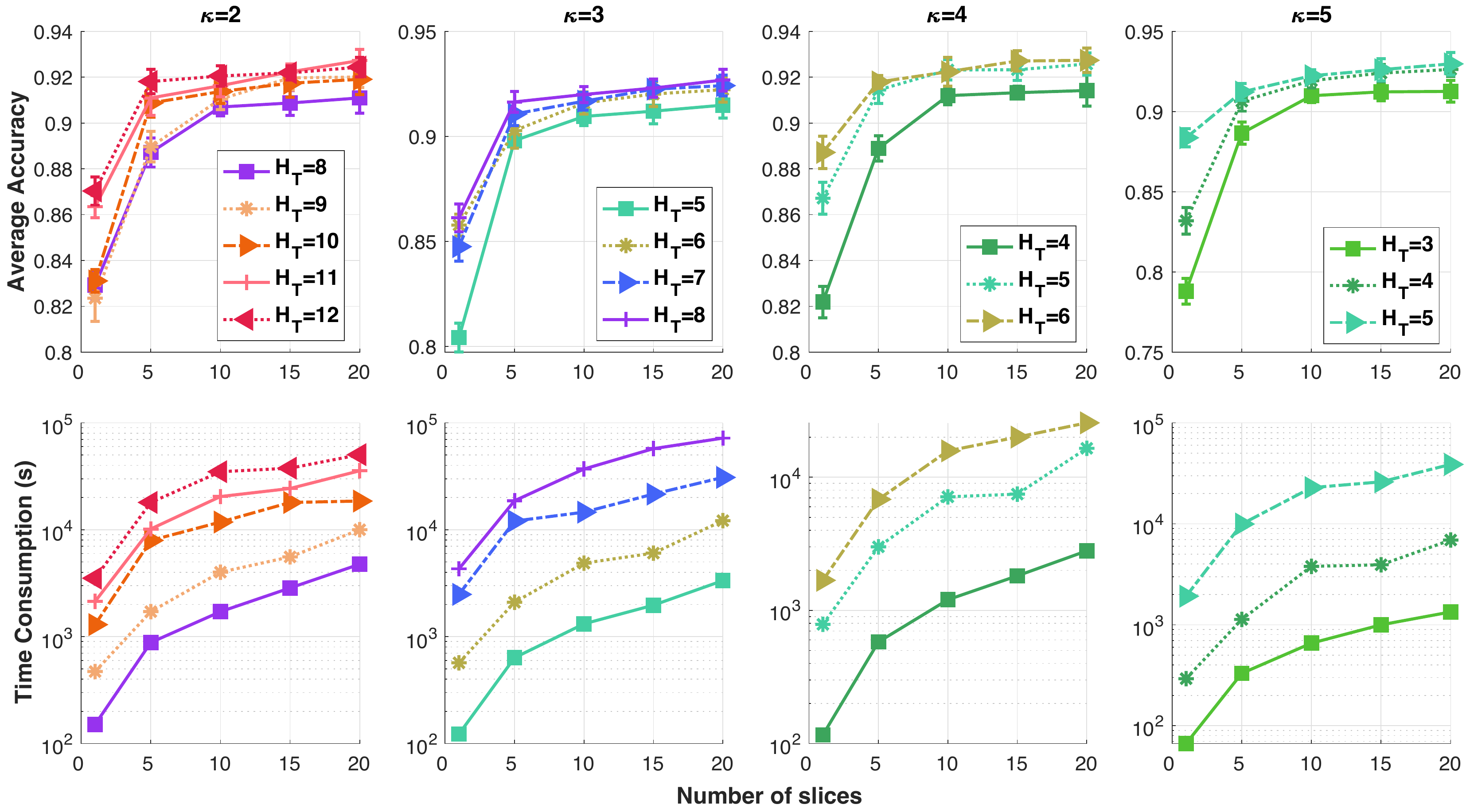}
 \vspace{-12pt}   
 \caption{In \texttt{AMAZON} dataset.}
  \label{fg:AMAZON_TM_dAlpha}
    \end{subfigure}
    
\vspace{-4pt}   
    \caption{SVM results and time consumption for corresponding kernel matrices of $d_0$ w.r.t. different parameters for clustering-based tree metric sampling (predefined tree deepest level $H_{\Tt}$, and number of tree branches $\kappa$---the number of clusters in the farthest-point clustering.).}
 \vspace{-10pt}
\end{figure}

\begin{figure}[h]
     \centering
     \begin{subfigure}[t]{0.48\textwidth}
    \centering
  \includegraphics[width=\textwidth]{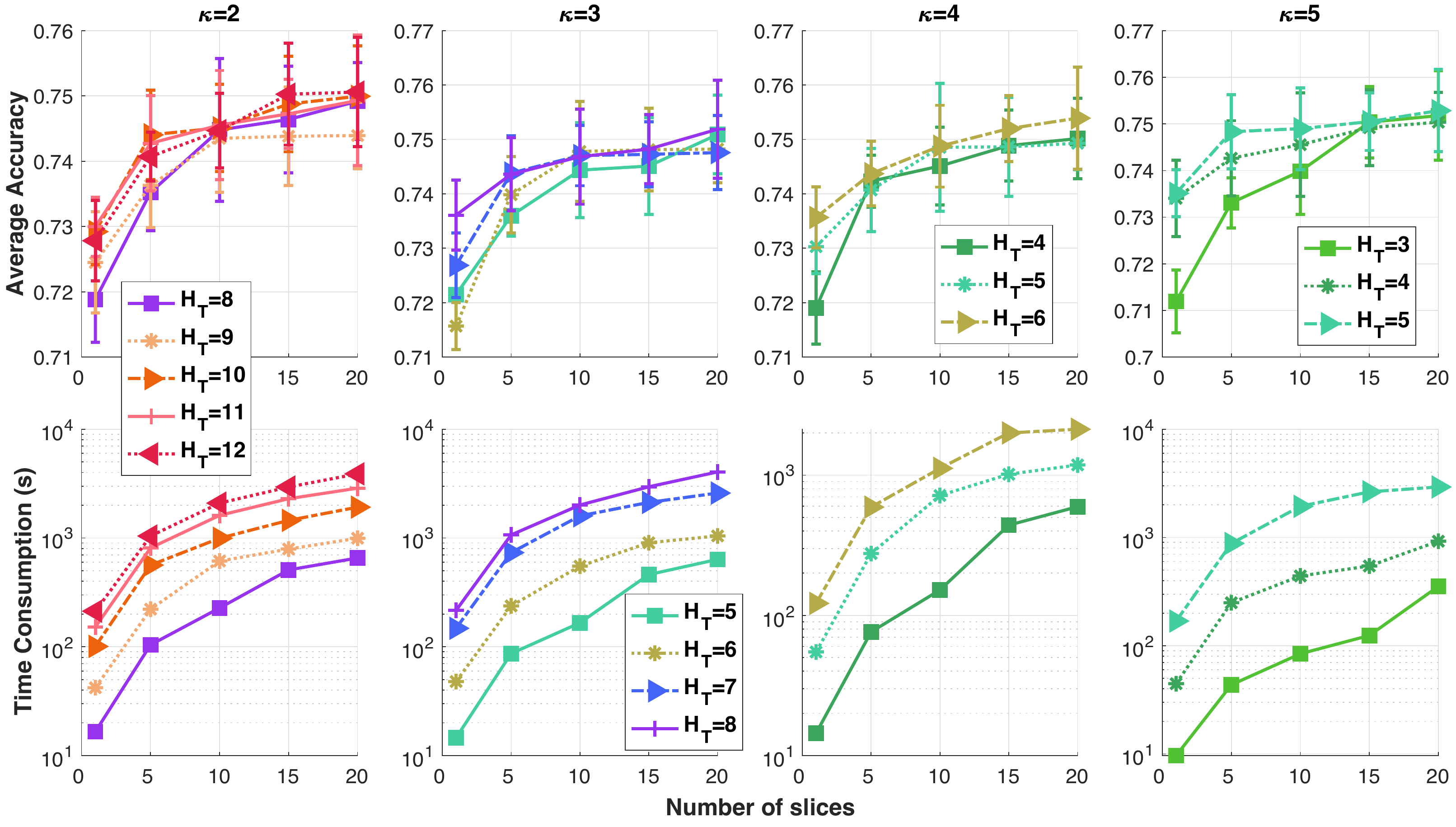}
\vspace{-12pt}        
  \caption{In \texttt{TWITTER} dataset.}
  \label{fg:TWITTER_TM_tildeET}
    \end{subfigure} 
   \hfill
    \begin{subfigure}[t]{0.48\textwidth}
        \centering
  \includegraphics[width=\textwidth]{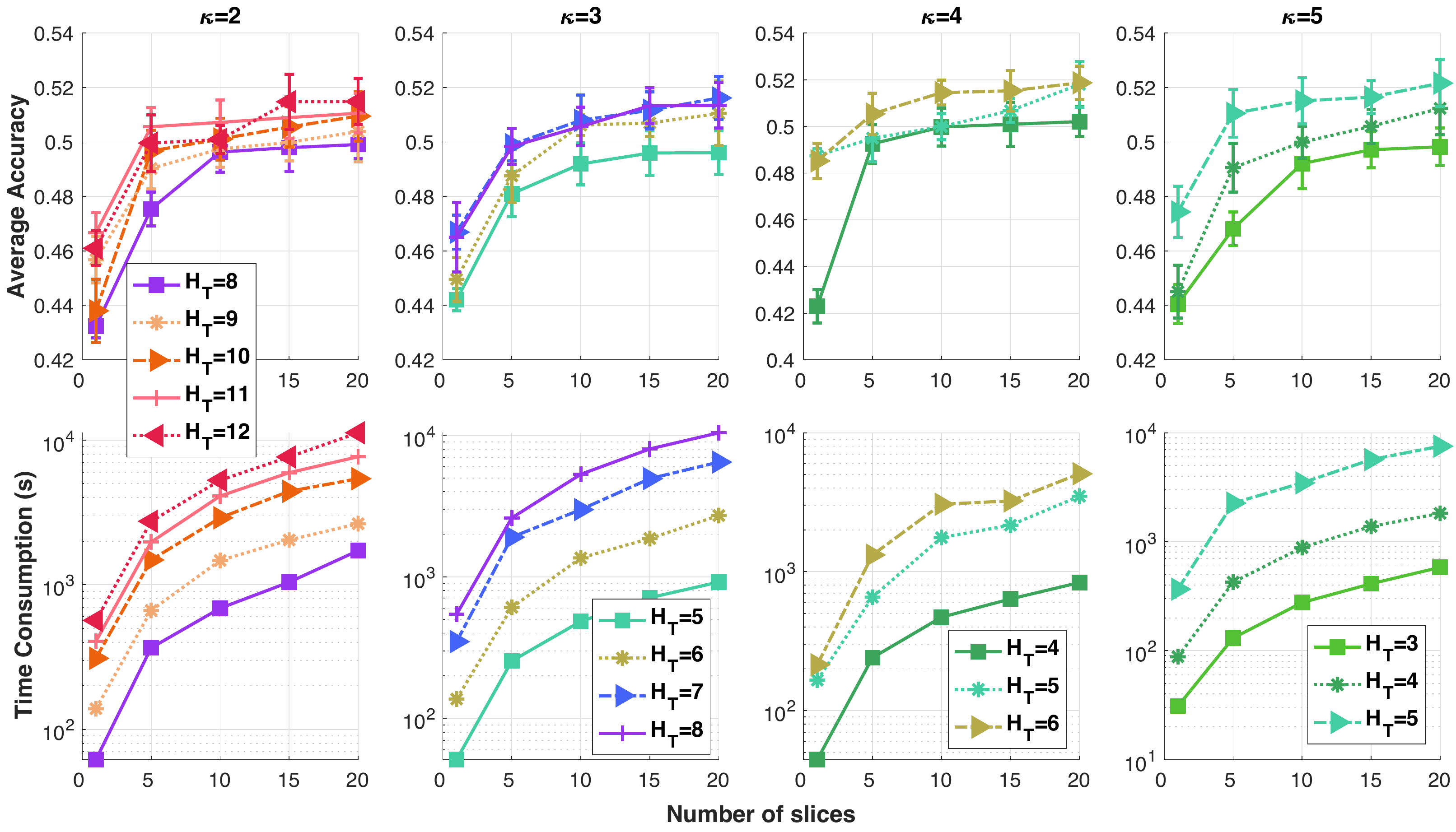}
 \vspace{-12pt}   
 \caption{In \texttt{RECIPE} dataset.}
  \label{fg:RECIPE_TM_tildeET}
    \end{subfigure}
    
    \vspace{2pt}
    
         \begin{subfigure}[t]{0.48\textwidth}
    \centering
  \includegraphics[width=\textwidth]{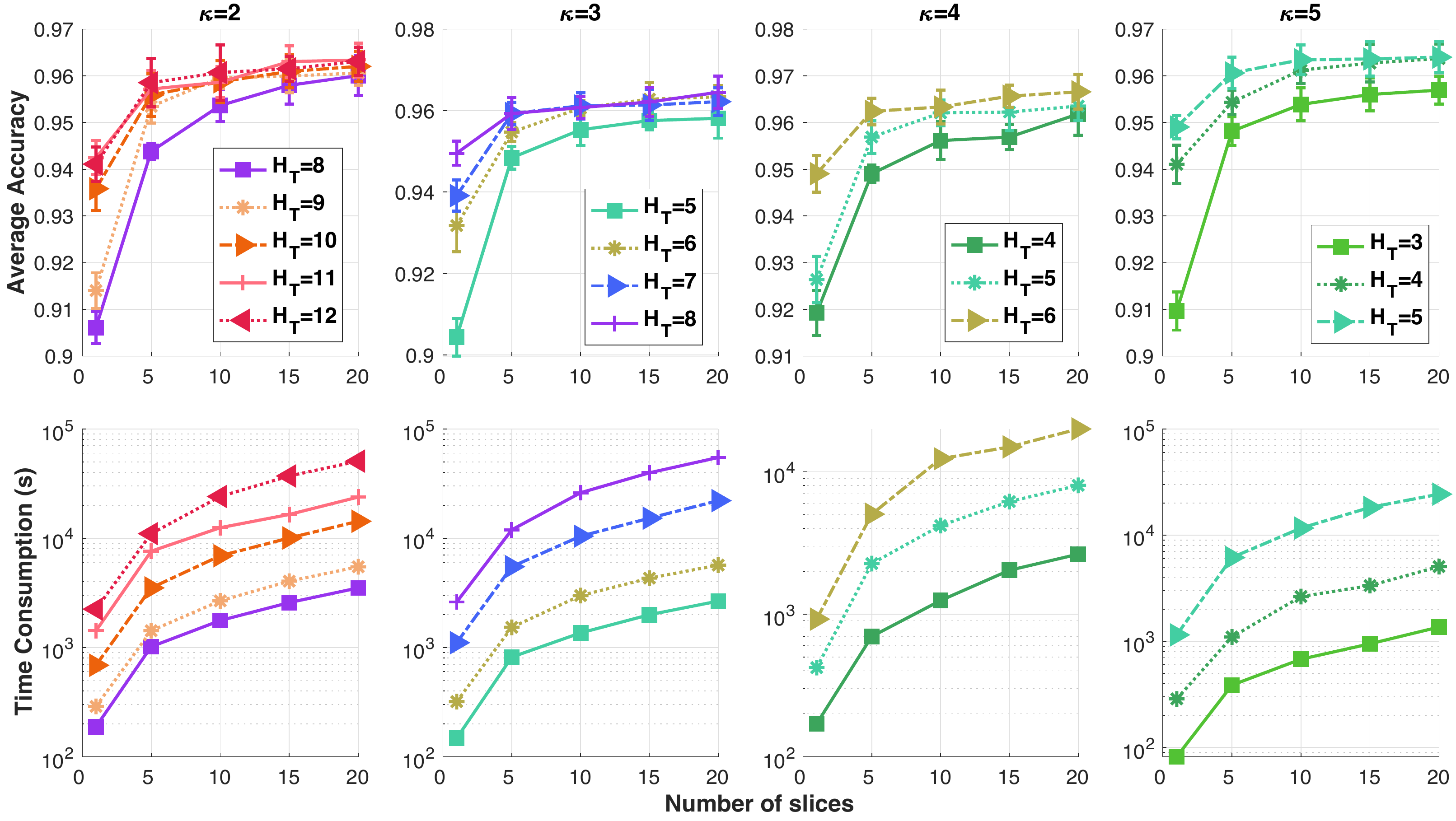}
\vspace{-12pt}        
  \caption{In \texttt{CLASSIC} dataset.}
  \label{fg:CLASSIC_TM_tildeET}
    \end{subfigure} 
   \hfill
    \begin{subfigure}[t]{0.48\textwidth}
        \centering
  \includegraphics[width=\textwidth]{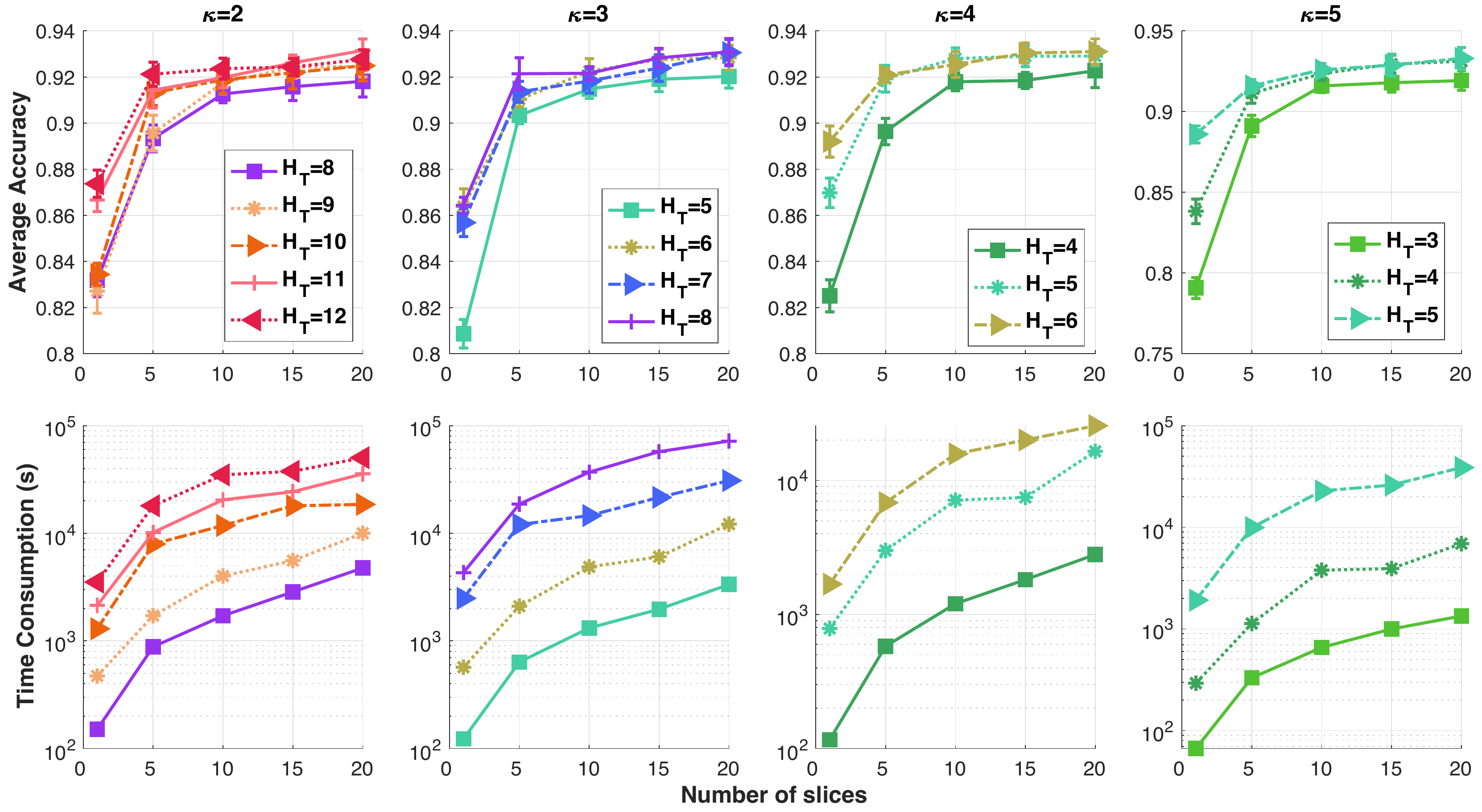}
 \vspace{-12pt}   
 \caption{In \texttt{AMAZON} dataset.}
  \label{fg:AMAZON_TM_tildeET}
    \end{subfigure}
    
\vspace{-4pt}   
    \caption{SVM results and time consumption for corresponding kernel matrices of $\tildeETlambda$ w.r.t. different parameters for clustering-based tree metric sampling (predefined tree deepest level $H_{\Tt}$, and number of tree branches $\kappa$---the number of clusters in the farthest-point clustering.).}
 \vspace{-10pt}
\end{figure}


\paragraph{TDA.}
\begin{itemize}
\item In Figure~\ref{fg:TDA_TM_dAlpha}, we illustrate further SVM results and time consumption for corresponding kernel matrices of $d_0$ in \texttt{Orbit, MPEG7} datasets w.r.t. different parameters for partition-based tree metric sampling such as the predefined tree deepest level $H_{\Tt}$.

\item In Figure~\ref{fg:TDA_TM_tildeET}, we illustrate further SVM results and time consumption for corresponding kernel matrices of $\tildeETlambda$ in \texttt{Orbit, MPEG7} datasets w.r.t. different parameters for partition-based tree metric sampling such as the predefined tree deepest level $H_{\Tt}$.

\end{itemize}

\begin{figure}[h]
     \centering
     \begin{subfigure}[t]{0.48\textwidth}
    \centering
    \includegraphics[width=\textwidth]{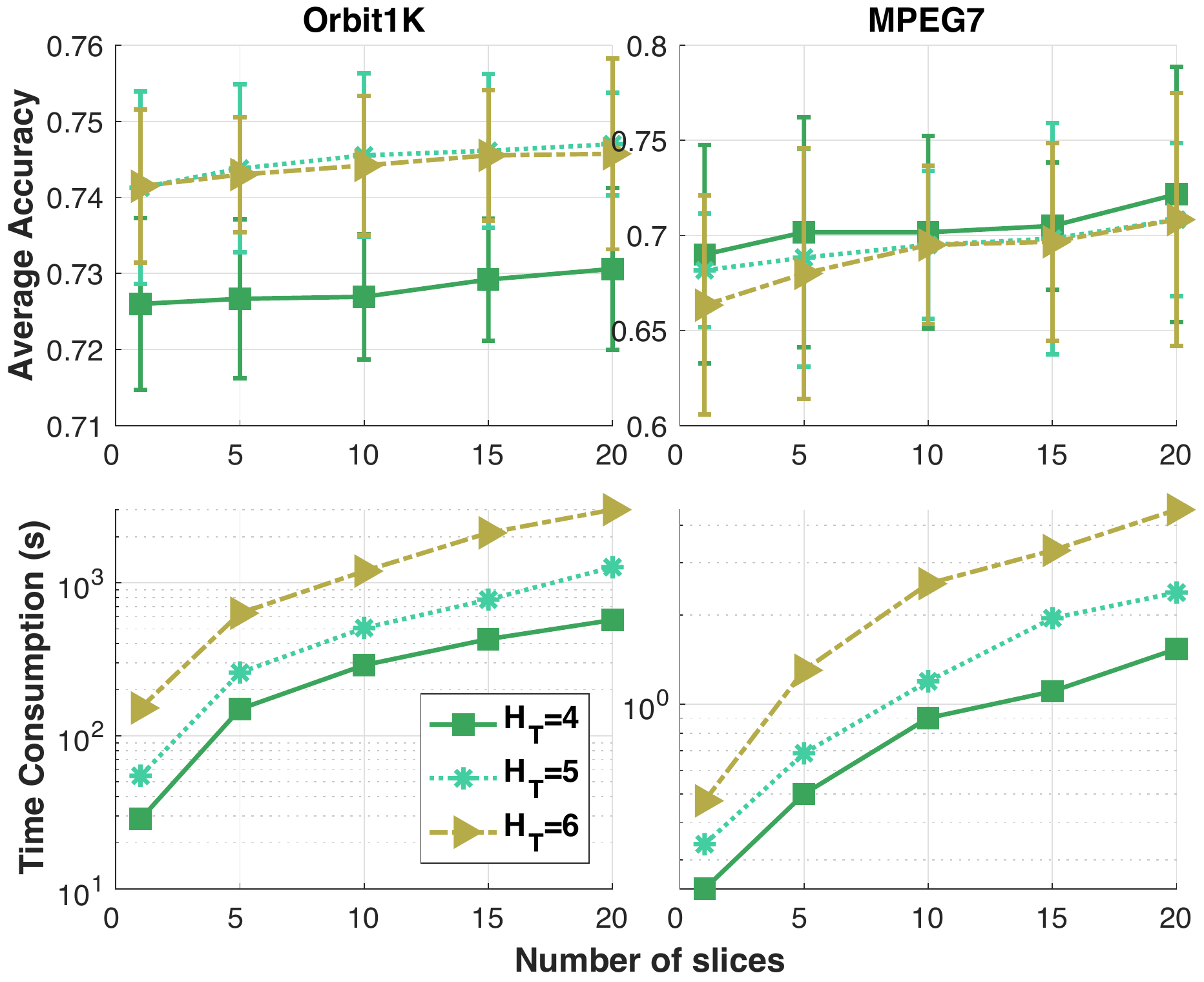}
\vspace{-12pt}        
  \caption{For $d_0$.}
  \label{fg:TDA_TM_dAlpha}
    \end{subfigure} 
   \hfill
    \begin{subfigure}[t]{0.48\textwidth}
        \centering
    \includegraphics[width=\textwidth]{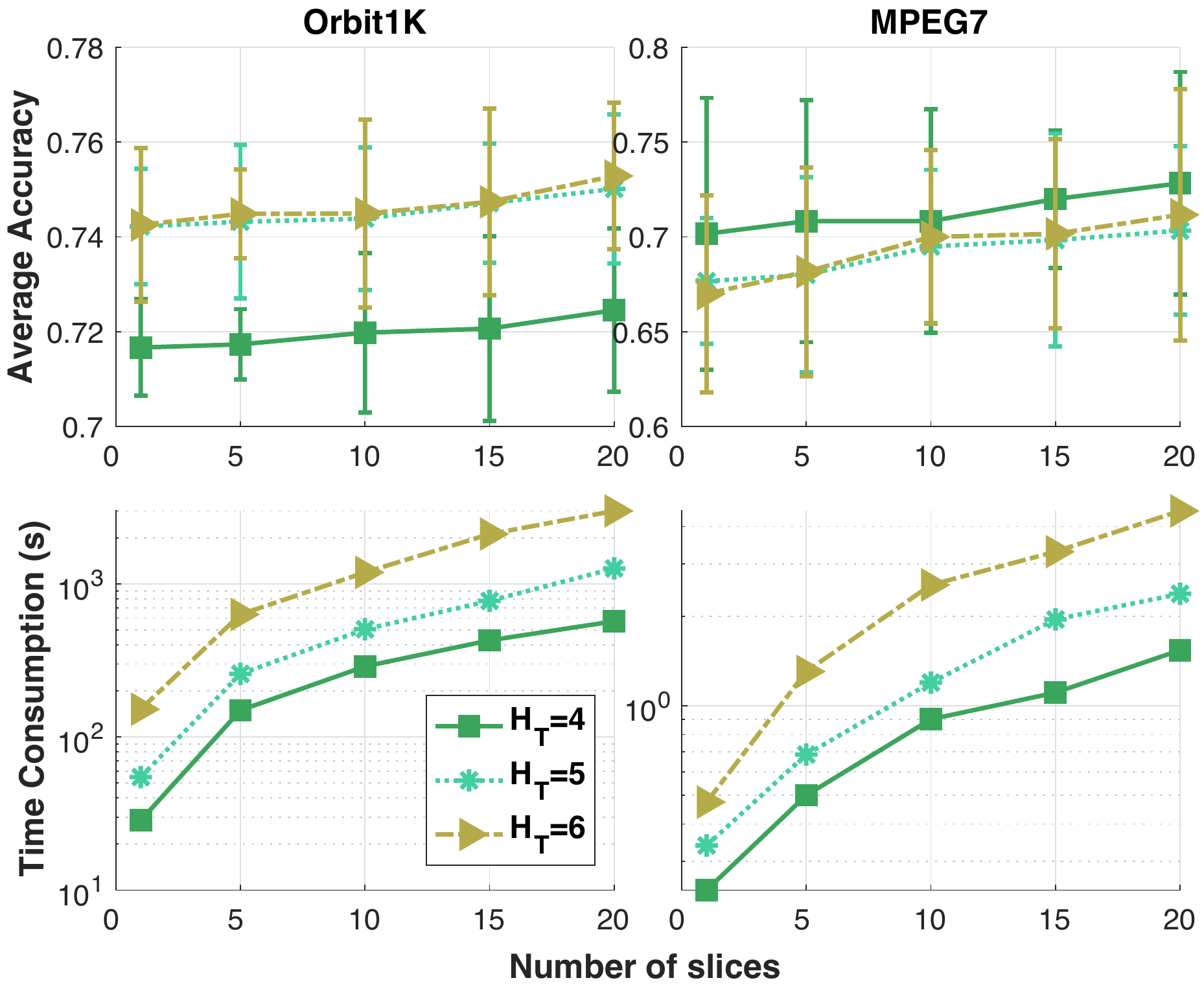}
 \vspace{-12pt}   
 \caption{For $\tildeETlambda$.}
  \label{fg:TDA_TM_tildeET}
    \end{subfigure}
\vspace{-4pt}   
    \caption{SVM results and time consumption for corresponding kernel matrices in \texttt{Orbit, MPEG7} datasets w.r.t. different parameters for partition-based tree metric sampling (predefined tree deepest level $H_{\Tt}$).}
 \vspace{-10pt}
\end{figure}

Similar as in \cite{LYFC} (tree metric sampling for tree-sliced-Wasserstein in applications), we also observed that the default parameters (e.g., the predefined deepest level $H_{\Tt}=6$, and the tree branches $\kappa=4$---the number of clusters in the farthest-point clustering) is a reasonable choice to trade-off about performances and time consumption. With these default parameters, sampled trees contains about $4000$ nodes.

\section{Further Details and Discussions}\label{supp:sec:details_discussions}

In this section, we give further details about experiments, some brief reviews about important aspects used in our work and discuss other relations to other work.

\subsection{More Details about Experiments}

In this section, we give further details about softwares, datasets and experimental setups.

\paragraph{For softwares.}

\begin{itemize}

\item For experiments in topological data analysis, we used DIPHA toolbox, available at \url{https://github.com/DIPHA/dipha}, to extract persistence diagrams.

\item For the standard complete optimal transport (OT) problem (e.g., KT in our work which we used to compute the corresponding $\ET_\lambda$), we used a fast OT implementation, available at \url{https://github.com/gpeyre/2017-ot-beginners/tree/master/matlab/mexEMD}. It is about 4 times faster than the popular mex-file with Rubner's implementation in C, available at \url{http://robotics.stanford.edu/~rubner/emd/default.htm}.

\item For tree metric sampling, we used the MATLAB implementation, available at \url{https://github.com/lttam/TreeWasserstein}. We directly used this code for clustering-based tree metric sampling, and adapted it into its special case partition-based tree metric sampling.

\item For Sinkhorn-based approach for unbalanced OT (Sinkhorn-UOT), we used the MATLAB implementation, available at \url{https://github.com/gpeyre/2017-MCOM-unbalanced-ot}.

\item For sliced partial optimal transport (SPOT), we adapt the C++ implementation, available at \url{https://github.com/nbonneel/spot}, into MATLAB.

\end{itemize}

\paragraph{For datasets.}

\begin{itemize}

\item For document datasets (e.g., \texttt{TWITTER, RECIPE, CLASSIC, AMAZON}), they are available at \url{https://github.com/mkusner/wmd}.

\item For \texttt{Orbit} dataset, we follow the procedure, detailed in \cite{adams2017persistence} to generate the dataset.

\item For \texttt{MPEG7} dataset, it is available at \url{http://www.imageprocessingplace.com/downloads_V3/root_downloads/image_databases/MPEG7_CE-Shape-1_Part_B.zip}, then we follow \cite{LYFC} to extract the 10-class subset of the dataset.

\item For granular packing system and SiO$_2$ datasets, one may access to them by contacting the corresponding authors.

\end{itemize}

\paragraph{For experimental setups.} We further clarify some details about experimental setup.

As mentioned in the main text, for $d_{0}$ and $\tildeETlambda$, we choose the weight functions for $w_1, w_2$ as
\[
	w_1(x) = w_2(x) = a_1 d_{\Tt}(r, x) + a_0,
\] 
where $r$ is the root of tree $\Tt$, we set $\lambda = b = 1$, $a_0=1$. Following \S 5.1 in the main text, we set $a_1 = b = 1$. As in \S 3.2 in the main text, $\alpha \in \left[0, \frac{1}{2}\left(b\lambda + w_1(r) + w_2(r)\right)\right]$. Thus, $\alpha \in [0, \frac{3}{2}]$ in our experiments (see more experiment results with different values of $\alpha$ in \S \ref{supp:sec:results_wrt_alpha}). We used $n_s = 10$ (tree) slices for $d_{0}$, $\tildeETlambda$ and SPOT. For tree metric sampling, we used the default hyperparameters, the predefined tree deepest level $H_{\Tt}=6$, and the tree branches $\kappa=4$---the number of clusters used in the farthest-point clustering.

\subsection{Some Brief Reviews}\label{supp:sec:reviews}

In this section, we give some brief reviews (or more referred details) about some important aspects in our work.

\paragraph{For kernels.} We review some important definitions (e.g., positive/negative definite kernels \cite{Berg84}) and theorems (e.g., Theorem 3.2.2 in \cite{Berg84}) about kernels used in our work.

\paragraph{$\bullet$ Positive definite kernels \cite[p.66--67]{Berg84}.} A kernel function $k: \Xx \times \Xx \rightarrow \RR$ is  positive definite if $\forall n \in \NN^{*}, \forall x_1, x_2, ..., x_n \in \Xx$, we have 
\[
\sum_{i, j} c_i c_j k(x_i, x_j) \ge 0, \qquad \forall c_i \in \RR.
\]

\paragraph{$\bullet$ Negative definite kernels \cite[p.66--67]{Berg84}.} A kernel function $k: \Xx \times \Xx \rightarrow \RR$ is  negative definite if $\forall n \ge 2, \forall x_1, x_2, ..., x_n \in \Xx$, we have 
\[
\sum_{i, j} c_i c_j k(x_i, x_j) \le 0, \qquad \forall c_i \in \RR \,\, \text{s.t.} \, \sum_i c_i = 0.
\]

\paragraph{$\bullet$ Theorem 3.2.2 in \cite[p.74]{Berg84} for kernels.}
If $\kappa$ is a \textit{negative definite} kernel, then $\forall t > 0$, kernel 
\[
k_{t}(x, z) := \exp{\left(- t \kappa(x, z)\right)}
\]
is positive definite.

\paragraph{For tree metric sampling.} The tree metric sampling is described in details in \cite{LYFC}[S4]. Le et al. \cite{LYFC} also reviewed the details of the farthest-point clustering in \S4.2 in the supplementary, discussed about thee quantization/clustering sensitivity problems of tree metric sampling in \S5 in the supplementary. 

\paragraph{For persistence diagrams and related mathematical definitions in topological data analysis.} We refer the reader into \cite[\S2]{kusano2017kernel} for a review about mathematical framework for persistence diagrams (e.g., persistence diagrams, filtrations, persistent homology).

\subsection{Discussions about Other Relations to Other Work}

We note that ultrametric (i.e., non-Archimedean metric, or isosceles metric) and its special case---binary metric are tree metrics \cite{LYFC}. Additionally, a metric for points in a line (e.g., in 1-dimensional projections for supports in SPOT, or SW), or in 1-dimensional manifold (e.g., in 1-dimensional manifold projections for supports in generalized SW \cite{kolouri2019generalized}) is also a tree metric since we have a corresponding tree as a chain of these points.

We also list some other studies related to OT problem with tree metrics as follows: (i) Kloeckner \cite{kloeckner2015geometric} derived geometric properties of OT space for measures on an ultrametric space, (ii) Sommerfeld and Munk \cite{Sommerfeld2016InferenceFE} studied statistical inferences for OT on finite spaces including tree metrics, (iii) tree-Wasserstein barycenter \cite{le2019wb}, (iv) alignment problems for probability measures having supports in different spaces (i.e., fast tree variants for Gromov-Wasserstein) \cite{le2021fba}, and ultrametric Gromov-Wasserstein \cite{memoli2021ultrametric}.

We note that we consider the \textit{\textbf{discrete}} measures in our work (e.g., empirical measures). The closed-form formulation of our regularized entropy partial transport (EPT) $\widetilde{\ET}_{\lambda}^{\alpha}$ in Equation (8) in the main text is for general discrete nonnegative measures having different masses. To our knowledge, the proposed regularized EPT (i.e., $\widetilde{\ET}_{\lambda}^{\alpha}$ in Equation (8) in the main text) is the first approach that yields a closed-form solution among available variants of unbalanced OT for discrete measures. In the context of unbalanced OT for \textit{\textbf{continuous}} measures (e.g., probability measures are scaled by positive constants), Janati et al. \cite{janati2020entropic} recently showed that entropic optimal transport for unbalanced Gaussian measures (i.e., Gaussian measures are scaled by different positive constants) has a closed-form solution.

\end{document}